\RequirePackage[l2tabu,orthodox]{nag}
\documentclass
[letterpaper,11pt,]
{article}

\usepackage{etex}
\usepackage{verbatim}
\usepackage{xspace,enumerate}
\usepackage[dvipsnames]{xcolor}
\usepackage[T1]{fontenc}
\usepackage[full]{textcomp}
\usepackage[american]{babel}
\usepackage{mathtools}

\usepackage{amsthm}
\usepackage[
letterpaper,
top=1in,
bottom=1in,
left=1in,
right=1in]{geometry}
\usepackage{newpxtext} %
\usepackage{textcomp} %
\usepackage[varg,bigdelims]{newpxmath}
\usepackage[scr=rsfso]{mathalfa}%
\usepackage{bm} %
\let\mathbb\varmathbb
\usepackage{microtype}
\usepackage[pagebackref,colorlinks=true,urlcolor=blue,linkcolor=blue,citecolor=OliveGreen]{hyperref}
\usepackage[capitalise,nameinlink]{cleveref}
\crefname{lemma}{Lemma}{Lemmas}
\crefname{fact}{Fact}{Facts}
\crefname{theorem}{Theorem}{Theorems}
\crefname{corollary}{Corollary}{Corollaries}
\crefname{claim}{Claim}{Claims}
\crefname{example}{Example}{Examples}
\crefname{algorithm}{Algorithm}{Algorithms}
\crefname{problem}{Problem}{Problems}
\crefname{definition}{Definition}{Definitions}
\crefname{exercise}{Exercise}{Exercises}
\usepackage{amsthm}

\newtheorem{theorem}{Theorem}[section]
\newtheorem*{theorem*}{Theorem}
\newtheorem{lemma}[theorem]{Lemma}
\newtheorem*{lemma*}{Lemma}
\newtheorem{fact}[theorem]{Fact}
\newtheorem*{fact*}{Fact}

\newtheorem*{proposition*}{Proposition}
\newtheorem{corollary}[theorem]{Corollary}
\newtheorem*{corollary*}{Corollary}

\newtheorem*{hypothesis*}{Hypothesis}

\newtheorem*{conjecture*}{Conjecture}
\theoremstyle{definition}
\newtheorem{definition}[theorem]{Definition}
\newtheorem*{definition*}{Definition}

\newtheorem*{construction*}{Construction}

\newtheorem*{example*}{Example}

\newtheorem*{question*}{Question}
\newtheorem{algorithm}[theorem]{Algorithm}
\newtheorem*{algorithm*}{Algorithm}

\newtheorem*{assumption*}{Assumption}

\newtheorem*{problem*}{Problem}

\newtheorem*{openquestion*}{Open Question}
\theoremstyle{remark}

\newtheorem*{claim*}{Claim}
\newtheorem{remark}[theorem]{Remark}
\newtheorem*{remark*}{Remark}

\newtheorem*{observation*}{Observation}
\usepackage{paralist}
\let\originalleft\left
\let\originalright\right
\renewcommand{\left}{\mathopen{}\mathclose\bgroup\originalleft}
\renewcommand{\right}{\aftergroup\egroup\originalright}
\usepackage{turnstile}
\usepackage{mdframed}
\usepackage{tikz}
\usetikzlibrary{positioning}
\usepackage{caption}
\DeclareCaptionType{Algorithm}
\usepackage{newfloat}
\usepackage{array}
\usepackage{subfig}
\usepackage{bbm}
\usepackage{xparse}
\usepackage{amsthm} %
\makeatletter
\let\latexparagraph\paragraph
\RenewDocumentCommand{\paragraph}{som}{%
  \IfBooleanTF{#1}
    {\latexparagraph*{#3}}
    {\IfNoValueTF{#2}
       {\latexparagraph{\maybe@addperiod{#3}}}
       {\latexparagraph[#2]{\maybe@addperiod{#3}}}%
  }%
}
\newcommand{\maybe@addperiod}[1]{%
  #1\@addpunct{.}%
}
\makeatother

\usepackage{boxedminipage}
\newcommand{\paren}[1]{(#1)}
\newcommand{\Paren}[1]{\left(#1\right)}

\newcommand{\Brac}[1]{\left[#1\right]}

\newcommand{\abs}[1]{\lvert#1\rvert}
\newcommand{\Abs}[1]{\left\lvert#1\right\rvert}

\newcommand{\Card}[1]{\left\lvert#1\right\rvert}

\newcommand{\set}[1]{\{#1\}}
\newcommand{\Set}[1]{\left\{#1\right\}}

\newcommand{\norm}[1]{\lVert#1\rVert}
\newcommand{\Norm}[1]{\left\lVert#1\right\rVert}

\newcommand{\Normt}[1]{\Norm{#1}_2}

\newcommand{\snorm}[1]{\norm{#1}^2}
\newcommand{\Snorm}[1]{\Norm{#1}^2}

\newcommand{\iprod}[1]{\langle#1\rangle}
\newcommand{\Iprod}[1]{\left\langle#1\right\rangle}

\newcommand{\Esymb}{\mathbb{E}}
\newcommand{\Psymb}{\mathbb{P}}

\DeclareMathOperator*{\E}{\Esymb}

\DeclareMathOperator*{\ProbOp}{\Psymb}
\renewcommand{\Pr}{\ProbOp}

\newcommand{\suchthat}{\;\middle\vert\;}
\newcommand{\tensor}{\otimes}

\newcommand{\tensorpower}[2]{#1^{\tensor #2}}

\newcommand{\Mid}{\nonscript\;\middle\vert\nonscript\;}

\newcommand\bdot\bullet

\newcommand{\flattent}[4]{#1_{\{#2\}\{#3\}\{#4\}}}

\DeclareMathOperator{\poly}{poly}

\DeclareMathOperator{\polylog}{polylog}

\DeclareMathOperator{\rank}{rank}

\newcommand{\iid}{i.i.d.\xspace}
\newcommand\naive{na\"{\i}ve\xspace}

\newcommand\naively{na\"{\i}vely\xspace}

\newcommand{\R}{\mathbb R}

\newcommand{\bbS}{\mathbb S}

\renewcommand{\leq}{\leqslant}
\renewcommand{\le}{\leqslant}
\renewcommand{\geq}{\geqslant}
\renewcommand{\ge}{\geqslant}
\let\epsilon=\varepsilon
\numberwithin{equation}{section}
\newcommand\MYcurrentlabel{xxx}
\newcommand{\MYstore}[2]{%
  \global\expandafter \def \csname MYMEMORY #1 \endcsname{#2}%
}
\newcommand{\MYload}[1]{%
  \csname MYMEMORY #1 \endcsname%
}
\newcommand{\MYnewlabel}[1]{%
  \renewcommand\MYcurrentlabel{#1}%
  \MYoldlabel{#1}%
}
\newcommand{\MYdummylabel}[1]{}
\newcommand{\torestate}[1]{%
  \let\MYoldlabel\label%
  \let\label\MYnewlabel%
  #1%
  \MYstore{\MYcurrentlabel}{#1}%
  \let\label\MYoldlabel%
}
\newcommand{\restatetheorem}[1]{%
  \let\MYoldlabel\label
  \let\label\MYdummylabel
  \begin{theorem*}[Restatement of \cref{#1}]
    \MYload{#1}
  \end{theorem*}
  \let\label\MYoldlabel
}
\newcommand{\restatelemma}[1]{%
  \let\MYoldlabel\label
  \let\label\MYdummylabel
  \begin{lemma*}[Restatement of \cref{#1}]
    \MYload{#1}
  \end{lemma*}
  \let\label\MYoldlabel
}
\newcommand{\restateprop}[1]{%
  \let\MYoldlabel\label
  \let\label\MYdummylabel
  \begin{proposition*}[Restatement of \cref{#1}]
    \MYload{#1}
  \end{proposition*}
  \let\label\MYoldlabel
}
\newcommand{\restatefact}[1]{%
  \let\MYoldlabel\label
  \let\label\MYdummylabel
  \begin{fact*}[Restatement of \cref{#1}]
    \MYload{#1}
  \end{fact*}
  \let\label\MYoldlabel
}
\newcommand{\restate}[1]{%
  \let\MYoldlabel\label
  \let\label\MYdummylabel
  \MYload{#1}
  \let\label\MYoldlabel
}

\newcommand{\e}{\epsilon}

\allowdisplaybreaks
\newcommand*{\Id}{\mathrm{Id}}

\newcommand*{\normf}[1]{\norm{#1}_{\mathrm{F}}}
\newcommand*{\Normf}[1]{\Norm{#1}_{\mathrm{F}}}

\newenvironment{algorithmbox}{\begin{mdframed}[nobreak=true]
\begin{algorithm}}{\end{algorithm}\end{mdframed}}

\newcommand*{\transpose}[1]{{#1}{}^{\mkern-1.5mu\mathsf{T}}}
\newcommand*{\dyad}[1]{#1#1{}^{\mkern-1.5mu\mathsf{T}}}

\title{
  Fast algorithm for overcomplete order-3 tensor decomposition\thanks{This project has received funding from the European Research Council (ERC) under the European Union’s Horizon 2020 research and innovation programme (grant agreement No 815464)
}
}

\author{
  Jingqiu Ding\thanks{ETH Z\"urich.}
  \and
  Tommaso d'Orsi\footnotemark[2]
  \and
  Chih-Hung Liu\footnotemark[2]
  \and
  David Steurer\footnotemark[2]
    \and
  Stefan Tiegel\footnotemark[2]
}

\begin{document}

\pagestyle{empty}

\maketitle
\thispagestyle{empty} %

\begin{abstract}

We develop the first \textit{fast} spectral algorithm to decompose a random third-order tensor over $\R^d$ of rank up to $O(d^{3/2}/\polylog(d))$. 
Our algorithm only involves simple linear algebra operations and can recover all components  in time $O(d^{6.05})$ under the current matrix multiplication time.

Prior to this work, comparable guarantees could only be achieved via sum-of-squares [Ma, Shi, Steurer 2016]. In contrast, fast algorithms [Hopkins, Schramm, Shi, Steurer 2016] could only decompose tensors of rank at most $O(d^{4/3}/\polylog(d))$. 

Our algorithmic result rests on two key ingredients. A clean lifting of the third-order tensor to a sixth-order tensor, which  can be expressed in the language of tensor networks.
A careful  decomposition of the tensor network into a sequence of rectangular  matrix multiplications, which allows us to have a fast implementation of  the algorithm.

\end{abstract}

\clearpage

\microtypesetup{protrusion=false}
\tableofcontents{}
\microtypesetup{protrusion=true}

\clearpage

\pagestyle{plain}
\setcounter{page}{1}

\section{Introduction}\label{section:introduction}
Tensor decomposition is a widely studied problem in statistics and machine learning~\cite{RabanserSG17,SidiropoulosLFH17,BacciuM20}. Techniques that recover the hidden components of a given tensor have a wide range of applications such as dictionary learning  \cite{BarakKS15, MaSS16},  clustering \cite{HsuK13}, or topic modeling \cite{AnandkumarFHKL12}. %
From an algorithmic perspective, third-order tensors --which do not admit a natural unfolding\footnote{That is, a natural mapping to squared matrices}-- essentially capture the challenges of the problem.
Given
\begin{align}\label{eq:introduction-tensor}
	\mathbf{T} = \sum_{i \in [n]} \tensorpower{a_i}{3}\quad \in (\R^d)^{\tensor 3}\,,
\end{align}
we aim to approximately recover the unknown components $\{a_i\}$.
While, in general, decomposing \cref{eq:introduction-tensor} is NP-hard \cite{HillarL13}, under natural (distributional) assumptions, polynomial time algorithms are known to accurately recover the components.
When $n\leq d$, the problem is said to be undercomplete and when $n>d$ it is called overcomplete. 
In the undercomplete settings, a classical algorithm \cite{Har} (attributed to Jennrich) can efficiently decompose the input tensor when the hidden vectors are linearly independent.
In stark difference from the matrix settings,  tensor decompositions remain unique even when the number of factors $n$ is larger than the ambient dimension  $d$, making the problem suitable  for  applications where matrix factorizations are insufficient.
This observation has motivated a flurry of work \cite{LathauwerCC07, BarakKS15, GeM15, Anandkumar-pmlr15, MaSS16, HopkinsSSS16, HopkinsSS19} in an effort to design algorithms for overcomplete tensor decompositions.

When the hidden vectors are sampled uniformly from the unit sphere\footnote{It is understood that similar reasoning applies to \iid Gaussian vectors and other subgaussian symmetric distributions.}, the best guarantees in terms of number of components with respect to the ambient dimension, corresponding to $\tilde{\Omega}( n^{2/3}) \leq d$,\footnote{We hide constant factors with the notation $O(\cdot), \Omega(\cdot)$ and multiplicative \textit{polylogarithmic} factors in the ambient dimension $d$ by $\tilde{O}(\cdot), \tilde{\Omega}(\cdot)$.}  have been achieved through semidefinite-programming  \cite{MaSS16}.
The downside of this algorithm is that it is virtually impossible to be effectively used in practice due to the high order polynomial running time. For this reason, obtaining efficient algorithms for overcomplete tensor decomposition has remained a pressing research question. This is also the focus of our work.

Inspired by the insight of previous sum-of-squares algorithms \cite{GeM15}, \cite{HopkinsSSS16} proposed the first subquadratic spectral algorithm for overcomplete order-3 tensor decomposition. This algorithm,  successfully recovers the hidden vectors as long as $\tilde{\Omega}(n^{3/4}) \leq d$, but falls short of the  $\tilde{\Omega}( n^{2/3}) \leq d$ guarantees obtained via sum-of-squares.  
For $\tilde{\Omega}( n^{2/3}) \leq d$, the canonical tensor power iteration  is known to  converge to one of the hidden vectors --in nearly linear time\footnote{Hence it requires $\tilde{O}(n\cdot d^3)$ time to recover all components.}-- given an 
initialization vector with \textit{non-trivial}  correlation to one of the components \cite{Anandkumar-pmlr15}.   
Unfortunately, this does not translate to any speed up with respect to the aforementioned sum-of-squares algorithm, as that remains the only efficient algorithm known to obtain such an initialization vector.  
In the related context of fourth order tensors, under algebraic assumptions  satisfied by random vectors, \cite{LathauwerCC07, HopkinsSS19} could recover up to $n\leq d^2$ components in subquadratic time. These results however cannot be applied to third-order tensors. 

In this work, we present the first \textit{fast} spectral algorithm that  provably recovers all the hidden components as long as $ \tilde{\Omega}(n^{2/3}) \leq d$, under natural distributional assumptions.
To the best of our knowledge, this is the first  algorithm with a \textit{practical} running time that provides guarantees comparable to SDP-based algorithms.
More concretely we prove the following theorem.

\begin{theorem}[Fast overcomplete tensor decomposition]\label{theorem:main}
	Let $\mathbf{T}\in \Paren{\R^d}^{\otimes 3}$ be a tensor of the form 
	\begin{align*}
		\mathbf{T} = \underset{i\in [n]}{\sum} \tensorpower{a_i}{3}\,,
	\end{align*}
	where $a_1,\ldots, a_n$ are \iid vectors sampled uniformly from the unit sphere in $\R^d$ and $ \tilde{\Omega}\Paren{n^{2/3}} \leq d$.
	There exists a randomized algorithm that, given $\mathbf{T}$, with high probability recovers all components within error $\tilde{O}(\sqrt{n}/d)$ in time  $\tilde{O}\Paren{d^{2\omega\Paren{1+\frac{\log n}{2\log d}}}}$, where $d^{\omega(k)}$ is the time required to multiply a $(d^k\times d)$ matrix with a $(d\times d)$ matrix.\footnote{In \cref{section:matrix-multiplication-constants} we provide a table containing current upper bounds on rectangular matrix multiplication constants. }
\end{theorem}

In other words, \cref{theorem:main} states that there exists an algorithm that, in time $\tilde{O}\Paren{d^{2\omega\Paren{1+\frac{\log n}{2\log d}}}}$, outputs vectors $b_1,\ldots, b_n \in \R^d$ such that
\begin{align*}
	\forall i\in [n]\,, \quad \Norm{a_i - b_{\pi[i]}}\leq \tilde{O}\Paren{\frac{\sqrt{n}}{d}}\,,
\end{align*}
for some permutation $\pi:[n]\rightarrow [n]$.

The distributional assumptions of \cref{theorem:main} are the same of \cite{HopkinsSSS16, MaSS16}. %
In contrast to \cite{HopkinsSSS16}, our result  can deal with the inherently harder settings of $\tilde{\Omega}(n^{2/3})\leq d\leq \tilde{O}(n^{3/4}) $. %
In comparison to the  sum-of-squares algorithm in \cite{MaSS16}, which runs in time $\tilde{O}(nd)^{C}$, for a large constant $C\geq 12$, our algorithm provides significantly better running time.
For $ \tilde{\Omega}(n^{2/3}) \leq d$, it holds that $\omega\Paren{1+\frac{\log n}{2\log d}}\leq \omega(1.75) $. Current upper bounds on rectangular matrix multiplication constants show that  $\omega(1.75)\leq 3.021591$ and thus, the algorithm runs in time at most $\tilde{O}\Paren{d^{6.043182}}$.
Moreover, with the current upper bounds on $\omega(\frac{5}{3})$, the algorithm even runs in subquadratic time for $ \tilde{\Omega}(n^{3/4}) \leq d$.

\section{Preliminaries}\label{section:preliminaries}

\providecommand{\unitsphere}{{\bbS^{n-1}}}

\paragraph{Organization}
The paper is organized as follows. We present the main ideas in 
\cref{section:techniques}. In \cref{section:non-robust-algorithm} we present the algorithm 
for fast overcomplete third-order tensor decomposition. We prove 
its correctness through \cref{section:lifting-tensor_network}, 
\cref{section:robust-six-tensor-decomposition}, and \cref{section:non-robust-algorithm-full-recovery}. In section \cref{section:simple-algorithm-running-time} we analyze the running time of the algorithm. Finally, \cref{section:robust-six-tensor-decomposition} contains a proof for robust order-$6$ tensor decomposition which is essentially  standard, but instrumental for our result.

\paragraph{Notations for matrices} Throughout the paper, we denote 
matrices by non-bold capital letters $M \in \R^{d\times d}$ 
and vectors $v\in \R^d$ by lower-case letters.
Given a matrix $M\in \R^{d^2\times d^2}$, 
at times we denote its entries with the indices $i,j,k,\ell \in [d]$. $M_{i,j,k,\ell}$ is the 
$(i\cdot j)$-$(k\cdot \ell)$-th entry of $M$. We then write $M_{\Set{1,2,3}\Set{4}}$ for the $d^3$-by-$d$ matrix obtained reshaping $M$, so that $\Paren{M_{\Set{1,2,3}\Set{4}}}_{i,j,k,\ell} = M_{i,j,k,\ell}$. Analogously, we express reshapings of matrices in $\R^{d^3\times d^3}$.
We denote the identity matrix in $\R^{m\times m}$ by $\Id_m$.
For any matrix $M$, we denote its Moore-Penrose inverse as $M^+$, its spectral norm 
 as $\norm{M}$ and its Frobenius norm as $\normf{M}$.

\paragraph{Notations for tensors} Throughout the paper we denote tensors by boldface capital letters $\mathbf{T}\in \Paren{\R^{d}}^{\otimes t}$.
For simplicity, for a vector $v \in \R^d$, we denote by $\tensorpower{v}{t}\in \Paren{\R^{d}}^{\otimes t}$ both the tensor $v\underbrace{\otimes \ldots \otimes}_{t \text{ times}} v$ and its vectorization $\tensorpower{v}{t}\in \R^{d^t}$, we also write $\Paren{\tensorpower{v}{\ell}}\transpose{\Paren{\tensorpower{v}{t-\ell}}}\in \R^{d^\ell\times d^{t-\ell}}$ for the $d^\ell$-by-$d^{t-\ell}$ matrix flattening of $\tensorpower{v}{t}$.
If this is denoted by a boldface capital letter it is taken to be a tensor and if it is denoted by a non-bold capital letter as  a matrix.
We expect the meaning to be clear from context.
For  a tensor  $\mathbf{T}\in \Paren{\R^{d}}^{\otimes t}$ and a partition of its modes into ordered sets $S_1,\ldots, S_\ell \subseteq \Set{1,\ldots,t}$  we denote by $\mathbf{T}_{S_1,\ldots, S_\ell}$ its flattening into an $\ell$-th order tensor. For example, for $A,B\subseteq \Set{1,\ldots, t}$ with $A\cup B = \Set{1,\ldots, t}$ and $A\cap B=\emptyset$, $\mathbf{T}_{A, B}$ is a $d^{\Card{A}}$-by-$d^{\Card{B}}$ matrix flattening of $\mathbf{T}$. We remark that the order of the modes matter.
For a tensor $\mathbf{T}\in (\R^d)^{\tensor 3}$ and a vector $v\in \R^d$, we denote by $\mathbf{T}(v,\cdot,\cdot)$ 
 or $\Paren{v\otimes \Id_{d}\otimes \Id_{d}}T$ the matrix obtain contracting the first mode of $\mathbf{T}$ with $v$. 
 A similar notation will be used for higher order tensors. 
 Given a tensor $\mathbf T\in (\R^{d})^{\otimes 6}$, we sometimes write 
  $\mathbf T_{\{1,2\}\{3,4\}\{5,6\}}$ as its reshaping to a $d^2\times d^2\times d^2$ tensor.

\paragraph{Notations for probability and asymptotic bounds}  
We hide constant factors with the notation $O(\cdot), \Omega(\cdot)$ and 
multiplicative \textit{polylogarithmic} factors in the ambient dimension 
$d$ by $\tilde{O}(\cdot), \tilde{\Omega}(\cdot)$.

We denote the standard Gaussian distribution by $N(0, \Id_m)$. We say an event happens with high probability if it happens with probability 
$1-o(1)$.  We say an event happens with overwhelming probability (or \emph{w.ov.p}) if 
 it happens with probability $1-d^{-\omega(1)}$.

\paragraph{Tensor networks} There are many different ways one can multiply tensors together. An expressive tool that can be used to represent some specific tensor multiplication is that of tensor networks. 
A tensor newtork is a diagram with nodes and edges (or legs). Nodes represent tensors and edges between nodes represent contractions. Edges can be dangling and need not be between pairs of nodes.
Thus a third order tensor $\mathbf{T}\in(\R^d)^{\tensor 3}$ corresponds to a node with three dangling legs. Further examples are shown in the picture below. For a more detailed discussion we direct the reader to \cite{MoitraW19}. 

\begin{figure}[!ht]
	\label{fig:tensor_network}
	\centering
	\includegraphics[width=13cm]{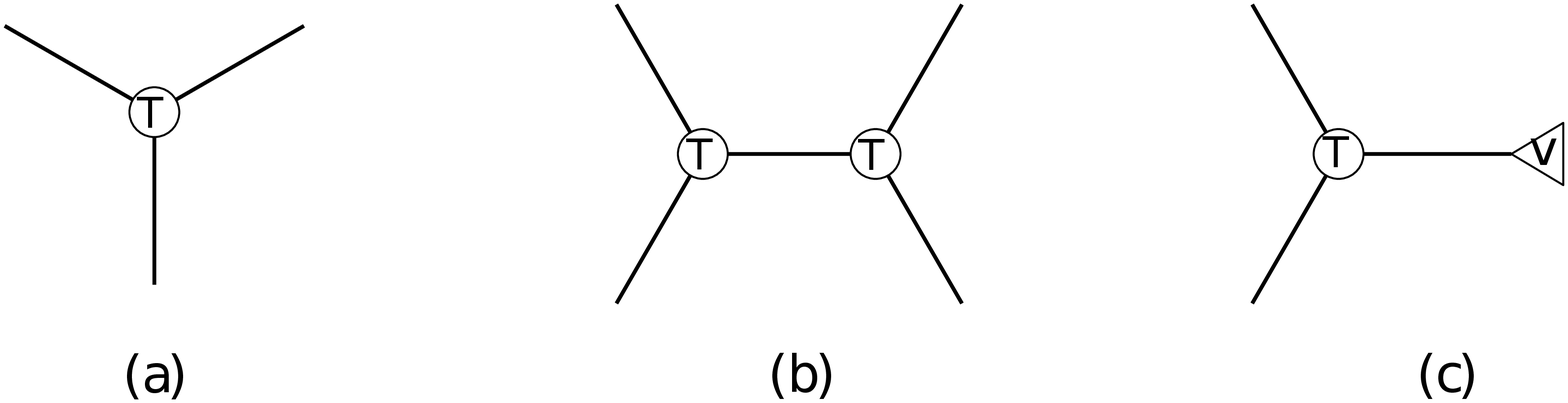}

	\caption{\hyperref[fig:tensor_network]{Fig 1.(a)} represents a single third-order tensor. \hyperref[fig:tensor_network]{Fig 1.(b)} depicts two tensor contracted via one mode. \hyperref[fig:tensor_network]{Fig 1.(c)} represents a tensor contracted on one mode with a vector.}
\end{figure}
\section{Techniques}\label{section:techniques}

Here we present the main ideas behind our result. Throughout the section we assume to be given a tensor $\mathbf{T}=\underset{i \in [n]}{\sum}\tensorpower{a_i}{3}\in (\R^d)^{\tensor 3}$ with components $a_1,\ldots, a_n\in \R^d$ independently and uniformly sampled from the unit sphere.

\paragraph{From $ \tilde \Omega \paren{n^{3/4}} \leq d$ to $ \tilde \Omega \paren{n^{2/3}} \leq d$: a first matrix with large spectral gap}
To understand how to recover the components for $ \tilde{\Omega}\Paren{n^{2/3}} \leq d$, it is useful to revisit the spectral algorithm in \cite{HopkinsSSS16}.
For a random contraction $g\sim N(0,\Id_d)$, this can be described by the tensor network in  \cref{fig:order-4-tensor-networks}(a) and amounts to computing the $n$ leading eigenvectors of the matrix 
\begin{align}
\underset{i,j \in [n]}{\sum} \iprod{g, \mathbf{T}(a_i\tensor a_j)}\Paren{a_i\tensor a_j}\transpose{\Paren{a_i\tensor a_j}} =&
\sum_{i\in [n]}\iprod{g,a_i}\Paren{\tensorpower{a_i}{2}}\transpose{\Paren{\tensorpower{a_i}{2}}} \nonumber \\
&+ \underbrace{\underset{i,j \in [n]\,, i\neq j}{\sum} \iprod{g, \mathbf{T}(a_i\tensor a_j)}\Paren{a_i\tensor a_j}\transpose{\Paren{a_i\tensor a_j}}}_{:=\bm{E}}\label{eq:techniques-order-4-tensor}
\end{align}
Since $\Norm{\sum_{i\in [n]}\iprod{g,a_i}\Paren{\tensorpower{a_i}{2}}\transpose{\Paren{\tensorpower{a_i}{2}}}}= \tilde{\Theta}(1)$, as long as the spectral norm of the noise $E$ is significantly smaller, the signal-to-noise ratio stays bounded away from zero and we can hope to recover the components.
By decoupling inequalities similar to those in \cite{GeM15}, w.h.p., it holds that $\langle g, \mathbf{T}(a_i \otimes a_j)\rangle \leq \tilde{O}(\sqrt{n}/d)$,
and the derivations in \cite{HopkinsSSS16} further show that  $\lVert E\rVert \leq \tilde{O}(n^{3/2}/d^2)$.
Hence, this algorithm can recover the components as long as  $ \tilde{O}(n^{3/4}) \leq d$.

\begin{figure}%
	\centering
	\includegraphics[width=12cm]{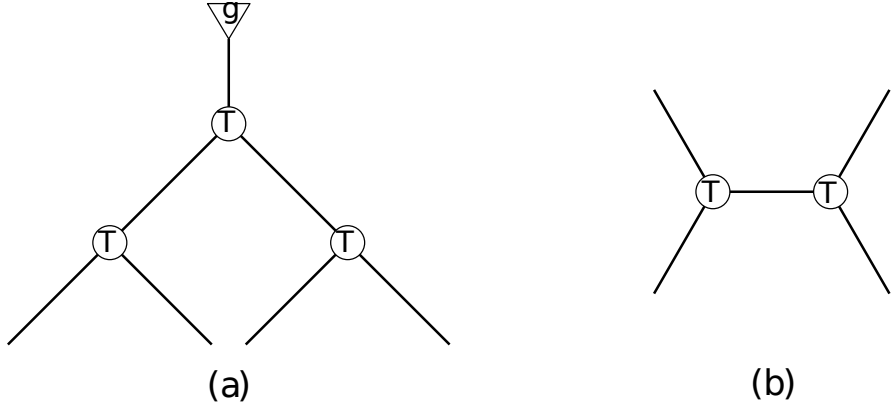}
	\caption{(a ) The tensor network for the algorithm in \cite{HopkinsSSS16} where $g\sim N(0, \Id_d)$. (b) A simple tensor network with signal-to-noise ratio $\tilde{\Omega}\Paren{d^{3/2}/n}$.}\label{fig:order-4-tensor-networks}
\end{figure}

To improve over this result, the first key observation to make is that the term $\iprod{g, \mathbf{T}(a_i\tensor a_j)}$ is unnecessarily large. In fact, for $n>d$, it is significantly larger (in absolute value) than the inner product $\Abs{\iprod{a_i, a_j}}\leq \tilde{O}(1/\sqrt{d})$, which appears to be a reasonable yardstick for the scalar values at play in the computation, as we try to exploit the near orthogonality of the components.
This suggest that even simply replacing  $\iprod{g, \mathbf{T}(a_i\tensor a_j)} $ by the inner product $\iprod{a_i, a_j}$ could increase the spectral gap between the components we are trying to retrieve and the noise.
Indeed, this can be achieved by considering the tensor network in \cref{fig:order-4-tensor-networks}(b), corresponding to the matrix 
\begin{align*}
	\underset{i,j\in [n]}{\sum} \iprod{a_i, a_j}\Paren{a_i\tensor a_j}\transpose{\Paren{a_i\tensor a_j}} &= \sum_{i\in [n]}\Paren{\tensorpower{a_i}{2}}\transpose{\Paren{\tensorpower{a_i}{2}}} + \underbrace{\underset{i,j\in [n], i\neq j}{\sum} \iprod{a_i, a_j}\Paren{a_i\tensor a_j}\transpose{\Paren{a_i\tensor a_j}}}_{:=E}.
\end{align*}
On the one hand, with high probability, the spectral norm of the signal part satisfies $\Norm{\sum_{i\in [n]}\Paren{\tensorpower{a_i}{2}}\transpose{\Paren{\tensorpower{a_i}{2}}}}
 =\Omega(1)$. 
On the other hand by \cite[Lemma~13]{GeM15}, with high probability, the spectral norm of $E$ is $\tilde{O}({n/d^{3/2}})$.
Thus, this simple tensor network provides the noise with the spectral norm we are looking for, i.e., $o(1)$ as long as $n\leq \tilde{O}(d^{3/2})$.\footnote{We remark that the tensor network in \cref{fig:order-4-tensor-networks}(b) was implicitly considered in \cite{GeM15} in the analysis of their quasi-polynomial time SoS algorithm.}

The problem with the fourth order tensor network above is that it is not clear how one could directly extract even a single component. The canonical recipe, namely: (i) apply a random contraction  $g\sim N(0, \Id_{d^2})$, (ii) recover the top eigenvector; does not work as after contracting the tensor we would end up with a rank $d$ matrix, while we wish to recover $n>d$ vectors.
A natural workaround to this issue consists of lifting the fourth order tensor to a higher dimensional space and \textit{then} applying the canonical recipe.

\paragraph{Lifting to a higher order using tensor networks}

It is straightforward to phrase lifting to higher orders in the language of tensor networks.
For example, consider the following network (\cref{fig:techniques-order-six-network}):
\begin{figure}[!ht]
	\centering
	\vspace{0.5cm}
	\includegraphics[width=5cm]{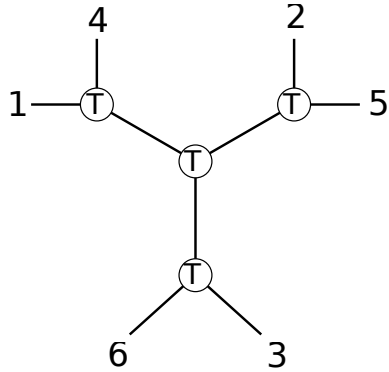}
	\caption{Lifting of the tensor network in \cref{fig:order-4-tensor-networks}(b). The numbers attached to the dangling edges can be used to keep track of the flattenings we will use throughout the paper.}\label{fig:techniques-order-six-network}
\vspace{0.5cm}
\end{figure}

\noindent In a similar spirit to \cref{fig:order-4-tensor-networks}(b), this tensor network can be flattened as the $d^3$-by-$d^3$ matrix
\begin{align*}
	T_6 &= \sum_{i\in [n]}\Paren{a_i^{\tensor 3}}\transpose{\Paren{a_i^{\tensor 3}}}  + \underbrace{\sum_{\substack{\Set{i\,, j\,,k\,, \ell}\in [n]^4\\ i,j,k,\ell \text{ not all equal}}}\iprod{a_i, a_j}\iprod{a_i, a_k}\iprod{a_i, a_\ell}\paren{a_j\tensor a_k\tensor a_k}\transpose{\paren{a_j\tensor a_\ell\tensor a_\ell}}}_{\eqqcolon E} \,.
\end{align*}
Here $E$ is a sum of $O\Paren{n^4}$ dependent random matrices and thus, a priori, it is not clear how to study its spectrum. In particular there are many different terms in $E$ with  distinct, but possibly aligning, spectra. %
To overcome this obstacle, we partition the terms in $E$ based on their index patterns. Mapping each index to a color, this essentially amounts to considering all the non-isomorphic %
$2$-, $3$- or $4$-colorings  of the tensor network in \cref{fig:techniques-order-six-network} (picking one arbitrary representative per class).
Since the number of such non-isomorphic colorings is constant, we can bound each set in the partition separately, knowing that this triangle inequality will be tight up to constant factors.

To build some intuition consider as an example the case in which $i \neq j = k = l$.
This corresponds to the coloring in which we assign a given color to the center node and a different one to all the leaves.
Let $E'$ denote the error matrix corresponding to this case.
Then, using a decoupling inequality similar to the one used for the analysis of the networks in \cref{fig:order-4-tensor-networks} and standard Matrix Rademacher bounds, we obtain $$\Norm{E'} = \Norm{\sum_{i,j \in [n], i \neq j} \iprod{a_i,a_j}^3 (a_j^{\otimes 3}) (a_j^{\otimes 3})^\top} \leq \tilde{O} \Paren{\sqrt{n} \cdot \frac{1}{\sqrt{d^3}}} \cdot \Norm{\sum_{j \in [n]} (a_j^{\otimes 3}) (a_j^{\otimes 3})^\top} \,,$$
where we also used again that for $i \neq j$ it holds that $\abs{\iprod{a_i, a_j}} \leq  \tilde{O} (1/\sqrt{d})$.
Since the spectral norm of the sum on the right-hand side can be bounded by $\tilde{O}(1)$, it follows that $\Normt{E'} = \tilde{O}(\sqrt{n/d^3}) = \tilde{O}(\sqrt{n^2/d^3})$.
Using arguments in a similar spirit, we can also bound the spectral norm of the other colorings by $ \tilde{O}(\sqrt{n^2/d^3})$ as desired.
This allows us to show that overall the noise has also spectral norm bounded by $\tilde{O}(\sqrt{n^2/d^{3}})$, implying that the signal-to-noise ratio has not increased.

\paragraph{Recovering one component from the tensor network}
To recover a single component form this network, we can do the following:
Contracting (an appropriately flattened version of) $T_6$ with a random vector $g\sim N(0, \Id_{d^2})$ results in the matrix
\begin{align}\label{eq:techniques-bad-contraction}
	\sum_{i\in [n]}\iprod{g, a_i^{\tensor 2}}\Paren{a_i^{\tensor 2 }}\transpose{\Paren{a_i^{\tensor 2}}} + \underset{i,j \in [n]}{\sum} g_{ij} E_{ij}\,.
\end{align}
Compared to \cref{eq:techniques-order-4-tensor}, the good news is that the contraction has broken the symmetry of the signal. 
However, well-known facts about Gaussian matrix series assert that the spectral norm of the randomly contracted error term 
behaves like the norm of a $d^4$-by-$d^2$ flattening of $E$, which necessarily satisfies the inequality
\begin{align*}
	\Snorm{E_{\Set{1,2,3,4}\Set{5,6}}}\geq \frac{\Normf{E}^2}{\rank(E_{\Set{1,2,3,4}\Set{5,6}})}\geq \tilde{\Omega}(n/d)\,,
\end{align*}
thus jeopardizing our efforts of having a large spectral signal-to-noise ratio. 
We can overcome this issue with two preprocessing steps. (i) Truncate $T_6$ to its best rank-$n$ approximation $T_6^{\leq n}$ recovering its $n$ leading eigenvectors, so to have $\Normf{E}\leq \sqrt{n}\cdot \tilde{O}(n/d^{3/2})$. (ii) Project the truncated matrix  onto the space of matrices with bounded spectral norm after  rectangular reshapings\footnote{It can be observed that each of these projection does not destroy the properties ensured by the others. In other words two projections are enough to ensure the resulting matrix is in the intersection of the desired subspaces.} 
\begin{align*}
	\Norm{\Paren{T_6^{\leq n}}_{\Set{1,2,3,4}\Set{5,6}}}\leq 1\,,\,\qquad \Norm{\Paren{T_6^{\leq n}}_{\Set{1,2,5,6}\Set{3,4}}}\leq 1\,.
\end{align*}
After this sequence of projections, we can take a random contraction. In the resulting matrix
\begin{align*}
	\tilde{T}_4= 	\sum_{i\in [n]}\iprod{g, a_i^{\tensor 2}}\Paren{a_i^{\tensor 2 }}\transpose{\Paren{a_i^{\tensor 2}}} +\tilde{E}\,,
\end{align*}
the noise satisfies $\Norm{\tilde{E}}\leq \Theta(1)$ and $\Normf{\tilde{E}}\leq \Norm{E}\cdot \sqrt{n}\leq \tilde{O}\Paren{\frac{n^2}{d^{3}}\cdot \sqrt{n}}$. 
We can thus approximately recover the components not hidden by the noise.
This approach for partially recovering the components is similar in spirit to \cite{SchrammS17}. 
However, for recovering all of the components, additional steps and a finer analysis are needed  
 compared to \cite{SchrammS17}, since the input tensor is overcomplete. 

\paragraph{Recovering all components from the tensor network}

While the noise in $\tilde{T}_4$ is not adversarial,  it has become difficult to manipulate after the pre-processing steps outlined above. 
The issue is that, without looking into $\mathbf{E}$, we cannot guarantee that its eigenvectors are spread enough and do not \textit{cancel out} a fraction of the components, 
making full recovery impossible.
Nevertheless the above reasoning ensures we can obtain $\tilde{O}(n/d^{3/2})$-close  approximation vectors $b_1,\ldots, b_m\in \R^d$ of  components $a_1,\ldots, a_m$ for some $\Omega(n)\leq m< n$.

Now, a natural approach to recover all components  would be that of subtracting the learned components
\begin{align*}
	T_6'&= T_6-\sum_{i \in [m]} \dyad{\Paren{\tensorpower{{b}_i}{3}}}
\end{align*}
 and repeat the algorithm on $T'_6$. 
The approximation error here is
 \begin{align*}
 	\Norm{\sum_{i \in [m]} \dyad{\Paren{\tensorpower{a_i}{3}}}-\sum_{i \in [m]} \dyad{\Paren{\tensorpower{b_i}{3}}}}\approx \tilde{O}\Paren{\sqrt{m}\cdot (n/d^{3/2})^3}
 \end{align*}
and so if indeed $n= o(d^{8/7})$ we could simply rewrite 
\begin{align*}
	T'_6 = \sum_{m+1\leq i\leq n} \dyad{\Paren{\tensorpower{a_i}{3}}} + E'
	\,,\qquad \text{where }\Norm{E'}\leq O(1/\polylog(d))\,.
\end{align*}
For $n=\omega(d^{(8/7)})$, however the approximation error of our estimates is too large and this strategy fails.

We work around this obstacle boosting the accuracy of our estimates.
We use each $b_i$ has a \textit{warm start} and perform tensor power iteration \cite{Anandkumar-pmlr15}. For each estimate this yield a new vector $\tilde{b}_i$ satisfying
\begin{align*}
	1-\iprod{a_i, \tilde{b}_i}\leq \tilde{O}(\sqrt{n}/d)\,.
\end{align*}
Since now 
 \begin{align*}
	\Norm{\sum_{i \in [m]} \dyad{\Paren{\tensorpower{a_i}{3}}}-\sum_{i \in [m]} \dyad{\Paren{\tensorpower{\tilde{b}_i}{3}}}}\approx O\Paren{\sqrt{m}\cdot (\sqrt{n}/d)^3}\,,
\end{align*}
as $ \tilde{\Omega}(n^{2/3}) \leq d$ and $m\leq n$, we can subtract these estimates from $T_6$ and repeat the algorithm.

\paragraph{Speeding up the  computation via tensor network decomposition} The algorithm outlined above is particularly natural and streamlined, however a \naive implementation would require running time significantly larger than the result in \cref{theorem:main}. For example,
\naively computing the first $n$  eigenvectors of $T_6$ already requires time $O(n\cdot d^6)$. 
To speed up the algorithm we carefully compute an implicit (approximate) representation of $T_6$ in terms of its $n$ leading eigenvectors. Then use Gaussian rounding on this approximate representation of the data.
Since the signal part $\sum_{i\in [n]}\dyad{\Paren{\tensorpower{a_i}{3}}}$ has rank $n$, this approximation should loose little information about the components.
This implicit representation is similar to the one used in \cite{HopkinsSS19}, however our path to computing it presents different challenges and thus differs significantly  from previous work.

Our strategy is to use power iteration over $T_6$. The running time of such an approach is bounded by the time required to contract $T_6$ with a vector $v$ in $\R^{d^3}$. 
However, since we have access to $\mathbf{T}$, by \textit{carefully decomposing} the tensor network we can perform this matrix-vector multiplication in a significantly smaller number of operations.
In particular, as shown in \cref{fig:power_iteration-ternary_tree}, 
\begin{figure}
	\centering	\includegraphics[width=15cm]{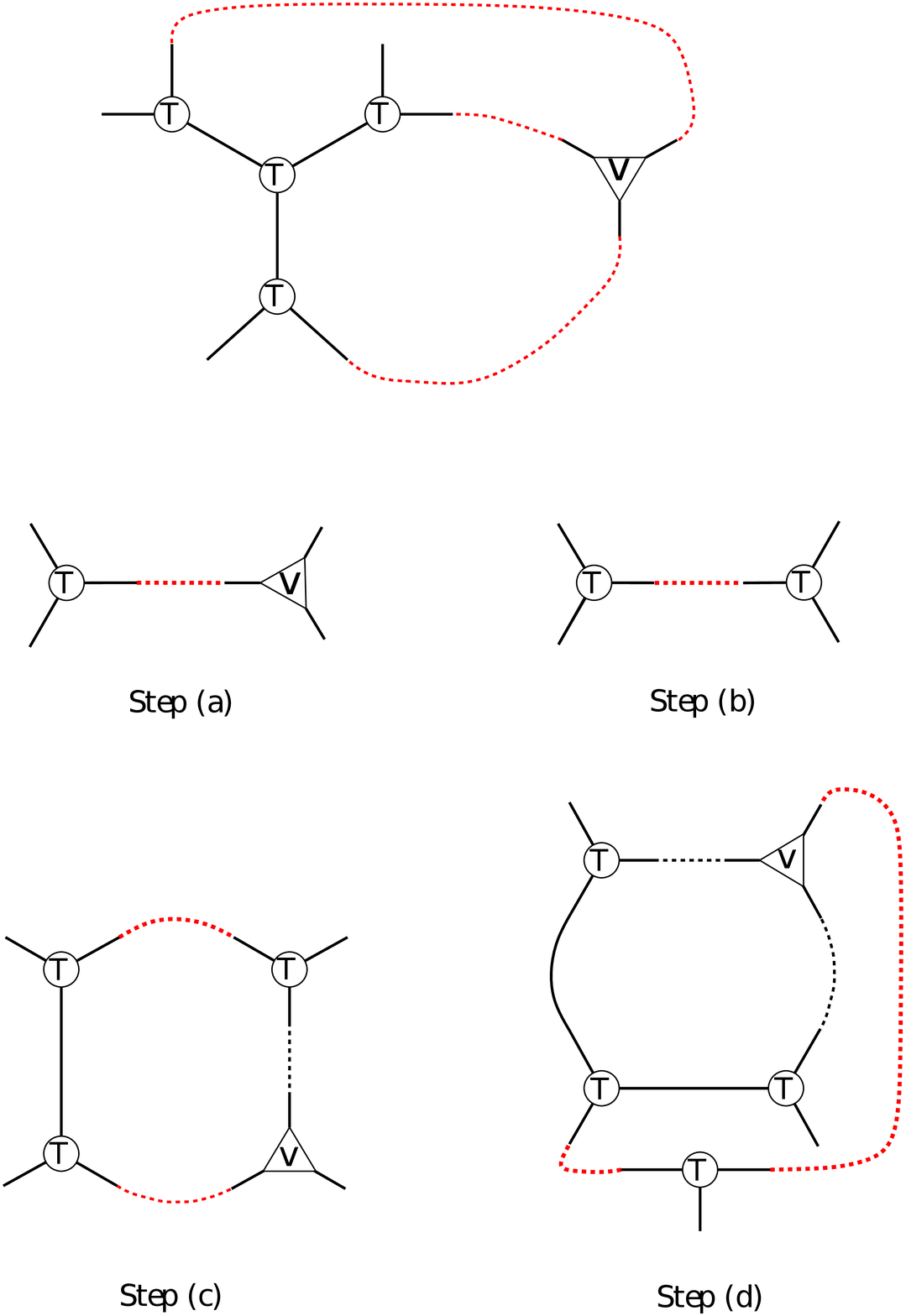}
	\caption{Step $a$ and $b$ can be seen as $(d^2\times d)$ times $(d\times d)$ matrix multiplications. Similarly, step $c$ (the bottleneck) and step $d$
		can be computed respectively as $(d^2\times d^2)$ times $(d^2\times d^2)$  and $(d^2\times d^2)$ times $(d^2\times d)$ matrix multiplications.} 
	\label{fig:power_iteration-ternary_tree}
\end{figure}
we may rewrite
\begin{align*}
	T_6 v &= \sum_{\substack{\Set{i\,, j\,,k\,, \ell}\in [n]^4}}\iprod{a_i, a_j}\iprod{a_i, a_k}\iprod{a_i, a_\ell}\paren{a_j\tensor a_k\tensor a_k}\transpose{\paren{a_j\tensor a_\ell\tensor a_\ell}}v\\
	&= \Brac{\Paren{\mathbf{T}_{\Set{1,2}\Set{3}}\mathbf{T}_{\Set{3}\Set{1,2}}} \Paren{\mathbf{T}_{\Set{1,2}\Set{3}} v_{\Set{3}\Set{1,2}}}}_{\Set{1,3}\Set{2,4}}\mathbf{T}_{\Set{1,2}\Set{3}}\,.
\end{align*}
In other words we may compute $T_6v$ using only a constant number of rectangular matrix multiplications, each of which has at most the complexity of a $d\times d^2$ times $d^2 \times d$ matrix multiplication!\footnote{Rectangular matrix multiplications of the form $d^k\times d^k$ times $d^k \times d$ can be reduced to rectangular matrix multiplication with dimension $d^k \times d$ \cite{GallU18}.} This approach can be even parallelized to compute the top $n$ eigenvectors of $T_6$ at the same time.

Upon obtaining this representation, we can perform basic operations (such as tensor contractions) required in the second part of the algorithm more quickly, further reducing the  running time of the algorithm.
Indeed, using the speed up described above, the algorithm based on the tensor network in \cref{fig:techniques-order-six-network} can be implemented in time $\tilde{O}\Paren{d^{2\omega\Paren{1+\frac{\log n}{2\log d}}}}$, which for $n=\Theta(d^{3/2}/\polylog(d))$ can be bounded by $\tilde{O}\Paren{d^{6.043182}}$.

\begin{remark}\label{remark:comparison-hss19}
We observe that applying the robust fourth-order tensor decomposition algorithm in \cite{HopkinsSS19} on the tensor network in \cref{fig:order-4-tensor-networks}(b) can recover ``\emph{a constant fraction, bounded away from 1},'' of the components, but not all of them, in $\tilde{O}(d^{6.5})$ time; see \cref{ap:hss19}.
In contrast, our algorithm based on the tensor network in \cref{fig:techniques-order-six-network} can recover ``\emph{all}'' the components in $\tilde{O}(d^{6.043182})$ time.
\end{remark}

 

\section{Fast and simple algorithm for third-order overcomplete tensor decomposition}
\label{section:non-robust-algorithm}

In this section, we present our fast algorithm for overcomplete tensor decomposition, which will be used to prove \cref{theorem:main}.  
Formally the algorithm is the following.

\crefname{enumi}{Step}{Steps}

\begin{algorithmbox}[Fast order-3 overcomplete tensor decomposition]
	\label{algorithm:non-robust-recovery}
	\mbox{}\\
	\noindent 
	
	\textbf{Input:} Tensor $\mathbf{T}=\underset{i \in[n]}{\sum} \tensorpower{a_i}{3}\,.$
	
	\textbf{Output:} Unit vectors $b_1,\ldots, b_n\in \R^d\,.$
	
	\begin{enumerate}
		\item \label{step:lifting} \textbf{Lifting:} Compute (as in \cref{algorithm:compute-implicit-representation}) the best rank-$n$ approximation $\hat{M}$ of the flattening $\mathbf{M}_{\Set{1,2,3},\set{4,5,6}}$ of the tensor network (\cref{fig:techniques-order-six-network})  %
			\begin{align*}
				\mathbf{M} = \underset{i,j,k,\ell\in [n]}{\sum}\iprod{a_i, a_j}\cdot \iprod{a_i, a_k}\cdot \iprod{a_i, a_\ell}\cdot \paren{\dyad{a_j}}				\tensor \paren{\dyad{a_k}} \tensor \paren{\dyad{a_\ell}}\,.
			\end{align*}
		
		\item \label{step:recovery} \textbf{Recovery:} Repeat $O(\log n)$ times:
		\begin{enumerate}
			\item \label{step:pre-processing}	\textbf{Pre-processing:} Project $\hat{M}$ into the space of matrices  in $\R^{d^3\times d^3}$ satisfying
			\begin{align*}
				\Norm{\hat{M}_{\Set{1,2,3,4}\Set{5,6}}}\leq 1\,, \quad \Norm{\hat{M}_{\Set{1,2,5,6}\Set{3,4}}}\leq 1\,.
			\end{align*}
			
			\item \label{step:rounding} \textbf{Rounding:} Run $\tilde{O}(d^2)$ independent trials of Gaussian Rounding on $\hat{M}$ contracting its first two modes to obtain a set of $0.99n$ candidate vectors $b_1, \ldots, b_{0.99n}$ (see~\cref{alg:rounding}).

			\item \label{step:boosting} \textbf{Accuracy boosting:} Boost the accuracy of each candidate $b_i$ via tensor power iteration.
			\item \label{step:peeling} \textbf{Peeling of recovered components:}
			\begin{itemize}
				\item Set $\hat{M}$ to be the best rank-$0.01n$ approximation of $\hat{M} -\sum_{i \leq 0.99n} \Paren{b_i^{\otimes 3}} \Paren{b_i^{\otimes 3}}^\top$
				\item Update $n \leftarrow 0.01n$.

			\end{itemize}
		\end{enumerate}
		
		\item Return all the candidate vectors $b_1, \ldots, b_n$ obtained above.
	\end{enumerate}
\end{algorithmbox}

As discussed before, the goal of the \hyperref[step:lifting]{Lifting step} is to compute an approximation of the sixth-order tensor $\sum_{i=1}^n a_i^{\otimes 6}$ and the goal of the \hyperref[step:recovery]{Recovery step} is to use this to recover the components.
To prove~\cref{theorem:main}, we will first prove that these two steps are correct and then argue about their running time.
Concretely, regarding the correctness of \cref{algorithm:non-robust-recovery} we prove the following two theorems:

\begin{theorem}[{Correctness of the \hyperref[step:lifting]{Lifting step}}]
\label{theorem:main-frobenius-norm-tensor-network}
Let $a_1,\ldots, a_n$ be \iid vectors sampled uniformly from the unit sphere in  $\R^d$ and consider%
\begin{align*}
		\mathbf M = \underset{i,j,k,\ell\in [n]}{\sum}\iprod{a_i, a_j}\cdot \iprod{a_i, a_k}\cdot \iprod{a_i, a_\ell}\cdot \paren{\dyad{a_j}}			\tensor \paren{\dyad{a_k}} \tensor \paren{\dyad{a_\ell}}\,.
\end{align*} 
Then, if $n\leq O(d^{3/2}/\polylog d)$ with overwhelming probability
	\[\mathbf{M}_{\Set{1,2,3},\set{4,5,6}}=\sum_{i\in[n]}a_i^{\otimes 3}\Paren{a_i^{\otimes 3}}^\top+E,\mbox{\quad where\quad}\norm{E}\leq \frac{1}{\polylog d}.\]
Moreover, let $\hat{M}$ be the best rank-$n$ approximation of $\mathbf{M}_{\Set{1,2,3},\set{4,5,6}}$ then
	\[\hat{M}=\sum_{i\in[n]}a_i^{\otimes 3}\Paren{a_i^{\otimes 3}}^\top+\hat{E},\mbox{\quad where\quad}\normf{\hat{E}}\leq \sqrt{8n} \cdot \norm{E} \text{, and } \norm{\hat{E}} \leq 2 \cdot \norm{E}.\]
\end{theorem}

\begin{remark}
Note that in the first display we identify $\mathbf M$ as a tensor and in the second display $M$ as a matrix. 
This should not lead to confusion as it should be clear from context which is meant and also from whether we use a bold or non-bold letter to denote it which is meant. 
\end{remark}

\begin{theorem}[{Correctness of the \hyperref[step:recovery]{Recovery step}}]
\label{thm:Full-recovery}
Let $a_1,\ldots, a_n$ be \iid vectors sampled uniformly from the unit sphere in  $\R^d$.
Given as input 
\begin{align*}
	\mathbf{T}=\sum_{i=1}^n a_i^{\otimes 3} \quad \text{and}\quad \hat{M}=\sum_{i=1}^n a_i^{\otimes 3}\Paren{a_i^{\otimes 3}}^\top+E\,, \text{with }\norm{E}\leq \epsilon \text{ and } \normf{E} \leq \e \sqrt{n}\,,
\end{align*}
the \hyperref[step:recovery]{Recovery step} of~\cref{algorithm:non-robust-recovery} returns unit norm vectors $b_1,b_2,\ldots,b_n$ satisfying 
 \begin{align*}
     \norm{a_i-b_{\pi(i)}}\leq \tilde{O}\Paren{\frac{\sqrt{n}}{d}}\,,
 \end{align*}
for some permutation $\pi:[n] \rightarrow [n]$.
\end{theorem}

\noindent Regarding the running time of the algorithm, we prove the result below.
\begin{theorem}
\label{thm:main-running-time}
\cref{algorithm:non-robust-recovery} can be implemented in time $\tilde{O}\Paren{d^{2\omega\Paren{1+\frac{\log n}{2\log d}}}+nd^4}$, where $d^{\omega(k)}$ is the time required to multiply a $(d^k\times d)$ matrix with a $(d\times d)$ matrix.
\end{theorem}

Combining the above three results directly yields a proof of~\cref{theorem:main}.
We will prove~\cref{theorem:main-frobenius-norm-tensor-network} in~\cref{section:lifting-tensor_network} and~\cref{thm:Full-recovery} over the course of Sections~\ref{section:robust-six-tensor-decomposition} and~\ref{section:non-robust-algorithm-full-recovery}, where~\cref{section:robust-six-tensor-decomposition} analyzes Steps~\hyperref[step:pre-processing]{2(a)} and~\hyperref[step:rounding]{2(b)} and~\cref{section:non-robust-algorithm-full-recovery} the rest.
Finally, in~\cref{section:simple-algorithm-running-time} we will prove~\cref{thm:main-running-time}.

\section{Lifting via tensor networks}\label{section:lifting-tensor_network}

In this section, we analyze the lifting part of our algorithm using tensor networks.
More precisely, we prove that the tensor network in~\cref{fig:techniques-order-six-network} has a large signal-to-noise ratio in the spectral norm sense, and that the noise of its corresponding top-$n$ eigenspace has a small Frobenius norm.   
Recall that our goal is to prove~\cref{theorem:main-frobenius-norm-tensor-network}:
\begin{theorem}[Restatement of~\cref{theorem:main-frobenius-norm-tensor-network}]
Let $a_1,\ldots, a_n$ be \iid vectors sampled uniformly from the unit sphere in  $\R^d$ and consider%
\begin{align*}
		\mathbf M = \underset{i,j,k,\ell\in [n]}{\sum}\iprod{a_i, a_j}\cdot \iprod{a_i, a_k}\cdot \iprod{a_i, a_\ell}\cdot \paren{\dyad{a_j}}			\tensor \paren{\dyad{a_k}} \tensor \paren{\dyad{a_\ell}}\,.
\end{align*} 
Then, if $n\leq O(d^{3/2}/\polylog d)$ with overwhelming probability
	\[\mathbf{M}_{\Set{1,2,3},\set{4,5,6}}=\sum_{i\in[n]}a_i^{\otimes 3}\Paren{a_i^{\otimes 3}}^\top+E,\mbox{\quad where\quad}\norm{E}\leq \frac{1}{\polylog d}.\]
Moreover, let $\hat{M}$ be the best rank-$n$ approximation of $\mathbf{M}_{\Set{1,2,3},\set{4,5,6}}$ then
	\[\hat{M}=\sum_{i\in[n]}a_i^{\otimes 3}\Paren{a_i^{\otimes 3}}^\top+\hat{E},\mbox{\quad where\quad}\normf{\hat{E}}\leq \sqrt{8n} \cdot \norm{E} \text{, and } \norm{\hat{E}} \leq 2 \cdot \norm{E}.\]
\end{theorem}

In~\cref{section:spectral-gap-ternary-tree} we will prove its first part and in~\cref{section:frobenius_norm-top-n-eigenspace}, we analyze the best rank-$n$ approximation of $M$ to prove the second part.

\subsection{Spectral gap of the ternary-tree tensor network}
\label{section:spectral-gap-ternary-tree}

In this section, we will prove the first part of~\cref{theorem:main-frobenius-norm-tensor-network}.
\begin{lemma}\label{lem:spectral_norm-ternary_tree}
	Consider the setting of~\cref{theorem:main-frobenius-norm-tensor-network}: If $n\leq O(d^{3/2}/\polylog d)$,  then with overwhelming probability
	\[\mathbf{M}_{\Set{1,2,3},\set{4,5,6}}=\sum_{i\in[n]}a_i^{\otimes 3}\Paren{a_i^{\otimes 3}}^\top+E,\mbox{\quad where\quad}\norm{E}\leq \frac{1}{\polylog d}.\]
\end{lemma}

\begin{proof}
	For ease of notation we denote by $M = \mathbf{M}_{\Set{1,2,3},\set{4,5,6}}$.
	To proof the theorem, we will split the sum into the part where some of the indices disagree and the part where all are equal.
	This second term (where $i=j=k=l$) gives exactly $\sum_{i\in[n]}a_i^{\otimes 3}\Paren{a_i^{\otimes 3}}^\top$.
	Hence, $E$ is the remaining part of the quadruple sum where not all indices are equal.
	We will analyze the spectral norm of this by further splitting the sum into parts where only some of the indices are equal.
	A clean way to conceptualize how we do this is as follows:
	Notice that each index in the sum comes from one node in the tensor network.
	Hence, we can think of \emph{coloring} the four nodes of the ternary tree tensor network using four colors.
	We map a giving coloring to a part of the sum as follows: If two nodes share the same color, we will take this to mean that the corresponding indices in the sum are equal, whereas if they have different colors, this should mean that the indices are different.
	For example, the coloring that all the four nodes share the same color corresponds to the matrix $\sum_{i\in[n]}a_i^{\otimes 3}\Paren{a_i^{\otimes 3}}^\top$. 	
	Whereas the one where say the middle node and one of the leaves have the same color and the remaining two leaves have two different colors (cf.~\cref{fig:ternary_tree-all_different} (b)) corresponds to
	\begin{align*}
	\sum_{i \in [n]} \snorm{a_i} a_i a_i^\top \otimes  \sum_{k \neq i} \iprod{a_i, a_k} a_k a_k^\top \otimes \sum_{\ell \neq k, i} \iprod{a_i, a_\ell} a_\ell a_\ell^\top
	\end{align*}
	
	Therefore, each coloring corresponds to a matrix,  and if we ignore permutations of colors (e.g. all nodes blue or all nodes red are identified as the same), since there are a constant number of colorings of the four nodes, the error matrix $E$ can be represented as a sum of a constant number of matrices, each of which corresponds to one coloring - again ignoring permutations of the colors.
	To bound the spectral norm of $E$, we can then bound each of the colorings independently.
	The colorings fall into three categories which we will analyze one by one.
	\begin{enumerate}
	\item All leaves have different colors (see~\cref{fig:ternary_tree-all_different})
	\item Two leaves share the same color, but the other leaf doesn't (see~\cref{fig:ternary_tree-two_same})
	\item All leaves share the same color, but the internal note has a different color
	\end{enumerate}

	\begin{figure}
		\centering
		\includegraphics[width=12cm]{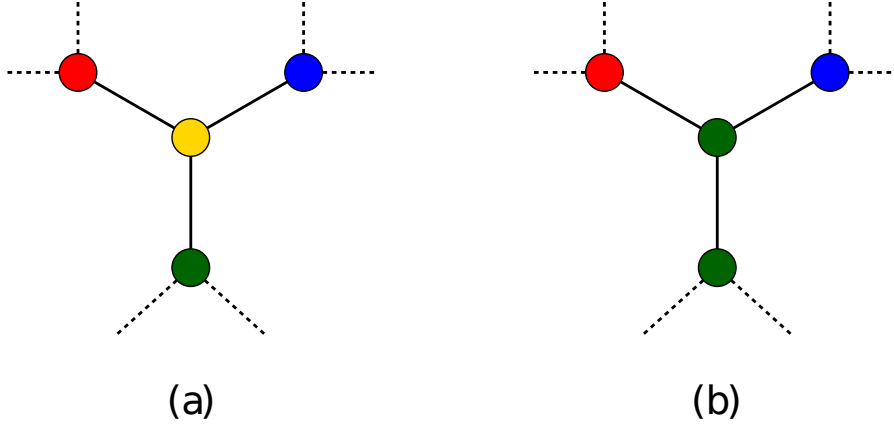}
		\caption{All leaves have a different color.}\label{fig:ternary_tree-all_different}
	\end{figure}

	\paragraph{First category.}
	We start with a detailed analysis for the coloring that all the four tensor nodes have different colors (\cref{fig:ternary_tree-all_different}(a)).
	This coloring corresponds to the following matrix
	\[M_{\mathrm{diff}}=\sum_{i\in [n]}\sum_{j\in[n],\,j\neq i} \Iprod{a_i, a_j}\cdot a_ja_j^\top\otimes \Paren{ \sum_{k\in[n],\, k\neq i,\,j}  \Iprod{a_i, a_k}\cdot a_ka_k^\top\otimes \Paren{\sum_{\ell\in[n],\, \ell\neq i,\,j,\,k}  \Iprod{a_i, a_\ell}\cdot a_\ell a_\ell^\top}}.\]
	To bound its spectral norm, we will use a decoupling argument:
	Let $s_1,\ldots, s_n$ be $n$ independent random signs.
	Since $a_i$ and $s_i\cdot a_i$ share the same distribution, analyzing $M_{\mathrm{diff}}$ is equivalent to analyzing
	\[\sum_{i\in [n]}s_i\cdot \sum_{j\in[n],\,j\neq i} s_j\cdot \Iprod{a_i, a_j}\cdot a_ja_j^\top\otimes \Paren{ \sum_{k\in[n],\, k\neq i,\,j}  s_k\cdot \Iprod{a_i, a_k}\cdot a_ka_k^\top\otimes \Paren{\sum_{\ell\in[n],\, \ell\neq i,\,j,\,k}  s_\ell\cdot \Iprod{a_i, a_\ell}\cdot a_\ell a_\ell^\top}}.\]
	To decouple the random signs in the above matrix, let $t_{i,j}$ for $1\leq i\leq 4$ and $1\leq j\leq n$ be $4n$ independent random signs, and define the following matrix
	\[\tilde{M}_{\mathrm{diff}}=\sum_{i\in [n]}t_{1,i}\cdot \sum_{j\in[n],\,j\neq i} t_{2,j}\cdot \Iprod{a_i, a_j}\cdot a_ja_j^\top\otimes \Paren{ \sum_{k\in[n],\, k\neq i,\,j}  t_{3,k}\cdot \Iprod{a_i, a_k}\cdot a_ka_k^\top\otimes \Paren{\sum_{\ell\in[n],\, \ell\neq i,\,j,\,k}  t_{4,\ell}\cdot \Iprod{a_i, a_\ell}\cdot a_\ell a_\ell^\top}}.\]
	By~\cref{thm:decoupling-inequality}, w.ov.p.,
	\begin{equation}\label{eq:coloring-different-decoupling}
		\Norm{M_{\mathrm{diff}}}=\tilde{O}\Paren{\Norm{\tilde{M}_{\mathrm{diff}}}}. 
	\end{equation}
	
	It hence suffices to analyze $\Norm{\tilde{M}_{\mathrm{diff}}}$.
	To simplify notation,
	define the following matrices
	\begin{align*}
	&N_{i,j,k} \coloneqq \sum_{\ell\in[n],\, \ell\neq i,\,j,\,k}  t_{4,\ell}\cdot \Iprod{a_i, a_\ell}\cdot a_\ell a_\ell^\top \\
	&N_{i,j} \coloneqq \sum_{k\in[n],\, k\neq i,\,j}  t_{3,k}\cdot \Iprod{a_i, a_k}\cdot a_ka_k^\top\otimes N_{i,j,k} \\
	&N_i \coloneqq \sum_{j\in[n],\,j\neq i} t_{2,j}\cdot \Iprod{a_i, a_j}\cdot a_ja_j^\top\otimes N_{i,j}
	\end{align*}
	First, by a Matrix Rademacher bound (\cref{thm:matrix_rademacher}) and by Triangle inequality we get
	\begin{align}
		\Norm{\tilde{M}_{\mathrm{diff}}} = \sum_{i\in[n]} t_{1,i}\cdot N_i \overset{w.ov.p}{\leq} \tilde{O}\Paren{\Norm{\sum_{i\in[n]} N_i^2}}^{1/2} \leq\tilde{O}\Paren{\sqrt{n}}\cdot \max_{i \in [n]} \Norm{N_i}. \label{eq:coloring-different-first}
	\end{align}

	Second, by \cref{lem:decoupling-tensor_product-Hermitian_matrices} and by \cref{cor:random_sign-concentration-input-vectors}(a)-(b) we have that for all $i$, 
	\begin{align}
		\Norm{N_i} & \;\;\,=\;\;\, \Norm{\sum_{j\in[n],\,j\neq i} t_{2,j}\cdot \Iprod{a_i, a_j}\cdot a_ja_j^\top\otimes N_{i,j}} \nonumber\\
		& \overset{w.ov.p.}{\leq}\tilde{O}\Paren{\Paren{\max_{j\in[n],\,j\neq i}\Norm{N_{i,j}}}\cdot\Norm{\sum_{j\in[n],\,j\neq i}\Paren{\Iprod{a_i,a_j}\cdot a_ja_j^\top}^2}^{1/2}}\nonumber \\
		& \;\;\,=\;\;\, \max_{j\in[n],\,j\neq i}\Norm{N_{i,j}} \cdot \tilde{O}\Paren{\Norm{\sum_{j\in[n],\,j\neq i}\Iprod{a_i,a_j}^2\cdot a_ja_j^\top}^{1/2}} \nonumber \\
		& \;\;\,\leq\;\;\, \max_{j\in[n],\,j\neq i}\Norm{N_{i,j}} \cdot \tilde{O}\Paren{\max_{j\in[n],\,j\neq i}\Abs{\Iprod{a_i,a_j}} \cdot \Norm{\sum_{j\in[n],\,j\neq i} a_ja_j^\top}^{1/2}} \nonumber \\
		& \overset{w.ov.p.}{\leq} \max_{j\in[n],\,j\neq i}\Norm{N_{i,j}} \cdot  \tilde{O}\Paren{\sqrt{\frac{n}{d^2}}} \label{eq:coloring-different-second}
	\end{align}
	
	By the same reasoning as above we get that for all $i \neq j$,
	\begin{equation}\label{eq:coloring-different-third}
		\Norm{N_{i,j}}\overset{w.ov.p.}{\leq} \tilde{O}\Paren{\sqrt{\frac{n}{d^2}}}\cdot \max_{k\in[n],\, k\neq i,\,j}\Norm{N_{i,j,k}}\overset{w.ov.p.}{\leq}\tilde{O}\Paren{\sqrt{\frac{n^2}{d^4}}}
	\end{equation}
	where the last inequality follows from a Matrix Rademacher bound, similar steps as above, and a union bound over all $k \neq i,j$.
	
	Combining \cref{eq:coloring-different-decoupling}, \cref{eq:coloring-different-first}, \cref{eq:coloring-different-second} and \cref{eq:coloring-different-third} and two more union bounds over $i$ and $j \neq i$ (i.e., $\max$ in \cref{eq:coloring-different-first} and \cref{eq:coloring-different-second}), we finally obtain,
	\begin{equation}\label{eq:coloring-different}
		\Norm{M_{\mathrm{diff}}}\overset{w.ov.p.}{\leq} \tilde{O}\Paren{\sqrt{n}}\cdot \tilde{O}\Paren{\sqrt{\frac{n}{d^2}}}\cdot \tilde{O}\Paren{\sqrt{\frac{n^2}{d^4}}} =\tilde{O}\Paren{\sqrt{\frac{n^4}{d^6}}} = \frac{1}{\polylog d}
	\end{equation}
	
	Next, we discuss the second coloring in the first category.
	As seen before the matrix corresponding to~\cref{fig:ternary_tree-all_different}(b) looks as follows:
	\begin{align*}
	\sum_{i \in [n]} \snorm{a_i} a_i a_i^\top \otimes  \sum_{k \neq i} \iprod{a_i, a_k} a_k a_k^\top \otimes \sum_{\ell \neq k, i} \iprod{a_i, a_\ell} a_\ell a_\ell^\top
	\end{align*}
	 Again considering $s_i a_i$ instead of $a_i$ for independent random signs and invoking~\cref{thm:decoupling-inequality} it suffices to bound the spectral norm of
	 \begin{align*}
	 \sum_{i \in [n]}  a_i a_i^\top \otimes  \sum_{k \neq i} t_{1,k} \iprod{a_i, a_k} a_k a_k^\top \otimes \sum_{\ell \neq k, i} t_{2,\ell}\iprod{a_i, a_\ell} a_\ell a_\ell^\top
	 \end{align*}
	where $t_{i,j}$ for $i = 1,2, j \in [n]$ are independent random signs.
	Similarly as before and overloading notation, we define $N_{i,k} \coloneqq \sum_{\ell \neq k, i} t_{2,\ell}\iprod{a_i, a_\ell} a_\ell a_\ell^\top$ and $N_i \coloneqq \sum_{k \neq i} t_{1,k} \iprod{a_i, a_k} a_k a_k^\top \otimes N_{i,k}$.
	First, using~\cref{lemma:lin_alg_fact_1} with the fact that $a_i a_i^\top$ is a psd matrix we get that the spectral norm of this is at most
	\begin{align*}
	\Norm{ \sum_{i \in [n]}  a_i a_i^\top \otimes N_i } \leq \Paren{ \max_{i \in [n]} \Norm{N_i}} \cdot \Norm{\sum_{i \in [n]} a_i a_i^\top}^{1/2} \leq \tilde{O} \Paren{\sqrt{\frac{n}{d}}} \cdot \max_{i \in [n]} \Norm{N_i}
	\end{align*}
	where the last inequality follows by~\cref{lem:concentration-input-vectors} (b).
	Using the same reasoning as in~\cref{eq:coloring-different-second} and a union bound over all $i$ we get that
	\begin{align*}
	\max_{i \in [n]} \Norm{N_i} \leq \tilde{O} \Paren{\sqrt{\frac{n}{d^2}}} \cdot \max_{k \in [n], k \neq i} \Norm{N_{i,k}} \leq \tilde{O} \Paren{\sqrt{\frac{n}{d^2}}} \cdot \tilde{O} \Paren{\sqrt{\frac{n}{d^2}}} = \tilde{O} \Paren{\sqrt{\frac{n^2}{d^4}}} = \frac{1}{\polylog d}
	\end{align*}
	where the last inequality again uses a Matrix Rademacher bound (and a union bound over all $k$).
	Putting things together, we get that the spectral norm we wanted to bound originally is at most $\tilde{O}(\frac{n}{d^{3/2}})$.
	
	\begin{figure}
		\centering
		\includegraphics[width=16cm]{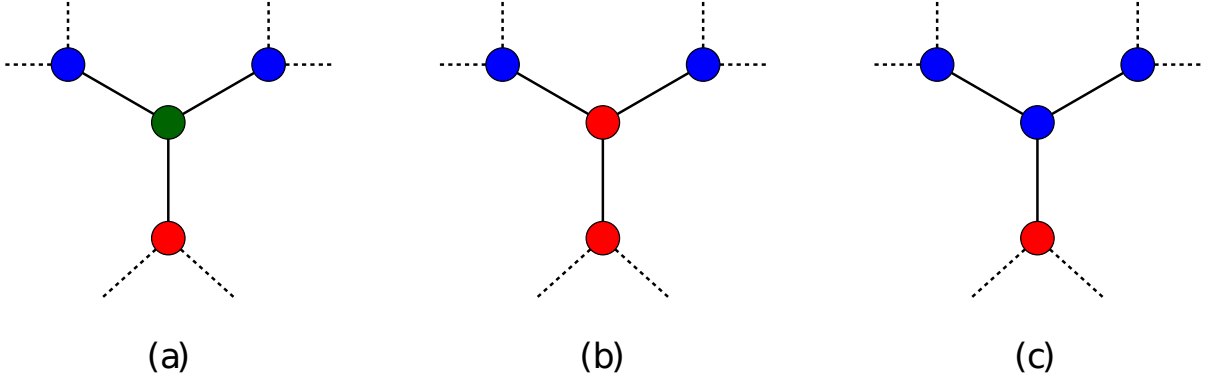}
		\caption{Two leaves share the same color but the other leaf does not.}\label{fig:ternary_tree-two_same}
	\end{figure}

	For completeness we will also supply the proofs for the second and third category although they are very similar to the above.
	
	\paragraph{Second category.}
	Since we will always first multiply the $a_i$'s by random sign and then apply the decoupling theorem we will omit this step below.
	We will also us analogous notation.
	\cref{fig:ternary_tree-two_same} shows the three cases for the second category with which we will start. 
	For (a), the matrix looks as follows:
	\begin{align*}
	\sum_{i \in [n]} t_{1,i} \sum_{j \in [n], j \neq i} t_{2,j} \iprod{a_i, a_j} a_j a_j \otimes \sum_{k \in [n], k \neq i,j} \iprod{a_i, a_k}^2 (a_k^{\otimes 2})(a_k^{\otimes 2})^\top
	\end{align*}
	Define $N_{i,j} \coloneqq \sum_{k \in [n], k \neq i,j} \iprod{a_i, a_k}^2 (a_k^{\otimes 2})(a_k^{\otimes 2})^\top$ and $N_i \coloneqq \sum_{j \in [n], j \neq i} t_{2,j} \iprod{a_i, a_j} a_j a_j \otimes N_{i,j}$.
	Then similarly as before, we get
	\begin{align*}
	\Norm{\sum_{i \in [n]} t_{1,i} N_i} \leq \tilde{O} \Paren{\sqrt{n}} \cdot \max_{i \in [n]} \Norm{N_i} \leq \tilde{O} \Paren{\sqrt{n}} \cdot \tilde{O}\Paren{\sqrt{\frac{n}{d^2}}} \cdot \max_{i,j \in [n], i \neq j} \Norm{N_{i,j}}
	\end{align*}
	To bound the last term, we notice that for each $i \neq j$ we have that w.ov.p.
	\begin{align*}
	\norm{N_{i,j}} \leq \max_{k \in [n], k \neq i,j} \iprod{a_i, a_k}^2 \Norm{\sum_{k \in [n], k \neq i,j} (a_k^{\otimes 2})(a_k^{\otimes 2})^\top} \leq \tilde{O} \Paren{\frac{1}{d} \cdot \frac{n}{d}} = \tilde{O} \Paren{\frac{n}{d^2}}
	\end{align*}
	Using a last union bound, we get that the spectral norm of the term corresponding to this coloring is at most $\tilde{O} \Paren{\frac{n^2}{d^3}} = \frac{1}{\polylog d}$.
	
	For~\cref{fig:ternary_tree-two_same} (b) the matrix looks like
	\begin{align*}
	\sum_{i \in [n]} a_i a_i^\top \otimes \sum_{j \in [n],  j \neq i} \iprod{a_i,a_j}^2 (a_j^{\otimes 2})(a_j^{\otimes 2})^\top
	\end{align*}
	Defining $N_i \coloneqq \sum_{j \in [n],  j \neq i} \iprod{a_i,a_j}^2 (a_j^{\otimes 2})(a_j^{\otimes 2})^\top$ and using~\cref{lemma:lin_alg_fact_1} we can bound the spectral norm of this as
	\begin{align*}
	\Norm{\sum_{i \in [n]} a_i a_i^\top} \cdot \max_{i \in [n]} \Norm{N_i} &\leq \tilde{O} \Paren{\frac{n}{d}} \cdot \Paren{\max_{i,j\in [n], i \neq j} \iprod{a_i, a_j}^2} \cdot \Norm{\sum_{j \in [n],  j \neq i} (a_j^{\otimes 2})(a_j^{\otimes 2})^\top} \\
	&\leq \tilde{O} \Paren{\frac{n}{d} \cdot \frac{1}{d} \cdot \frac{n}{d}} = \tilde{O} \Paren{ \frac{n^2}{d^3} } = \frac{1}{\polylog d}
	\end{align*}
	
	For~\cref{fig:ternary_tree-two_same} (c) the matrix resulting matrix is
	\begin{align*}
	\sum_{i \in [n]} (a_i^{\otimes 2})(a_i^{\otimes 2})^\top \otimes \sum_{j \in [n], j \neq i} t_{1,j} \iprod{a_i, a_j} a_j a_j^\top
	\end{align*}
	Again using~\cref{lemma:lin_alg_fact_1} and a Matrix Rademacher bound we bound the spectral norm of this term as follows:
	\begin{align*}
	\Norm{\sum_{i \in [n]} (a_i^{\otimes 2})(a_i^{\otimes 2})^\top} \cdot \max_{i \in [n]} \Norm{\sum_{j \in [n], j \neq i} t_{1,j} \iprod{a_i, a_j} a_j a_j^\top} &\leq \tilde{O} \Paren{\frac{n}{d}} \cdot \Paren{\max_{i,j \in [n], i \neq j} \iprod{a_i, a_j}} \cdot \max_{i \in [n]} \Norm{\sum_{j \in [n], j \neq i} a_j a_j^\top}^{1/2} \\
	&\leq \tilde{O} \Paren{\frac{n}{d} \cdot \frac{1}{\sqrt{d}} \cdot \sqrt{\frac{n}{d}}} = \tilde{O} \Paren{\sqrt{\frac{n^2}{d^3}}} = \frac{1}{\polylog d}
	\end{align*}
	
	\paragraph{Third category.}
	The last missing case is the one in the third category, where all three leaves have the same color but the internal node has a different one.
	In this case, the matrix we consider is
	\begin{align*}
	\sum_{i \in [n]} t_{1,i} \sum_{j \in [n], j \neq i} \iprod{a_i, a_j}^3 (a_j^{\otimes 3}) (a_j^{\otimes 3})^\top
	\end{align*}
	Using a Matrix Rademacher bound, Triangle Inequality, and~\cref{lem:concentration-input-vectors} (c) we bound its spectral norm by
	\begin{align*}
	\tilde{O}(\sqrt{n}) \cdot \max_{i \in [n]} \Norm{\sum_{j \in [n], j \neq i} \iprod{a_i, a_j}^6 (a_j^{\otimes 3}) (a_j^{\otimes 3})^\top}^{1/2} &\leq \tilde{O}(\sqrt{n}) \cdot \max_{i,j \in [n], i \neq j} \abs{\iprod{a_i, a_j}}^3 \cdot \max_{i \in [n]} \Norm{\sum_{j \in [n], j \neq i} (a_j^{\otimes 3}) (a_j^{\otimes 3})^\top} \\
	&\leq \tilde{O} \Paren{\sqrt{n} \cdot \frac{1}{\sqrt{d^3}} \cdot 1} = \tilde{O} \Paren{\sqrt{\frac{n}{d^3}}} = \frac{1}{\polylog d}
	\end{align*}

\end{proof}

\subsection{From spectral norm error to frobenius norm error}
\label{section:frobenius_norm-top-n-eigenspace}

In this section our goal is to prove the second part of~\cref{theorem:main-frobenius-norm-tensor-network}.
More precisely, we will show the following lemma:

\begin{lemma}\label{lem:frobenius_norm-top_n-ternary_tree}
	Let $M = \sum_{i\in[n]}a_i^{\otimes 3}\Paren{a_i^{\otimes 3}}^\top+E$, where $a_1, \ldots, a_n$ are \iid vectors uniformly sampled from the unit sphere in $\R^d$ and $\norm{E} \leq \e$.
	Let $\hat{M}$ be the best rank-$n$ approximation of $M$, i.e., $\hat{M}=\sum_{i\in[n]}\lambda_iv_iv_i^\top$ where $\lambda_i$'s are the top $n$ eigenvalues of $M$ and $v_i$'s are the corresponding eigenvectors.
		Then
	\[\hat{M}=\sum_{i\in[n]}a_i^{\otimes 3}\Paren{a_i^{\otimes 3}}^\top+\hat{E},\mbox{\quad where\quad}\normf{\hat{E}}\leq \sqrt{8n} \cdot \norm{E}\quad \text{ and } \quad \norm{\hat{E}} \leq 2 \cdot \norm{E}\]
\end{lemma}
\begin{proof}
	Define $S = \sum_{i\in[n]}a_i^{\otimes 3}\Paren{a_i^{\otimes 3}}^\top$, then $M = S + E$.
	Also, define $\hat{E} = \hat{M} - S$, then our goal will be to bound $\norm{\hat{E}}$ and $\normf{\hat{E}}$.
	Since $\hat{M}$ is the best rank-$n$ approximation of $M$ we know that $\norm{M - \hat{M}} \leq \norm{M - S} = \norm{E}$.
	We hence get
	\begin{align*}
	\norm{\hat{E}} = \norm{\hat{M} - S} \leq \norm{\hat{M} - M} + \norm{M - S} \leq 2 \cdot \norm{E}
	\end{align*}
	Further, since both $S$ and $\hat{M}$ have rank $n$, the rank of $\hat{M} - S$ is at most $2n$, and it follows that
	\begin{align*}
	\normf{\hat{E}} = \normf{\hat{M} - S} \leq \sqrt{2n} \cdot \norm{\hat{M} - S} \leq \sqrt{8n} \cdot \norm{E}
	\end{align*}

\end{proof}

\section{Recovering a constant fraction of the components
 using robust order-6 tensor decomposition}
\label{section:robust-six-tensor-decomposition}

The goal of this section is to prove that in each iteration of the~\hyperref[step:recovery]{Recovery step} in \cref{algorithm:non-robust-recovery}, Steps~\hyperref[step:pre-processing]{2(a)} and~\hyperref[step:rounding]{2(b)} recover a $0.99$ fraction of the remaining components up to constant correlation.
More precisely, we will show the following theorem:
\begin{theorem}[Recovery for constant fraction of component vectors]\label{thm:partial-recovery-main}
    Let $n \leq O\Paren{d^{3/2}/\polylog(d)}$, let $a_1,a_2,\ldots,a_{n}\in\mathbb{R}^d$ be independently  and uniformly sampled from the unit sphere, and let $\e\leq \frac{1}{\polylog(d)}$.
     There exists an algorithm (\cref{alg:rounding} below) that with high probability over 
     $a_1,a_2,\ldots,a_{n}$, for $d\leq n'\leq n$, for any subset 
     $S_0\subseteq [n]$ of size $n'$ and for a matrix $\hat{M}$ satisfying
     \begin{align*}
     	\Normf{\hat{M} - \sum_{i\in S_0} a_i^{\otimes 3}\Paren{a_i^{\otimes 3}}^\top} \leq \e\sqrt{n'} = \frac{\sqrt{n'}}{\polylog d},
     \end{align*}
 returns unit vectors  $b_1,b_2,\ldots,b_m\in \mathbb{R}^d$ for $m \geq 0.99 n'$
     such that for each $j \in [m]$ there exists a unique $i\in S_0$  with $\iprod{b_j, a_i} \geq 0.99$.

 \end{theorem}
The algorithm looks as follows:
\begin{algorithmbox}[Rounding step]
\label{alg:rounding}
\mbox{}\\
	\noindent 
	\textbf{Input:} A matrix $\hat{M} \in \R^{d^3 \times d^3}$ such that
	\begin{align*}
	\Norm{\hat{M} - \sum_{i\in S_0} a_i^{\otimes 3}\Paren{a_i^{\otimes 3}}^\top} \leq \e \sqrt{n}  %
	\end{align*}
	where $a_1, \ldots, a_n$ are \iid sampled uniformly from the unit sphere, $S_0 \subseteq [n]$ of size $n'$,
     and $\epsilon=\frac{1}{\polylog(d)}$.
    
     \noindent
	\textbf{Output:} A set $S$ of unit vectors $b_1, \ldots, b_m$ where $m \geq 0.99n'$
	
	\begin{description}
	\item[Spectral truncation]
    \label{step:spectral-truncation-rounding} (\emph{Corresponds to~\hyperref[step:pre-processing]{Step 2(a)} of~\cref{algorithm:non-robust-recovery}})
	\begin{enumerate}[(1).]
        \item Compute $\hat{M'}$ the projection of $\hat{M}_{\Set{1,2,3,4}\Set{5,6}}$ 
        into the set of $d^4\times d^2$ matrices with spectral norm bounded by $1$.
        \item Compute $M^{\le 1}$ the projection of $\hat{M'}_{\Set{1,2,5,6}\Set{3,4}}$ into the 
        set of $d^4\times d^2$ matrices 
        with spectral norm bounded by $1$.
    \end{enumerate}

	\item[Gaussian rounding]\label{step:random-contraction-rounding}
    	\item Initialize $C \leftarrow \emptyset$.
    Repeat $\tilde{O}(d^2)$ times:
	\begin{enumerate}[(1).]
		\item Sample $g \sim N(0, \Id_{d^2})$ and compute 
        $M_g = \Paren{g\otimes \Id_{d^2}\otimes \Id_{d^2}}\flattent{\mathbf M^{\le 1}}{1,2}{3,4}{5,6}$.
		\item Compute the top right singular vector of $M_g$ denoted by $u\in \R^{d^2}$ and
		flatten it into square matrix $U\in\mathbb{R}^{d\times d}$.
		\item Compute the top left and right singular vectors of $U$
        denoted by $v_l, v_r \in \R^d$.
		\item For $b\in \{\pm v_{L},\pm v_{R}\}$: 
		
        \qquad If $\iprod{T,b^{\otimes 3}}\geq 1-\frac{1}{\polylog(n)}$

        \qquad\qquad Add $b$ to $C$
        
        \item For $b\in C$:
            
        \qquad if $\iprod{b, b'}\geq 0.99$ for all $b' \in S$
        
        \qquad \qquad add $b$ to $S$
	\end{enumerate}
	\item Output $S$
	\end{description}
\end{algorithmbox}

We will prove~\cref{thm:partial-recovery-main} in several steps.
Our strategy will be to apply so-called Gaussian rounding, a version of Jennrich's algorithm.
However, to make this succeed in the presence of the noise matrix $E$, we will need control the spectral norm of this reshaping.
In~\cref{sec:spectral-truncation-correctness}  we will show that this can be done by truncating all large singular values of the respective reshapings, 
Concretely, we will show the following:
\begin{lemma}[Spectral truncation]
\label{lem:spectral-truncation-correctness}
Let $n \leq O\Paren{d^{3/2}/\polylog(d)}$, let $a_1,a_2,\ldots,a_{n}\in\mathbb{R}^d$ be independently and uniformly sampled from the unit sphere, and  let $\e\leq \frac{1}{\polylog(d)}$. 
Then, for $d\leq n'\leq n$, for every $S_0 \subseteq [n]$ of size $n'$ and for a matrix $\hat{M}\in \mathbb{R}^{d^3 \times d^3}$ satisfying $\normf{\hat{M} - \sum_{i\in S_0} a_i^{\otimes 3}\Paren{a_i^{\otimes 3}}^\top}\leq \e \sqrt{n'}$, the \hyperref[step:spectral-truncation-rounding]{Spectral truncation} step of~\cref{alg:rounding} transforms $\hat{M}$ into tensor $\mathbf M^{\le 1}$ such that 
       \begin{itemize}
           \item the spectral norm of rectangular flattening is bounded by $1$:
           \begin{align*}
            \Norm{ \mathbf{M}^{\le 1}_{\Set{1,2,3,4}\Set{5,6}}}\leq 1\, \quad\mbox{and}\quad \Norm{ \mathbf{M}^{\le 1}_{\Set{1,2,5,6}\Set{3,4}}} \leq 1,
            \end{align*}
            \item and for $R=\sqrt{2}\cdot\Paren{\E_{a\sim N(0,\Id_d)} (aa^\top)^{\otimes 2}}^{+1/2}$, with high probability over $a_1,a_2,\ldots,a_{n}$, $\mathbf M^{\le 1}$ is close to $\mathbf{S}=\sum_{i \in S_0} \Paren{Ra_i^{\otimes 2}}^{\otimes 3}$ in Frobenius norm: $\Norm{\mathbf{M}^{\le 1}-\mathbf{S}}_F\leq 
            3\epsilon\sqrt{n'}$.
       \end{itemize}
\end{lemma}

Given this, we will prove the correctness of the rounding part 
in~\cref{sec:gaussian-rounding} and prove the following lemma:
\begin{lemma}
\label{lem:gaussian-rounding-correctness}
    Let$n \leq O\Paren{d^{3/2}/\polylog(d)}$, let $a_1,a_2,\ldots,a_{n}\in\mathbb{R}^d$ be independently 
    and uniformly sampled from the unit sphere, and let $\e\leq \frac{1}{\polylog(d)}$. 
    Then, with high probability over $a_1,a_2,\ldots,a_{n}$, for $d\leq n'\leq n$ and for any $S_0\subseteq [n]$ of size $n'$, 
     given any $\mathbf M^{\le 1}\in \mathbb{R}^{d^2\times d^2\times d^2}$ such that 
     \begin{align*}
    \normf{\mathbf M^{\le 1} - \sum_{i\in S_0}\Paren{(Ra_i^{\otimes 2}}^{\otimes 3} } 
    \leq \e \sqrt{n'} \quad \text{ and } \quad \Norm{\mathbf{M}^{\le 1}_{\Set{1,2,3,4}\Set{5,6}}}, 
    \Norm{\mathbf{M}^{\le 1}_{\Set{1,2,5,6}\Set{3,4}}}\leq 1,
    \end{align*}
   the \hyperref[step:random-contraction-rounding]{Gaussian rounding} 
   step of~\cref{alg:rounding} outputs unit vectors
    $b_1,b_2,\ldots,b_m\in \mathbb{R}^d$ for $m \geq 0.99n'$
    such that for each $j\in [m]$ there exists a unique  $i \in S_0$
    with $\iprod{b_j, a_i} \geq 0.99$.
\end{lemma}
Combining the two above theorems directly proves~\cref{thm:partial-recovery-main}.
However, there are two technical subtleties in the proof.

\paragraph{Subsets of components need not be independent.} Second, it might be the case that a selected
 subset of the algorithm of independent random vectors are not independent.
To overcome this difficulty, we instead introduce the following more general definition:

\begin{definition}[Nicely-separated vectors]\label{def:nicely-separated}
    Let $R=\sqrt{2}\Paren{\E_{a\sim N(0,\Id_d)} \Paren{aa^\top}^{\otimes 2}}^{+1/2}$. The set of vectors 
    $a_1,a_2,\ldots,a_n'$ is called $(n,d)$-nicely-separated if all of the following
     are satisfied.
    \begin{enumerate}
        \item $\Norm{\sum_{i\in [n']} a_i^{\otimes 3}\Paren{a_i^{\otimes 3}}^\top}=1\pm o(1)$
        \item $\Norm{\sum_{i\in [n']} a_i^{\otimes 2}\Paren{a_i^{\otimes 2}}^\top}=
        \tilde{O}\Paren{\frac{n}{d}}$
        \item $\Norm{\sum_{i\in [n']} a_ia_i^\top}=\tilde{O}\Paren{\frac{n}{d}}$
        \item For any $S\subseteq [n']$ with size at least $d$,
        $$\Norm{\sum_{i\in S} Ra_i^{\otimes 2}\Paren{Ra_i^{\otimes 2}}^\top-\Pi}
        =1\pm \tilde{O}\Paren{\frac{n}{d^{3/2}}}$$, 
        where $\Pi$ is the projection matrix into the span of
        $\Set{Ra_i^{\otimes 2}:i\in S}$
        \item For each $j\in [n']$, $\sum_{i\in [n']\setminus \{j\}} \Iprod{Ra_i^{\otimes 2},Ra_j^{\otimes 2}}^2\leq 
        \tilde{O}\Paren{\frac{n}{d^2}}$
        \item For $i\in [n']$, $\Norm{Ra_i^{\otimes 2}-a_i^{\otimes 2}}^2=\tilde{O}\Paren{\frac{1}{d}}$
        \item For $i\in [n']$, $\Norm{a_i}=1\pm \tilde{O}\Paren{\frac{1}{\sqrt{d}}}$
        \item For $i,j\in [n']$, $\iprod{a_i,a_j}^2\leq \tilde{O}\Paren{\frac{1}{d}}$
    \end{enumerate}   
\end{definition}

It can be verified that with high probability, when the component vectors
are independently and uniformly sampled from the unit sphere, with high probability 
any subset of them is nicely-separated.
In fact, we prove the following lemma in~\cref{sec:proof-nicely-separated}.
 \begin{lemma}[Satisfaction of separation assumptions]
    \label{lem:component-assumptions}
    With probability at least $1-o(1)$ over the random vectors
    $a_1,a_2,\ldots,a_n\in \mathbb{R}^d$ independently and 
    uniformly sampled from the unit sphere, for every 
    $S\subseteq [n]$,
    the set of vectors $\{a_i: i\in S\}$ is 
    $(n,d)$-nicely separated.
 \end{lemma}
\noindent It is hence enough to proof~\cref{thm:partial-recovery-main} for the case when the subset of components indexed by $S_0$ is $(n,d)$-nicely separated. 

\paragraph{Isotropic components.} First, for this analysis to work we need to assume that the squared components ($a_i^{\otimes 2}$) are in isotropic position.
That is, we would like to rewrite the tensor $\sum_{i \in S_0} a_i^{\otimes 6}$ as  $\sum_{i \in S_0} (R a_i^{\otimes 2})^{\otimes 3}$ where $\sqrt{2}\cdot \Paren{\E_{a\sim N(0,\Id_d)} a^{\otimes 2}\Paren{a^{\otimes 2}}^\top}^{+1/2}$.
The following theorem shows that we can do this without loss of generality.
\begin{lemma}\label{lem:equivalence-isotropic}
    Let $n \leq O\Paren{d^{3/2}/\polylog(d)}$, let $n'\leq n$, let 
    $a_1,a_2,\ldots,a_{n'}\in \mathbb{R}^d$ be $(n,d)$-nicely-separated, and let $R=\sqrt{2}\cdot \Paren{\E_{a\sim N(0,\Id_d)} a^{\otimes 2}\Paren{a^{\otimes 2}}^\top}^{+1/2}$.
    For any tensor $\hat{\mathbf M}=\sum_{i=1}^{n} a_i^{\otimes 6}+\mathbf E$ with 
    $\norm{\mathbf E}_F\leq  \tilde{O}\Paren{\frac{n}{d^{3/2}}}\cdot \sqrt{n'}$, we have 
    $$\Norm{\hat{\mathbf M}-\sum_{i=1}^{n'}\Paren{Ra_i^{\otimes 2}}^{\otimes 3}}_F\leq \tilde{O}\Paren{\frac{n}{d^{3/2}}}\cdot \sqrt{n'}\,.$$ 
\end{lemma}
\noindent We will give a proof in~\cref{sec:isotropic-components}

\subsection{Spectral truncation}
\label{sec:spectral-truncation-correctness}
The goal of this section is to prove~\cref{lem:spectral-truncation-correctness} which we restate below:

\begin{lemma}[Restatement of~\cref{lem:spectral-truncation-correctness}]
    Let $n \leq O\Paren{d^{3/2}/\polylog(d)}$, let $a_1,a_2,\ldots,a_{n}\in\mathbb{R}^d$ be independently and uniformly sampled from the unit sphere, and  let $\e\leq \frac{1}{\polylog(d)}$. 
    Then, for $d\leq n'\leq n$, for every $S_0 \subseteq [n]$ of size $n'$ and for a matrix $\hat{M}\in \mathbb{R}^{d^3 \times d^3}$ satisfying $\normf{\hat{M} - \sum_{i\in S_0} a_i^{\otimes 3}\Paren{a_i^{\otimes 3}}^\top}\leq \e \sqrt{n'}$, the \hyperref[step:spectral-truncation-rounding]{Spectral truncation} step of~\cref{alg:rounding} transforms $\hat{M}$ into tensor $\mathbf M^{\le 1}$ such that 
    \begin{itemize}
    	\item the spectral norm of rectangular flattening is bounded by $1$:
    	\begin{align*}
    		\Norm{ \mathbf{M}^{\le 1}_{\Set{1,2,3,4}\Set{5,6}}}\leq 1\, \quad\mbox{and}\quad \Norm{ \mathbf{M}^{\le 1}_{\Set{1,2,5,6}\Set{3,4}}} \leq 1,
    	\end{align*}
    	\item and for $R=\sqrt{2}\cdot\Paren{\E_{a\sim N(0,\Id_d)} (aa^\top)^{\otimes 2}}^{+1/2}$, with high probability over $a_1,a_2,\ldots,a_{n}$, $\mathbf M^{\le 1}$ is close to $\mathbf{S}=\sum_{i \in S_0} \Paren{Ra_i^{\otimes 2}}^{\otimes 3}$ in Frobenius norm: $\Norm{\mathbf{M}^{\le 1}-\mathbf{S}}_F\leq 
    	3\epsilon\sqrt{n'}$.
    \end{itemize}
\end{lemma}
\begin{proof}
W.l.o.g. assume that $S_0 = [n']$.
By~\cref{lem:component-assumptions} we know that the set $\Set{a_1, \ldots, a_{n'}}$ is $(n,d)$-nicely separated.
For each $i\in [n']$, we denote $b_i\coloneqq Ra_i^{\otimes 2}$.
First by \cref{lem:equivalence-isotropic},
 we have 
 \begin{equation*}
    \Normf{\hat{\mathbf M}-\sum_{i=1}^{n'} b_i\Paren{b_i^{\otimes 2}}^\top}
    \leq 2\epsilon\sqrt{n}
 \end{equation*}
Then by \cref{lem:rectangle-matrix-norm}, with high probability 
we have 
$$\norm{S_{\{1,2\},\{3\}}}=\Norm{\sum_{i=1}^{n'} b_i\Paren{b_i^{\otimes 2}}^\top}
\leq 1+\tilde{O}\Paren{\frac{n}{d^{3/2}}}$$ 
We denote
$\mathbf{S'}\coloneqq\frac{\mathbf{S}}{\Norm{S_{\{1,2\},\{3\}}}}$.
Since the square flattenings of $\mathbf S'$ and $\mathbf S$ both have rank $n'$ it follows that 
\begin{equation*}
 \normf{\mathbf{S}-\mathbf{S'}}\leq \tilde{O}\Paren{
     \frac{n}{d^{3/2}}}\cdot \sqrt{n'}
\end{equation*}
and $\Norm{S'_{\{1,2\}\{3\}}}=\Norm{S^{\prime}_{\{1,3\}\{2\}}}=1$.

We denote $\mathbf{E'}\coloneqq \hat{\mathbf{M}}-\mathbf{S'}$,
 then we have
\begin{equation*}
    \mathbf{T}=\mathbf{S}+\mathbf{E}=\mathbf{S'}+\mathbf{E'}
\end{equation*}
and further
$$\norm{\mathbf{E'}}_F\leq \norm{\mathbf{E}}_F+\Normf{\mathbf{S'}-\mathbf{S}}
\leq 2\epsilon \sqrt{n'}$$

Denote
 $\mathcal{O}$ as the set of $d^2\times d^4$ matrices 
with singular values at most $1$. 
 Since $S^{\prime}_{\{1,2\}\{3\}}\in \mathcal{O}$, and 
 $S^{\prime}_{\{1,3\}\{2\}}\in \mathcal{O}$, 
 we have
 \begin{align*}
 \Norm{\mathbf M^{\le 1}-\mathbf{S'}}_F\leq  \Normf{\hat{\mathbf M^{'}}-\mathbf S'}\leq \Normf{\hat{\mathbf M}-\mathbf{S'}}\leq 2\epsilon\sqrt{n'}.
 \end{align*}
 And thus $\Normf{\hat{\mathbf{M}'}-\mathbf{S}}\leq \Normf{\mathbf{S}-\mathbf{S'}}+
 2\epsilon\sqrt{n'}\leq 3\epsilon\sqrt{n'}$

Trivially, we then have $\norm{M^{\le 1}_{\{1,3\}\{2\}}} \leq 1$ so what remains to show is that the second projection didn't increase the spectral norm of the $\Set{1,2}\Set{3}$-flattening:
I.e., that $\norm{M^{\le 1}_{\{1,2\}\{3\}}} = \norm{M^{\le 1}_{\{1,2\}\{3\}}} \leq 1$ as well.
To see this, we notice the following: Let $U \Sigma V^\top$ be a SVD of 
$\hat{M'}_{\{1,3\}\{2\}}$ and $P = V \Theta V^\top$, 
where $\Theta_{i,i} = 1/\Sigma_{i,i}$ if $\Sigma_{i,i} > 1$ 
and $1$ otherwise.
Clearly, we have that $M^{\le 1}_{\{1,3\}\{2\}} = \hat{M'}_{\{1,3\}\{2\}} P$.
So $M^{\le 1}_{\{1,2\}\{3\}}$ is obtained by starting with $\hat{M'}_{\{1,3\}\{2\}}$, 
switching modes 2 and 3, right-multiplying by $P$ and switching back modes 2 and 3.
This is in fact equivalent to left-multiplying $(\Id \otimes P)$ and hence 
we have $\norm{M^{\le 1}_{\{1,2\}\{3\}}} = \norm{M^{\le 1}_{\{1,2\}\{3\}}} = 
\norm{(\Id \otimes P) \hat{M'}} \leq \norm{\hat{M'}}$ 
since the spectral norm of $P$ is at most 1.
To see why this is equivalent, write $\hat{M'}$ as an $\R^{d^2 \times d}$ matrix with $d$ blocks $B_1, \ldots, B_d \in \R^{d \times d}$.
Exchanging modes 2 and 3 then yields the matrix with blocks 
$B_1^\top, \ldots B_d^\top$.
So that right-multiplying with P and exchanging back modes 2 and 3 
yields the matrix with $P B_1, \ldots P B_d$ which equals 
$(P \otimes \Id) \hat{M'}$ (note that $P$ is symmetric).
\end{proof}

\subsection{Gaussian rounding}
\label{sec:gaussian-rounding}

The goal of this section is to prove~\cref{lem:gaussian-rounding-correctness} which we restate below.
\begin{lemma}[Restatement of~\cref{lem:gaussian-rounding-correctness}]
    Let $n \leq O\Paren{d^{3/2}/\polylog(d)}$, let $a_1,a_2,\ldots,a_{n}\in\mathbb{R}^d$ be independently 
    and uniformly sampled from the unit sphere, and let $\e\in \frac{1}{\polylog(d)}$. 
    Then, with high probability over $a_1,a_2,\ldots,a_{n}$, for $d\leq n'\leq n$ and for any $S_0\subseteq [n]$ of size $n'$, 
     given any $\mathbf M^{\le 1}\in \mathbb{R}^{d^2\times d^2\times d^2}$ such that 
     \begin{align*}
    \normf{\mathbf M^{\le 1} - \sum_{i\in S_0}\Paren{(Ra_i^{\otimes 2}}^{\otimes 3} } 
    \leq \e \sqrt{n'} \quad \text{ and } \quad \Norm{\mathbf{M}^{\le 1}_{\Set{1,2,3,4}\Set{5,6}}}, 
    \Norm{\mathbf{M}^{\le 1}_{\Set{1,2,5,6}\Set{3,4}}}\leq 1,
    \end{align*}
   the \hyperref[step:random-contraction-rounding]{Gaussian rounding} 
   step of~\cref{alg:rounding} outputs unit vectors
    $b_1,b_2,\ldots,b_m\in \mathbb{R}^d$ for $m \geq 0.99n'$
    such that for each $j\in [m]$ there exists a unique  $i \in S_0$
    with $\iprod{b_j, a_i} \geq 0.99$.
\end{lemma}

We also restate the relevant part of~\cref{alg:rounding} here:
\begin{algorithmbox}[Restatement of Gaussian Rounding step of~\cref{alg:rounding}]
\label{alg:gaussian-rounding}
\mbox{}\\
	\noindent
		\begin{itemize}
		\item Initialize $C \leftarrow \emptyset$
    
    		\item Repeat $\tilde{O}(d^2)$ times:
	\begin{enumerate}
		\item Sample $g \sim N(0, \Id_{d^2})$ and compute $d^2\times d^2$ matrix
        $M_g = \Paren{g\otimes \Id_{d^2} \otimes \Id_{d^2}} \flattent{\mathbf{M}^{\le 1}}{1,2}{3,4}{5,6}$.
		\item Compute the top right singular vector of $M_g$ denoted by $u\in \R^{d^2}$ and
		flatten it into square matrix $U\in\mathbb{R}^{d\times d}$.
		\item Compute the top left and right singular vectors of $U$
        denoted by $v_l, v_r \in \R^d$.
		\item For $b\in \{\pm v_{L},\pm v_{R}\}$: 
		
        \qquad If $\iprod{T,b^{\otimes 3}}\geq 1-\frac{1}{\polylog(n)}$

        \qquad\qquad Add $b$ to $C$

        \item For $b\in C$:
        
        \qquad if $\iprod{b, b'}\leq 0.99$ for all $b' \in S$
        
        \qquad \qquad add $b$ to $S$
	\end{enumerate}
	\item Output $S$.
	\end{itemize}
\end{algorithmbox}

To prove~\cref{lem:gaussian-rounding-correctness} we will proceed in several steps.
For the sake of presentation we will only outline the proofs and move the more technical steps to~\cref{sec:missing-proofs}.
First, we will show that the subroutine in~\hyperref[step:gaussian-rounding]{Step 1} in~\cref{alg:gaussian-rounding} recovers one of the components up to constant correlation with probability at least $\tilde{\Theta}(d^{-2})$.
Concretely, we will show the following lemma:
\begin{lemma}
\label{lem:recover-one-component}
Consider the setting of \cref{lem:gaussian-rounding-correctness}.
Let $S_0 \subseteq [n]$ be of size $d\leq n' \leq n$ and assume that the set $\Set{a_i \suchthat i \in S_0}$ is $(n,d)$-nicely separated.
Consider $v_l$ and $v_r$ in~\cref{alg:gaussian-rounding}, then there 
exists a set $S \subseteq S_0$ of size $m \geq 0.99n'$ such that for each $i \in S$ it holds with probability $\tilde{\Theta}(d^{-2})$ that $\max_{v \in \{\pm v_l, \pm v_r\}} \iprod{v, a_i}\geq 1-\frac{1}{\polylog(d)}$.
\end{lemma}
This will follow by the following sequence of lemmas.
The first one show that the top singular vector of the matrix $M_g$ in~\cref{alg:gaussian-rounding} is correlated with one of the components and that it further admits a spectral gap.
\begin{lemma}
\label{lem:diagonal-spectral-gap}
Consider the setting of \cref{lem:gaussian-rounding-correctness}.
Let $R=\sqrt{2}\cdot \Paren{\E_{a\sim \Id_d}\Paren{aa^\top}^{\otimes 2}}^{+1/2}$, let $S_0 \subseteq [n]$ be of size $n'$ where $d \leq n' \leq n$, and assume that the set $\Set{a_i \suchthat i \in S_0}$ is $(n,d)$-nicely separated.
Further, let $\hat{\mathbf M}$ be such that
\begin{align*}
\normf{\mathbf M^{\le 1} - \sum_{i\in S_0} (R a_i^{\otimes 2})^{\otimes 3}} \leq \e \sqrt{n'} \quad \text{ and } \quad \Norm{ \mathbf{M}^{\le 1}_{\Set{1,2,3,4}\Set{5,6}}}, \Norm{ \mathbf{M}^{\le 1}_{\Set{1,2,5,6}\Set{3,4}}}\leq 1.
\end{align*}
Consider the matrix $M_g = \Paren{g\otimes \Id_{d^2}\otimes \Id_{d^2}}  \flattent{\mathbf{M}^{\le 1}}{1,2}{3,4}{5,6}$ 
in~\cref{alg:gaussian-rounding}.
 Then there exists a subset 
   $S\subseteq S_0$ of size $m \geq 0.99n'$, 
   such that for each $i\in S$, and
   $v=Ra_i^{\otimes 2}$, with probability at least
   $1/d^{2(1+1/\log n)}$ over $g$,
   we have $M=cvv^\top+N$ where  
   \begin{itemize}
   \item 
   $\norm{cvv^\top}\geq (1+\frac{1}{\log d})\norm{N}$
   \item $\norm{Nv},\norm{vN}\leq \epsilon c\norm{v}^2$ 
\end{itemize}
\end{lemma}
The proof of this lemma resembles Lemma 4.6 in \cite{SchrammS17}, 
and we defer to \ref{sec:random-contraction-missing-proof}.

Next, we will show how to use this spectral gap to recover one of the components up to accuracy $1-\frac{1}{\polylog d}$:
\begin{lemma}
\label{lem:top-singular-vector}
Consider the setting of \cref{lem:gaussian-rounding-correctness}.
Let $R=\sqrt{2}\cdot \Paren{\E_{a\sim \Id_d}\Paren{aa^\top}^{\otimes 2}}^{+1/2}$, 
let $S_0 \subseteq [n]$ be of size $d\leq n' \leq n$, and assume that
 the set $\Set{a_i \suchthat i \in S_0}$ is $(n,d)$-nicely separated.
Consider the matrix $M_g$ and its top right singular vector $u_r \in \R^{d^2}$ obtained in one iteration of~\cref{alg:gaussian-rounding}.
Then, there exists a set $S \subseteq S_0$ with size at least $0.99n'$,
such that for each $i \in S$,  
it holds with probability $\tilde{\Theta}(d^{-2})$ that 
\begin{itemize}
    \item $\iprod{u_r, Ra_i^{\otimes 2}} \geq 1 - \frac{1}{\polylog d}$.
    \item the ratio between largest and second largest singular values of 
    $M_g$ is larger than $1+\frac{1}{\polylog d}$
\end{itemize}
\end{lemma}
 \begin{lemma}
 \label{lem:extracting-component-from-square}
 Consider the setting of \cref{lem:gaussian-rounding-correctness}.
     Suppose for some unit norm vector $a\in \mathbb{R}^d$ and some unit vector $u\in \mathbf{R}^{d^2}$, $\iprod{u,Ra^{\otimes 2}}\geq 1-\frac{1}{\polylog(d)}$. 
     Then flattening $u$ into a $d\times d$
    matrix $U$, the top left or right singular vector of $U$ denoted by $v$ will
     satisfy $\iprod{a,v}^2\geq 1-\frac{1}{\polylog(d)}$. 
\end{lemma}
The proof of \cref{lem:top-singular-vector} is essentially the same as 
Lemma 4.7 in \cite{SchrammS17}. The proof of 
\cref{lem:extracting-component-from-square} essentially the same as Lemma 
19 in \cite{HopkinsSS19}. 
 We defer the proofs of these two lemmas to 
 \cref{sec:partial-recovery-missing-proof}.

With this in place, it follows that the list of vectors $C = \Set{b_1, \ldots, b_L}$ for $L = \tilde{O}(d^2)$ obtained by~\cref{alg:gaussian-rounding} satisfies the following where $S$ is the subset of components of~\cref{lem:recover-one-component}:
 \begin{equation*}
        \forall i \in S \colon \max_{b\in \mathcal{C}} \Abs{\iprod{b,a_i}}\geq 1-\frac{1}{\polylog(d)}
    \end{equation*}
and
\begin{equation*}
        \forall b \in C \colon \max_{i\in S} \Abs{\iprod{b,a_i}}\geq 1-\frac{1}{\polylog(d)}
\end{equation*} 
The first equation follows by the Coupon Collector problem,~\cref{lem:recover-one-component}, and the fact that we repeat the inner loop of~\cref{alg:gaussian-rounding} $\tilde{O}(d^2)$ times.
The second equation follows since by \cref{lem:tensor-vector-correlation}, we have $\iprod{T,v^{\otimes 3}}\geq 1-\frac{1}{\polylog(d)}$ if and only if $\iprod{v,a_i}\geq 1-\frac{1}{\polylog(d)}$.

Finally, the following lemma (proved in~\cref{sec:pruning}) states that~\hyperref[alg:gaussian-rounding]{Step 3} of~\cref{alg:gaussian-rounding} outputs a set of vectors satisfying the conclusion of~\cref{lem:gaussian-rounding-correctness}:
\begin{lemma}
\label{lem:pruning-components}
Let $S_0 \subseteq [n]$ be of size $n' \geq 0.99n$ and assume 
that the set $\Set{a_i \suchthat i \in S_0}$ is $(n,d)$-nicely separated.
Further, let $S$ be the set of vector computed in~\hyperref[alg:gaussian-rounding]{Step 3} 
of~\cref{alg:gaussian-rounding} and let $S'$ be the subset of components 
of~\cref{lem:recover-one-component}.
Then, for each $b \in S$, there exists a unique $i \in S'$ 
such that $\iprod{b, a_i} \geq 1- \frac{1}{\polylog d}$.
\end{lemma}

\section{Full recovery algorithm}\label{section:non-robust-algorithm-full-recovery}

In the previous section, we proved that the Gaussian Rounding subroutine (\hyperref[step:pre-processing]{Step 2(a)} and~\hyperref[step:rounding]{Step 2(b)}) in the~\hyperref[step:recovery]{Recovery step} of~\cref{algorithm:non-robust-recovery} recovers a $0.99$ fraction of the components. In this section, we will show how to build on this to recover all components.
More precisely, we will prove~\cref{thm:Full-recovery} which we restate below.

\begin{theorem}[Restatement of~\cref{thm:Full-recovery}]
Let $a_1,\ldots, a_n$ be \iid vectors sampled uniformly from the unit sphere in  $\R^d$.
For $\e=\frac{1}{\polylog(d)}$, given as input 
\begin{align*}
	\mathbf T=\sum_{i=1}^n a_i^{\otimes 3} 
    \quad \text{and}\quad 
    \hat{M}=\sum_{i=1}^n 
    a_i^{\otimes 3}\Paren{a_i^{\otimes 3}}^\top+E\,, 
    \text{with }\norm{E}\leq \epsilon \text{ and } 
    \normf{E} \leq \e \sqrt{n}\,,
\end{align*}
 \cref{algorithm:non-robust-recovery} returns unit norm 
vectors $b_1,b_2,\ldots,b_n$ satisfying 
 \begin{align*}
     \norm{a_i - b_{\pi(i)}}\leq \tilde{O}\Paren{\frac{\sqrt{n}}{d}}\,,
 \end{align*}
for some permutation $\pi:[n]\rightarrow [n]$.
\end{theorem}

For completeness, 
we also restate the relevant part 
of~\cref{algorithm:non-robust-recovery} here:'

\begin{algorithmbox}[{Restatement of the~\hyperref[step:recovery]{Recovery step} in~\cref{algorithm:non-robust-recovery}}]
	\mbox{}\\
	\noindent 

	\textbf{Input:} A matrix $\hat{M}$ such that for some $\epsilon=\frac{1}{\polylog(d)}$:
	\begin{align*}
		\norm{\hat{M} - \sum_{i=1}^n (a_i^{\otimes 3})(a_i^{\otimes 3})^\top} \leq \e \quad\text{ and }\quad \normf{\hat{M} - \sum_{i=1}^n (a_i^{\otimes 3})(a_i^{\otimes 3})^\top} 
        \leq \e \cdot \sqrt{n}
	\end{align*}
	
	\textbf{Output:} Unit vectors $b_1,\ldots, b_n\in \R^d\,.$

	\begin{itemize}
	
	\item Repeat $O(\log n)$ times:
	
	\begin{enumerate}[(a)]
			\item \textbf{Pre-processing:} Project $\hat{M}$ into the space of matrices  in $\R^{d^3\times d^3}$ 				satisfying
			\begin{align*}
				\Norm{\hat{M}_{\Set{1,2,3,4}\Set{5,6}}}\leq 1\,, \quad \Norm{\hat{M}_{\Set{1,2,5,6}\Set{3,4}}}\leq 1\,.
			\end{align*}
			
			\item \textbf{Rounding:} Run $\tilde{O}(d^2)$ independent trials of Gaussian Rounding on $\hat{M}$ contracting its first two modes (as in \cref{alg:rounding}) to obtain a set of $0.99n$ candidate vectors $b_1, \ldots, b_{0.99n}$.

			\item \label{step:full-boosting} \textbf{Accuracy boosting:} Boost the accuracy of each candidate $b_i$ via 				tensor power iteration.
			\item \label{step:full-peeling} \textbf{Peeling of recovered components:}
			\begin{itemize}[\labelitemi]
				\item Set $\hat{M}$ to be the best rank-$0.01n$ approximation of $\hat{M} -\sum_{i \leq 0.99n} \Paren{b_i^{\otimes 3}} \Paren{b_i^{\otimes 3}}^\top$
				\item Update $n \leftarrow 0.01n$.
			\end{itemize}
	\end{enumerate}		
	\item Return all the candidate vectors $b_1, \ldots, b_n$ obtained above.
	\end{itemize}	
\end{algorithmbox}

Our main goal will be to show that in each iteration the matrix $\hat{M}$ satisfies the assumption of~\cref{thm:partial-recovery-main} and then use an induction argument.
To show this, we will proceed using following steps:
\begin{itemize}
   \item By~\cref{thm:partial-recovery-main} we recover at least a 0.99 fraction of the remaining components up to accuracy $0.99$.%
   \item We will show that using tensor power iteration we can boost this accuracy to $1-\tilde{O}\Paren{\frac{\sqrt{n}}{d}}$.
    \item In a last step we prove that after the removal step (\hyperref[step:full-peeling]{Step 2(d)}) the resulting matrix satisfies the assumptions of~\cref{thm:partial-recovery-main}.
\end{itemize}
We will discuss the boosting step in~\cref{sec:boosting} and the removal step in~\cref{sec:peeling}.
In~\cref{sec:combining-boosting-and-peeling} we will show how to combine the two to prove~\cref{thm:Full-recovery}.

\subsection{Boosting the recovery accuracy
by tensor power iteration}
\label{sec:boosting}
Given the relatively coarse estimation of part of the components, we use tensor power iteration in~\cite{Anandkumar-pmlr15}  to boost the accuracy. 
\begin{lemma}[Lemma 2 in \cite{Anandkumar-pmlr15}]
    Let $T=\sum_{i=1}^n a_i^{\otimes 3}$, 
    where 
    $a_1,a_2,\ldots, a_n$ are independently and uniformly sampled
     from $d$-dimensional unit sphere. 
    Then with high probability over $a_1,a_2,\ldots,a_n$, 
    for any unit norm vector $v$ such that
    $\iprod{v,a_1}\geq 0.99$, 
    , the tensor power iteration 
    algorithm gives
     unit norm vector $b_1$ such that $\iprod{a_1,b_1}\geq 1-\tilde{O}\Paren{\frac{n}{d^2}}$
     and runs in $\tilde{O}(d^3)$ time.
\end{lemma}

By running tensor power iteration on the vectors obtained in the last subsection,
 we thus get the following guarantee:
\begin{corollary}\label{cor:boosting-partial-recovery}
    Given tensor $T=\sum_{i=1}^n a_i^{\otimes 3}$, where 
    $a_1,a_2,\ldots, a_n$ are independently and uniformly sampled
    from $d$-dimensional unit sphere. 
    Suppose for a set $S\subseteq [n]$ with size $m$, we
    are given vectors $b_1,b_2,\ldots,b_m$ such that
     for each $i\in S$,
    $$\max_{j\in [m]}\iprod{a_i,b_j}\geq 0.99$$
    Then in $\tilde{O}(nd^3)$ time, 
    we can get unit norm vectors $c_1,c_2,\ldots,c_m$ s.t 
    for each $i\in S$,
    $$\max_{j\in [m]}\iprod{a_i,c_j} \geq 1-\tilde{O}\Paren{\frac{n}{d^2}}\,.$$ 
\end{corollary}

\subsection{Removing recovered components}
\label{sec:peeling}
In this part, we mainly prove that we can remove the recovered 
 components as in \hyperref[step:peeling]{Step 2(d)} of~\cref{algorithm:non-robust-recovery}, 
 without increasing spectral norm of noise by more than $\poly\Paren{\frac{n}{d^{3/2}}}$. 
\begin{lemma}\label{lem:removal-process}
    Let $m\geq d$ and $d\leq n=O\Paren{d^{3/2}/\polylog(d)}$. Let
    $a_1,a_2,\ldots,a_{n}\in\mathbb{R}^d$ be i.i.d 
     random unit vectors sampled uniform from the sphere. 
     Then with high probability over $a_1,a_2,\ldots,a_{n}$, 
     for any $S=\{s_1,s_2,\ldots, s_m\}\subseteq [n]$,
      and $b_1,b_2,\ldots,b_{m}$ satisfying 
      $\norm{a_{s_i}-b_i}\leq \tilde{O}\Paren{\sqrt{n}/d}$, we have
      \begin{equation*}
        \Norm{\sum_{i\in S} \Paren{a_i^{\otimes 3}}\Paren{a_i^{\otimes 3}}^\top
        - \sum_{i\in [m]} \Paren{b_i^{\otimes 3}}\Paren{b_i^{\otimes 3}}^\top }
        \leq \frac{1}{\polylog(d)}
    \end{equation*}    
\end{lemma}

We first prove the same result under the deterministic assumption that 
$\{a_i:i\in S\}$ are $(n,d)$ nicely-separated. 
Then combining with \cref{lem:component-assumptions}, the \cref{lem:removal-process}
 follows as a corollary. 
 \begin{lemma}\label{lem:removal-process-lemma}
    Let $m\geq d$ and $d\leq n=O\Paren{d^{3/2}/\polylog(d)}$. Let
    $a_1,a_2,\ldots,a_{m}\in\mathbb{R}^d$ be 
    vectors satisfying the $(n,d)$ nicely-separated assumptions 
    of~\cref{def:nicely-separated}. 
    Suppose unit norm vectors $b_1,b_2,\ldots,b_{m}$ 
    satisfies that $\norm{a_i-b_i}\leq O\Paren{\sqrt{n}/d}$. Then we have
    \begin{equation*}
        \Norm{\sum_{i\in [m]} \Paren{a_i^{\otimes 3}}\Paren{a_i^{\otimes 3}}^\top
        - \sum_{i\in [m]} \Paren{b_i^{\otimes 3}}\Paren{b_i^{\otimes 3}}^\top }
        \leq \frac{1}{\polylog(d)}
    \end{equation*}
\end{lemma}
\begin{proof}
    We denote the matrix $U\in \mathbb{R}^{d^3\times m}$ with the $i$-th column
     given by $a_i^{\otimes 3}$, and the matrix $V\in \mathbb{R}^{d^3\times n}$ 
     with the $i$-th column
     given by $b_i^{\otimes 3}$. Then 
     \begin{equation*}
        \Norm{\sum_{i\in [m]} \Paren{a_i^{\otimes 3}}\Paren{a_i^{\otimes 3}}^\top
        - \sum_{i\in [m]} \Paren{b_i^{\otimes 3}}\Paren{b_i^{\otimes 3}}^\top }
        = \Norm{UU^\top-VV^\top}
     \end{equation*}
     Now since $\Norm{UU^\top-VV^\top}=\Norm{(U-V)U^\top+V(U-V)^\top}$ and 
     $\norm{U}\leq 1$ with high probability, it's suffcient to show that 
     $\norm{U-V}\leq \frac{n^2}{d^3}$, which is equivalent to $\sqrt{\Norm{(U-V)^\top(U-V)}}$.
  
We denote $W:= (U-V)^\top(U-V)$, and let $W=W_1+W_2$ 
where $W_1$ be the diagonal part of the matrix $W$ and
$W_2$ be the non-diagonal part. 
Then for $i\in [n]$, the diagonal entries of $W$ are 
given by 
$$W_{ii}=\Paren{a_i^{\otimes 3}-b_i^{\otimes 3}}^\top \Paren{a_i^{\otimes 3}-b_i^{\otimes 3}}=
\Norm{a_i^{\otimes 3}-b_i^{\otimes 3}}^2$$ 
Now since
\begin{align*}
    \Norm{a_i^{\otimes 3}-b_i^{\otimes 3}}^2\leq 2-2\iprod{a_i,b_i}^3
    =2-2\cdot\frac{\Paren{2-\norm{a_i-b_i}^2}^3}{8}\leq 2-\Paren{2-6 \cdot \norm{a_i-b_i}}
    = 6 \cdot \norm{a_i-b_i}^2
\end{align*}
it follows that $\Norm{a_i^{\otimes 3}-b_i^{\otimes 3}}\leq \tilde{O}(\sqrt{n}/d)$.
Since $W_1$ is a diagonal matrix, we have $\norm{W_1}\leq \tilde{O}(\sqrt{n}/d)$. 
   
Next we bound $\norm{W_2}_F$. We denote $c_i=a_i-b_i$. 
Then by assumption we have $\norm{c_i}\leq O\Paren{\sqrt{n}/d}$.
 Now we have
 \begin{align*}
      \iprod{a_i^{\otimes 3}-b_i^{\otimes 3},a_j^{\otimes 3}-b_j^{\otimes 3}}
      &= \iprod{a_i^{\otimes 3}-(a_i+c_i)^{\otimes 3},a_j^{\otimes 3}-(a_j+c_j)^{\otimes 3}}\\
      &= \sum_{\substack{g^{(1)}_i,g^{(2)}_i,g^{(3)}_i\\g^{(4)}_j,g^{(5)}_j,g^{(6)}_j}}
      \iprod{g_i^{(1)}, g_j^{(4)}}\iprod{g_i^{(2)},g_j^{(5)}}\iprod{g_i^{(3)},g_j^{(6)}}
 \end{align*}
 where for $k\in [6]$ and $i\in [m]$, $g_i^{(k)}\in \{a_i,c_i\}$. 
Now we rewrite $W_2=\sum_{g} M_g$,
 where $M_{g,i,j}=\iprod{g_i^{(1)}, g_j^{(4)}}\iprod{g_i^{(2)},g_j^{(5)}}\iprod{g_i^{(3)},g_j^{(6)}}$. Since there are less than 
$2^3$ choices for $g=\Set{g^{(1)},g{(2)},\ldots,g{(6)}}$, 
By \cref{lem:full-recovery-correlation-bound}, for every choice of $g$, 
we have $\normf{M_g}\leq \tilde{O}\Paren{\sqrt{\frac{n}{d^{3/2}}}}\leq \frac{1}{\polylog(d)}$. 
By applying triangle inequality, we have $\norm{W_2}_F\leq \frac{1}{\polylog(d)}$.

It follows that   
\begin{equation*}
    \Norm{(U-V)^\top (U-V)}=\norm{W}\leq \norm{W_1}+\norm{W_2}\leq \frac{1}{\polylog(d)}\,,  
\end{equation*}
which concludes the proof.
\end{proof}

\subsection{Putting things together}
\label{sec:combining-boosting-and-peeling}

\begin{proof}[Proof of \cref{thm:Full-recovery}] 
We show that, if the events in \cref{thm:partial-recovery-main},
 \cref{cor:boosting-partial-recovery}, and \cref{lem:removal-process}
  happen, then \cref{algorithm:non-robust-recovery} returns unit norm 
  vectors $b_1,b_2,\ldots,b_n$ satisfying 
   \begin{align*}
       \norm{a_i - b_{\pi(i)}}\leq \tilde{O}\Paren{\frac{\sqrt{n}}{d}}\,,
   \end{align*}
  for some permutation $\pi:[n]\rightarrow [n]$.
  Since by \cref{thm:partial-recovery-main},
  \cref{cor:boosting-partial-recovery}, and \cref{lem:removal-process},
   these events happen with high probability over random unit vectors
   $a_1,a_2,\ldots,a_n$, the theorem thus follows.

Let $\delta=\frac{1}{\log^{10}(n)}$, and 
$n\leq d^{3/2}/\log^{10000}n$. For $t\leq O(\log n)$,
we prove by mathematical induction that after $t$-th iteration
 of the~\hyperref[step:recovery]{Recovery step} 
 in~\cref{algorithm:non-robust-recovery},
    for a subset $S_t\subseteq [n]$, we have 
$$\Norm{M-\sum_{i\in S} a_i^{\otimes 3}
\Paren{a_i^{\otimes 3}}^\top}\leq (t+1)\delta\,.$$
Further we have $S_{t+1}\leq 0.01 S_t$

As base case 
after the~\hyperref[step:lifting]{Lifting step} of~\cref{algorithm:non-robust-recovery},
we have
$$\Norm{M-\sum_{i=1}^n a_i^{\otimes 3}\Paren{a_i^{\otimes 3}}^\top}
\leq \tilde{O}\Paren{\frac{n}{d^{3/2}}}\leq \delta\,.$$
 
For induction step, we suppose for some $S_t\subseteq [n]$, 
$$\Norm{M-\sum_{i\in S_t} a_i^{\otimes 3}
\Paren{a_i^{\otimes 3}}^\top}\leq t\delta\,.$$
Since we condition that the statement in \cref{thm:partial-recovery-main}
 holds,
for some $m\geq 0.99\Abs{S_t}$ and 
$S_{t}^{\prime}\subseteq n$ 
with size $m$,~\hyperref[step:rounding]{Step 2(b)} of~\cref{algorithm:non-robust-recovery}
outputs unit norm vectors 
$b_1,b_2,\ldots,b_{m}$ 
such that  for each $i\in S_{t}^{\prime}$, 
 \begin{equation*}
     \max_{j\in [m]}\iprod{a_i,b_j}\geq 1-\frac{1}{\polylog(d)}
 \end{equation*}
 Then combining~\cref{cor:boosting-partial-recovery}
  and~\cref{lem:removal-process}, before~\hyperref[step:peeling]{Step 2(d)} 
  of $t$-th iteration of the~\hyperref[step:recovery]{Recovery step},
 we have 
     \begin{equation*}
        \Norm{\sum_{i\in S_{t}^{\prime}}a_i^{\otimes 3}
        \Paren{a_i^{\otimes 3}}^\top
        -\sum_{i\in [m]}b_i^{\otimes 3}
        \Paren{b_i^{\otimes 3}}^\top}\leq \delta
     \end{equation*}
     By triangle inequality, after removal  
     step (d),
     it follows that
     \begin{equation*}
        \Norm{M
        -\sum_{i\in S_t\setminus S_{t}^{\prime}}a_i^{\otimes 3}
        \Paren{a_i^{\otimes 3}}^\top} \leq (t+1)\delta
     \end{equation*}
     By setting $S_{t+1}=S_t\setminus S_{t}^{\prime}$,
      we have $\Abs{S_{t+1}}\leq 0.01 \Abs{S_t}$, and 
      \begin{equation*}
        \Norm{M
        -\sum_{i\in S_{t+1}}a_i^{\otimes 3}
        \Paren{a_i^{\otimes 3}}^\top} \leq (t+1)\delta  
      \end{equation*}
      The induction step is thus finished.

      Now putting the recovery vectors obtained in 
       all the iterations, we finish the proof. 
\end{proof}

\section{Implementation and running time analysis}\label{section:simple-algorithm-running-time}

We prove here \cref{thm:main-running-time} concerning the running time of \cref{algorithm:non-robust-recovery}.

\begin{remark}[On the bit complexity of the algorithm]
	We assume that the vectors $a_1,\ldots, a_n\in \R^d$ have 
	polynomially (in the dimension) bounded norm. We can then represent each of the vectors, 
	matrices and tensor considered to polynomially small precision with 
	logarithmically many bits (per entry). This representation does not 
	significantly impact the overall running time of the algorithm, 
	while also not invalidating its error guarantees (with high probability). 
	For this reason we ignore the bit complexity aspects of the problem.
\end{remark}

\subsection{Running time analysis of the lifting step} 
For a matrix $A\in \R^{d\times d}$, we say that $L$ is the best rank-$m$ approximation of $A$ if 
\begin{align*}
	L= \arg\min \Set{\Normf{A-L}\suchthat L\in \R^{d\times d}\,, \rank(L)\le m}\,.
\end{align*}

We will consider the following algorithm:
\begin{algorithmbox}[Compute implicit representation]
	\label{algorithm:compute-implicit-representation}
	\mbox{}\\
	\noindent 
	
	\textbf{Input:} Tensor $\mathbf{T}=\underset{i \in[n]}{\sum} \tensorpower{a_i}{3}\,.$
	
	\textbf{Output:} $U, V\in \R^{d^3\times n}$.
	
	\begin{enumerate}
		\item 
		Use the $n$-dimensional subspace power method \cite{HardtP14} on the  $\Set{1,2,3}\Set{4,5,6}$ flattening of
		\begin{align}\label{eq:tensor-network}
			\mathbf{M} = \underset{i,j,k,\ell\in [n]}{\sum}\iprod{a_i, a_j}\cdot \iprod{a_i, a_k}\cdot \iprod{a_i, a_\ell}\cdot \paren{\dyad{a_j}}				\tensor \paren{\dyad{a_k}} \tensor \paren{\dyad{a_\ell}}\,,
		\end{align}
		decomposing contractions with $\mathbf{M}_{\Set{1,2,3}\Set{4,5,6}}$ as shown in \cref{fig:power_iteration-ternary_tree} and using the fast rectangular matrix multiplication algorithm of \cite{GallU18}.
		\item Return $U, V\in \R^{d^3\times n}$ computed from the  resulting (approximate) $n$ eigenvectors and eigenvalues.
	\end{enumerate}
\end{algorithmbox}

\begin{lemma}\label{lemma:running-time-lifting-step}
	Let $a_1,\ldots, a_n$ be \iid vectors uniformly sampled from the unit sphere in  $\R^d$.
	Consider the flattening $\mathbf{M}_{\Set{1,2,3}\Set{4,5,6}}$ of $\mathbf{M}$ as in \cref{eq:tensor-network}. Let $U'\Sigma' \transpose{U'}$ with $U'\in \R^{d^3\times n}\,, \Sigma'\in \R^{n \times n}$, be its best rank-$n$ approximation.
	Then, there exists an algorithm (\cref{algorithm:compute-implicit-representation}) that, given $\mathbf{T}$, computes $U, V\in \R^{d^3\times n}$ such that
	\begin{align*}
		\Normf{U\transpose{V}-U'\Sigma \transpose{U'}}\leq d^{-100}\,.
	\end{align*}
	 Moreover, the algorithm runs in time $\tilde{O}\Paren{d^{2\cdot \omega(1+\log n/2\log d)}}$, where $\omega(k)$ is the time required to multiply a $(d^k\times d)$ matrix with a $(d\times d)$ matrix.\footnote{See  \cref{fig:matrix-multiplication-table} in \cref{section:matrix-multiplication-constants}.} 
	\begin{proof}
		It suffices to show how to approximately compute the top $n$ eigenvectors and eigenvalues of $\mathbf{M}_{\Set{1,2,3}\Set{4,5,6}}$ as then deriving $U, V$ from there is trivial.
		
		We start by explaining how to use the structure of the tensor network to multiply $\mathbf{M}$ by a vector $v$ more efficiently, then extend this idea to the subspace power method~\cite{HardtP14}, and finally apply the rectangular matrix multiplication method~\cite{GallU18}.
		
		To efficiently multiply a vector $v\in\R^{d^3}$ by $\mathbf{M}_{\Set{1,2,3}\Set{4,5,6}}$, we partition the multiplication into four steps by cutting the tensor network ``cleverly.''  
		\cref{fig:power_iteration-ternary_tree} presents the four-step multiplication.
		The multiplication time is $O(d^{2w})$ as explained as following.
		Step~(a) multiplies a $d^2\times d$ matrix with a $d\times d^2$ matrix, and thus takes $O\Paren{d^{1+\omega}}$ time.
		Step~(b) multiplies a $d^2\times d$ matrix with a $d\times d^2$ matrix, and thus takes $O\Paren{d^{1+\omega}}$ time.
		Step~(c) multiplies a $d^2\times d^2$ matrix with a $d^2\times d^2$ matrix, and thus takes $O\Paren{d^{2\omega}}$ time.
		Step~(d) multiplies a $d^2\times d^2$ matrix with $d^2\times d$ matrix, and thus takes $O\Paren{d^{2+\omega}}$ time.	
		
		Each iteration of the subspace power method~\cite{HardtP14} multiplies $n$ vectors by $\mathbf{M}_{\Set{1,2,3}\Set{4,5,6}}$ simultaneously.
		Therefore, $v$ in the above 4-step multiplication is replaced with a $d^3\times n$ matrix.
		Then, Step~(a) becomes multiplying a $nd^2\times d$ matrix with a $d\times d^2$ matrix, 
		Step~(c) becomes multiplying a $nd^2\times d^2$ matrix with a $d^2\times d^2$ matrix,
		and  Step~(d) becomes multiplying a $nd^2\times d^2$ matrix with $d^2\times d$ matrix.
		
		The rectangular multiplication algorithm~\cite{GallU18} takes $O(d^{w(k)})$ time to multiply a $d^k\times d$ matrix by a $d\times d^k$ matrix. 
		Note that the time complexities of the following three problems are the same: 
		multiplying a $d^k\times d$ matrix by a $d\times d^k$ matrix, multiplying a $d\times d$ matrix by a $d^k\times d$ matrix, and multiplying a $d\times d^k$ matrix by a $d\times d$ matrix.
		By the rectangular multiplication algorithm, Step~(a) takes $O\Paren{n\cdot d^{2\cdot \omega(0.5)}}=O\Paren{n\cdot d^{4.093362}}$ time,
		Step~(c) takes $O\Paren{d^{2\cdot \omega(\log_{d^2}(d^2n))}}=O\Paren{ d^{2\cdot \omega(1+\log n/\log d^2)}}$ time, 
		and Step~(d) takes$O\Paren{nd \cdot d^{\omega(2)}}=O\Paren{n\cdot d^{4.256689}}$ time.
		
		Since the time of Step~(c) dominates that of Step~(a), Step~(b) and Step~(d), one iteration of the subspace power method takes $O\Paren{n\cdot d^{2\cdot \omega(1+\log n/\log d^2)}}$ time. 
		By \cref{lem:spectral_norm-ternary_tree}, $\lambda_{n+1}/\lambda_n\leq 1/\polylog d$, so the subspace power method takes $\polylog d$ iterations.
		To conclude, computing the top $n$ eigenvectors of $\mathbf{M}_{\Set{1,2,3}\Set{4,5,6}}$ takes $\tilde{O}\Paren{d^{2\cdot \omega(1+\log n/\log d^2)}}$ time.
	\end{proof}
\end{lemma}

\subsection{Running time analysis for the pre-processing step}

In this section we show that the implicit representation of tensor $M^{\le 1}$ in \cref{lem:spectral-truncation-correctness}  can be  computed in a fast way. By \cref{lem:equivalence-isotropic}
we may assume our matrix $U\transpose{V}$ is close to a matrix flattening of $\sum_{i=1}^n \Paren{Ra_i^{\otimes 2}}^{\otimes 3}$, where $R=\sqrt{2}\cdot \Paren{\E_{a\sim N(0,\Id_d)} a^{\otimes 2}  \Paren{a^{\otimes 2}}^\top}^{+1/2}$.

\begin{lemma}[Running time of the pre-processing step]
	\label{lemma:running-time-pre-processing}
	Let $a_1,\ldots, a_n$ be a subset of \iid vectors uniformly sampled from the unit sphere in  $\R^d$.
	Let $R=\sqrt{2}\cdot \Paren{\E_{a\sim N(0,\Id_d)} a^{\otimes 2}  \Paren{a^{\otimes 2}}^\top}^{+1/2}$ and denote 
	$\mathbf{S}_3=\sum_{i=1}^n \Paren{Ra_i^{\otimes 2}}^{\otimes 3}$.
	There exists an algorithm that, given matrices $U,V\in \mathbb{R}^{d^3\times n}$ satisfying
	\begin{align*}
		\normf{U\transpose{V}-(\mathbf{S}_3)_{\Set{1,2,3}\Set{4,5,6}}}\leq \epsilon\sqrt{n}\,,
	\end{align*} 	
	computes matrices $U',V'\in \mathbb{R}^{d^3\times 2n}$
	satisfying
	\begin{align*}
		\normf{U'\transpose{V'}-(\mathbf{S}_3)_{\Set{1,2,3}\Set{4,5,6}}}\leq \epsilon\sqrt{n}\,, \quad\norm{(U'\transpose{V'})_{\{5,6\}\{1,2,3,4\}}}\leq 1\,,\quad 
		\norm{U'\transpose{V'}_{\{3,4\}\{1,2,5,6\}}}\leq 1\,. 
	\end{align*}
	Moreover, the algorithm runs in time $\tilde{O}(d\cdot n^{\omega \Paren{\frac{2\log d}{\log n}}}+nd^4)\leq \tilde{O}(d^{5.05}+nd^4)$.
\end{lemma}

The algorithm used to compute these fast projections consists of two subsequent application of the following procedure (symmetrical with respect to the two distinct flattenings).

\begin{algorithmbox}[Fast projection]
	\label{algorithm:compute-pre-processing}
	\mbox{}\\
	\noindent 
	
	\textbf{Input:} Matrices $U, V \in \R^{d^3\times n}$.%
	
	\textbf{Output:} Matrices $U', V'\in \R^{d^3\times n}$.
	
	\begin{enumerate}
		\item Denote  $N=(U\transpose{V})_{\{5,6\}\{1,2,3,4\}}$. 
		\item Compute the $nd\times d^2$ reshaping $Z$  and the $d^2\times nd$ reshaping $\tilde{V}$ of $V$. 
		\item Compute $W=\transpose{Z}(\transpose{U}U\otimes \Id_d)Z$.
		\item Compute $H=(\Id_{d^2}-W^{-1/2})^{>0}$.
		\item Compute $L=\transpose{\tilde{V}}H$.
		\item Reshape $L$ and compute $N^{\leq 1}=U'\transpose{V'}= U\transpose{V}-U(L\otimes \Id_{d})$.
		\item Return the resulting matrices $U', V'$.
	\end{enumerate}
\end{algorithmbox}

Before presenting the proof, we first introduce some notation:
\begin{definition}
	For arbitrary matrix $M\in \mathbb{R}^{d\times d}$ with 
	eigenvalue decomposition $M=U\Sigma U^\top$,
	 we denote $M^{> t}\coloneqq U\Sigma^{> t}U^\top$,
	  where $\Sigma^{> t}$ is same as $\Sigma$ except for
	truncating entries larger than $t$ to $0$.
\end{definition}

Next we prove that the spectral truncation can be done via matrix multiplication. 
\begin{lemma}\label{lem:spectral-truncation-projector}
Consider matrices $N\in \mathbb{R}^{d^4\times d^2}$ and 
 $M\coloneqq N^\top N$.  
 Then 
$N^{\leq 1}\coloneqq N\Paren{\Id_{d^2}-\Paren{\Id_{d^2}-M^{-1/2}}^{>0}}$
 is the projection of $N$ into the set of $d^4\times d^2$ matrices 
 with spectral norm bounded by $1$
\end{lemma}
\begin{proof}
	Indeed suppose $N$ has singular value decomposition 
	$N=P\Sigma Q^\top$, then $M^{-1/2}=Q\tilde{\Sigma}^{-1} Q^\top$,
	 where $\Sigma$ is a $d^4\times d^2$ diagonal matrix and 
	 $\tilde{\Sigma}=(\Sigma^\top \Sigma)^{1/2}$.
	It follows that 
	\begin{align*}
		N\Paren{\Id_{d^2}-\Paren{\Id_{d^2}-M^{-1/2}}^{>0}} &=
		P\Sigma Q^\top\Paren{\Id_{d^2}-Q\Paren{\Id_{d^2}-\tilde{\Sigma}^{-1}}^{>0}Q^\top}
		\\
		& =
		P\Sigma\Paren{\Id_{d^2}-\Paren{\Id_{d^2}-\tilde{\Sigma}^{-1}}^{>0}}
		 Q^\top\\
		& =	P \Sigma' Q^\top	  
	\end{align*}
	where $\Sigma'\coloneqq \Sigma\Paren{\Id_{d^2}-\Paren{\Id_{d^2}-\tilde{\Sigma}^{-1}}^{>0}}
	$.
	Now we note that for each $i$, if $\Sigma_{ii}>1$, then 
	$\Sigma'_{ii}=\Sigma_{ii}\cdot \Sigma^{-1}_{ii}=1$; otherwise 
	$\Sigma'_{ii}=\Sigma_{ii}$. Therefore $P \Sigma' Q^\top$ is exactly the projection of
	$N$ into the set of $d^4\times d^2$ matrices 
	with spectral norm bounded by $1$.	
\end{proof}

We are now ready to prove \cref{lemma:running-time-pre-processing}.

\begin{proof}[Proof of \cref{lemma:running-time-pre-processing}]
	Without loss of generality, we consider the flattening  $\hat{M}_{\{5,6\},\{1,2,3,4\}}$. 
	For simplicity, we denote $N\coloneqq \hat{M}_{\{1,2,3,4\},\{5,6\}}$. 
     Let $Z$ be an appropriate $nd\times d^2$ reshaping of $V$. Since for any vector $y\in \R^{d^2}$,  we have  that $Ny$ is the flattening of $UV^\top (y\otimes \Id_d)$ into a $d^4$ dimensional vector and  $Ny=(U\otimes \Id_d)Z y$,  it follows that $N=(U\otimes \Id_d)Z$. 
	 Further, we denote
    $W\coloneqq N^\top N=Z^\top \left( U^\top U \otimes \Id_{d}\right) Z$. 
    Then the $i$-th singular value of $W$ is given by the square 
    of the $i$-th singular value of $N$. 
	
	We show that matrix $W$ can be computed in a fast way. 
	Since $U\in \mathbb{R}^{d^3\times n}$, we can compute 
	$U^\top U$ in time $n^{\omega\Paren{\frac{3\log d}{\log n}}}$. When 
	$n\leq d^{3/2}$, this is bounded by $d^{\frac{3}{2}\omega(2)}\leq d^5$.
	Then since $U^\top U$ is an $n\times n$ matrix, and $Z$ is a 
	 $nd\times d^2$ matrix, $\left( U^\top U \otimes \Id_{d}\right) Z$
	  requires $d$ distinct multiplications each between  an
	  $n\times n$ and an $n\times d^2$ matrices. Each of these multiplications takes
	   time $O\Paren{n^{\omega(\frac{2\log d}{\log n})}}$.
	   When $n\leq d^{3/2}$, this is bounded by $O(d^{5.05})$.

	By \cref{lem:spectral-truncation-projector}, the projection matrix 
	is given by
	$N^{\leq 1}=N\Paren{\Id_{d^2}-\Paren{\Id_{d^2}-W^{-1/2}}^{>0}}$.
	Now we claim that with high probability the 
	matrix $\Paren{\Id_{d^2}-W^{-1/2}}^{>0}$ 
	has rank at most $n$. Indeed since matrix $N$ has Frobenius norm 
	at most $2\sqrt{n}$, it has at most $2n$ eigenvalues at least $1$.
	Since $W=N^\top N$, it has at most $2n$ eigenvalues at least $1$ as well.
	We then can compute the eigenvalue decomposition 
	$H:=\Paren{\Id_{d^2}-
	W^{-1/2}}^{> 0}=P\Lambda^{-1/2}P^\top$ in time $O(nd^4)$. 
	
	Using this low rank representation, we show that we can compute
	matrices $U',V'\in \mathbb{R}^{d^3\times n}$ such that 
	$N^{\leq 1}= UV^\top-U'V^{\prime\top}$. Indeed, since 
	$N^{\leq 1}=UV^{\top}-U V^{\top}(H\otimes \Id_d)$, it's sufficient to
	calculate $\transpose{V}(H\otimes \Id_d)$. For this, 
	we first reshape $V$ into a $d^2\times nd$ matrix $\tilde{V}$ and then 
	do the matrix multiplication $\tilde{V}^\top H=\tilde{V}^\top P\Lambda^{-1/2}P^\top$. Tthen we can reshape  	$\tilde{V}^\top H$ 
	into an appropriate $d^3\times n$ matrix $V'$. For $U'=U$ we then have  
	 $U V^{\top}(H\otimes \Id_d)= U'V^{\prime\top}$.
	Since $P\in \mathbb{R}^{d^2\times n}$ and 
	$\tilde{V}\in \mathbb{R}^{d^2\times nd}$, when $n\leq d^{3/2}$,
	it takes time $O(d\cdot n^{\omega(4/3)})\leq d^5$.

	All in all, the total running time is bounded by $O(d^{5.05}+nd^4)$.

\end{proof}

\subsection{Running time analysis of Gaussian rounding}
\begin{lemma}[Running time of the rounding step]\label{lem:running-time-rounding}
	In each iteration of the \hyperref[step:lifting]{recovery step} in algorithm \cref{algorithm:non-robust-recovery}, 
	the \hyperref[step:rounding]{rounding step} takes 
	time at most $O\Paren{n\cdot d^4+d^{\omega\Paren{\frac{1}{2}+\frac{\log n}{2\log d}}}}
	\leq O(n\cdot d^4+d^{5.25})$.
\end{lemma}
\begin{proof}
	We divide the discussion in three steps. 
	
	\paragraph{Running time for a random contraction and taking top eigenvectors}
	We sample $\ell=\tilde{O}(d^2)$ independent random Gaussian vectors
	$g_1,g_2,\ldots,g_\ell\sim N(0,\Id_{d^2})$.    In \cref{alg:rounding}, 
	we use power method to obtain the top right
    singular vectors of $M_t(g)$ for all $t\in [\ell]$. 
	We first take random initialization vectors $x_1,x_2,\ldots,x_\ell$. 
	Then we do $\tilde{O}(1)$ power iterations. 
	In each iteration,
	we update $x_i\leftarrow (x_i\otimes\Id \otimes g_i) \mathbf{\hat{M}}$.
	
	Since for arbitrary vectors
	$x_1,x_2,\ldots,x_\ell\in \R^{d^2}$, by \cref{lem:time-tensor-contraction}, 
   we can obtain $(x_i\otimes\Id \otimes g_i) \mathbf{\hat{M}}$
   for $i\in [\ell]$ in  
   $O\Paren{n\cdot d^4+d^{2\omega\Paren{\frac{1+\log_d n}{2}}}}$
   time. Thus combining all iterations, the total running time is bounded by
   $\tilde{O}\Paren{n\cdot d^4+d^{2\omega\Paren{5/4}}}
   \leq \tilde{O}\Paren{n\cdot d^4+d^{5.25}}$ time.
   
    Next we show it's sufficient to run $\tilde{O}(1)$ power iterations 
	 to get accurate approximation of top singular vectors.
	Consider the setting of
	\cref{lem:top-singular-vector}. Suppose the matrix 
	$(\Id_{d^2}\otimes \Id_{d^2}\otimes g_i)\mathbf{\hat{M}}$ satisfy the conditions
	that
	\begin{itemize}
		\item the top singular vector $u$ recovers some component vector $a_i$: 
		$\Abs{\iprod{u, a_i^{\otimes 2}}} \geq 1 - \frac{1}{\polylog d}$
		\item the ratio between the largest and second largest singular value of $M_g$
		is larger than $1/\log\log n$.
	\end{itemize}
	Then by the second condition, after $\polylog(n)$ power iterations,
	 we will get $\Abs{\iprod{x_i,u}}\geq 1-\frac{1}{\polylog(n)}$.

	Then for these 
	top eigenvectors, we flatten them into $d\times d$ matrices 
	$B_1,B_2,\ldots,B_{\ell}\in \mathbb{R}^{d^2}$, and then take top singular vectors
	 of these matrices. This takes time at most
	 $\tilde{O}\Paren{\ell\cdot d^2}=\tilde{O}\Paren{d^4}$.
	 As a result, we obtain $O(\ell)$ candidate recovery vectors.
	
	 \paragraph{Running time for checking candidate recovery vectors}
	In \cref{alg:rounding}
	 for each of the $\ell$ candidate recovery vectors $v$, 
	we check the value of $\iprod{T,v^{\otimes 3}}$. This requires  
	 $\tilde{O}(\ell\cdot d^3)=\tilde{O}(d^5)$ time. 

	 \paragraph{Running time for removing redundant vectors}
	 We consider the running time of , which is a 
	 detailed exposition of the relevant step in 
	 \cref{alg:rounding}.
	 In each of the $\tilde{O}(d^2)$
	  iterations, we need to check the correlation of $b_i$
	  with each vector in $S'$. Since $S'$ has size at most $n$,
	   this takes time at most $O(nd)$. Therefore the total
	   running time is bounded by $\tilde{O}(nd^3)$.

	Thus in all the running time is given by $\tilde{O}\Paren{n\cdot d^4+d^{5.25}}$.
\end{proof}

\subsection{Running time analysis of accuracy boosting}
\begin{lemma}\label{lem:running-time-tensor-iteration}
	In each iteration of the \hyperref[step:lifting]{Recovery step} in algorithm \cref{algorithm:non-robust-recovery}, the \hyperref[step:lifting]{accuracy boosting step} takes time at most $\tilde{O}(n\cdot d^3)$.
\end{lemma}
\begin{proof}
	In each iteration we perform the accuracy boosting step for at most $0.99n$ vectors.
	For each such vector   we need to run $O(\log d)$ rounds of tensor    power iterations\cite{Anandkumar-pmlr15}. Since each round of tensor power iteration takes 
    $\tilde{O}(d^3)$ time, the total running time is bounded by 
    $\tilde{O}(n\cdot d^3)$.
\end{proof}

\subsection{Running time analysis of peeling}
The last operation in each iteration of \hyperref[step:lifting]{Recovery step} in algorithm \cref{algorithm:non-robust-recovery} consists of "peeling off" the components just learned and obtain an implicit representation of the modified data.
$UV^\top-\sum_{i=1}^{0.99n}b_i^{\otimes 3}\Paren{b_i^{\otimes 3}}^\top$,
 and obtain the implicit representation.
\begin{lemma}\label{lem:running-time-top-eigenspace}
	Let $\epsilon, \delta >0$ and let $m<n$ be positive integers.
	Let $a_1,\ldots, a_n\in \R^d$ be any subset of \iid vectors uniformly sampled from the unit sphere in  $\R^d$.
	Let $U, V\in \R^{d^3\times n}$ be such that
	\begin{align*}
		\Norm{U\transpose{V}- \sum_{i\in [n]}\dyad{(a_i^{\otimes 3})}}\leq \epsilon\,.
	\end{align*}
	Let $b_1\ldots, b_m \in \R^d$ be such that
	\begin{align*}
		\forall i \in [m]\,,\quad \iprod{a_i, b_i}\geq 1-1/\polylog(d)\,.
	\end{align*}
	Then there exists an algorithm (a slight variation of \cref{algorithm:compute-implicit-representation}) that, given $b_1\,,\ldots, b_m\,, U,V$, computes $U',V'\in \R^{d^3\times n-m}$ satisfying
	\begin{align*}
		\Norm{U'\transpose{(V')}- \sum_{i>m}^{n}\dyad{(a_i^{\otimes 3})}}\leq O(\epsilon)\,.
	\end{align*}
	Moreover, the algorithm runs in time $\tilde{O}\Paren{d^{2\cdot \omega(1+\log n/\log d^2)}}$, where $\omega(k)$ is the time required to multiply a $(d^k\times d)$ matrix with a $(d\times d)$ matrix.\footnote{See  \cref{fig:matrix-multiplication-table} in \cref{section:matrix-multiplication-constants}.} 
	\begin{proof}
	$\sum_{i\in [m]} \dyad{(\tensorpower{b_i}{3})}$ can be written as tensor networks as in \cref{fig:techniques-order-six-network}. 
	On the other hand 	multiplying $U\transpose{V}$ by a $d^3$-dimensional vector takes time at most $\tilde{O}(n^{\omega(2)})\leq \tilde{O}(d^{4.9})$.
	Thus, as in \cref{lemma:running-time-lifting-step}, we can compute the top $n-m$ eigenspace of their difference in time $\tilde{O}\Paren{d^{2\cdot \omega(1+\log n/2\log d)}}$.
	By \cref{lem:removal-process} the result follows.
	\end{proof}
\end{lemma}

\subsection{Putting things together}

We are now ready to prove \cref{thm:main-running-time}.

\begin{proof}[Proof of \cref{thm:main-running-time}]
	By lemma \cref{lemma:running-time-lifting-step}, the lifting step 
	of \cref{algorithm:non-robust-recovery} can
	 done in $\tilde{O}\Paren{d^{2\cdot \omega(1+\log n/2\log d)}}$ time. 
	 Combining \cref{lemma:running-time-pre-processing},
	\cref{lem:running-time-rounding} \cref{lem:running-time-tensor-iteration}, 
	and 
	\cref{lem:running-time-top-eigenspace}, each iteration of the step 2 in
	\cref{algorithm:non-robust-recovery} can be done in time
	$O\Paren{n\cdot d^4+d^{2\cdot \omega(1+\log n/2\log d)}}$.
	There are at most $O(\log n)$ iterations, and thus the total running time of the loop
	 is bounded by 
	 $\tilde{O}\Paren{d^{2\cdot \omega(1+\log n/2\log d)}+nd^4}$.
\end{proof}

\phantomsection
\addcontentsline{toc}{section}{References}
\bibliographystyle{amsalpha}
\bibliography{bib/necessary}

\appendix
\section{Partial recovery from reducing to robust fourth-order decomposition}\label{ap:hss19}

We observed that the tensor network in \cref{fig:order-4-tensor-networks}(b) allows us to partially reduce the problem of third-order tensor decomposition to the problem of robust fourth-order tensor decomposition. 
A natural idea would thus be to apply existing algorithms, e.g., \cite{HopkinsSS19}, to this latter problem.
However, such a black-box reduction faces several issues:
First, the spectral norm of the noise of the network in \cref{fig:order-4-tensor-networks}(b) can only be bounded by $1/\polylog(d)$.
For this amount of noise, the algorithm in \cite{HopkinsSS19} can only recover a constant fraction, bounded away from 1, of the components, but not all of them.
It is unclear, if their analysis can be adapted to handle larger amount of noise, since they deal with the inherently harder setting of adversarial instead of random noise.
Second, the running time of this black-box reduction would be $\tilde{O}(n\cdot d^5)$,\footnote{We remark that the main result in \cite{HopkinsSS19} contains a minor imprecision concerning the running time.  In particular, their algorithm runs in time $\tilde{O}(n\cdot d^5)$ while their result states $\tilde{O}(n^2d^3)$ time. In the context of our interest this is a meaningful difference as $n/d^2 = o(1/\sqrt{d})$.} which is $\tilde{O}(d^{6.5})$ for $n=\Theta(d^{3/2}/\polylog(d))$.
This is even slower than our nearly-quadratic running time of $\tilde{O}\Paren{d^{6.043182}}$.
Lastly, their analysis is quite involved and we argue that the language of tensor networks captures the essence of the third-order problem and thus yields a considerably simpler algorithm than this black-box reduction.

\section{Boosting to arbitrary accuracy}
Given good initialization vector for every component,
 it is shown in \cite{Anandkumar-pmlr15} that we can get arbitrarily accurate 
 estimation of the components by combining the tensor power iteration algorithm 
 and residual error removal: 
\begin{theorem}[Theorem 1 in \cite{Anandkumar-pmlr15}]\label{thm:full-recovery-boosting}
    Suppose we are given tensor $T=\sum_{i=1}^n a_i^{\otimes 3}$,
    where $n=O\Paren{d^{3/2}/\polylog(d)}$ and $a_1,a_2,\ldots, a_n$ are independent
     and uniformly sampled from the unit sphere and 
     $\lambda_i=1\pm o(1)$. Then given vectors
     $b_1,b_2,\ldots,b_n$ s.t $\iprod{a_i,b_i}\geq 0.99$,
     there is a polynomial time algorithm outputting unit norm vectors
     $c_1,c_2,\ldots,c_n$ s.t
     \begin{equation*}
         \iprod{c_i,a_i}\geq 1-\epsilon
     \end{equation*} 
\end{theorem}
Combining with~\cref{thm:Full-recovery} in this section, we thus get the following corollary
\begin{corollary}
    Suppose we are given tensor $T=\sum_{i=1}^n a_i^{\otimes 3}$,
    where $n=O\Paren{d^{3/2}/\polylog(d)}$ and 
    $a_1,a_2,\ldots, a_n$ are independently and uniformly sampled from the
     dimension $d$ unit sphere, then there is a $\poly(d)$-time algorithm
    outputting unit norm vectors $b_1,b_2,\ldots,b_n\in\mathbb{R}^d$ such that probability
    $1-o(1)$ over $a_1,a_2,\ldots,a_n$, for each $i\in [n]$, 
    $\max_{j\in [n]}\iprod{a_i,b_j}\geq (1-2^{-n})\norm{a_i}$.
\end{corollary}
\section{Concentration bounds}\label{section:concentration-bounds}

\subsection{Concentration of Gaussian polynomials}
\begin{fact}\label{lem:Gaussian-tail}[Lemma A.4 in \cite{HopkinsSSS16}]
 Let $X \sim \mathcal{N}(0,1)$. Then for $t>0$,
$$
\mathbb{P}(X>t) \leqslant \frac{e^{-t^{2} / 2}}{t \sqrt{2 \pi}}
$$
and
$$
\mathbb{P}(X>t) \geqslant \frac{e^{-t^{2} / 2}}{\sqrt{2 \pi}} \cdot\left(\frac{1}{t}-\frac{1}{t^{3}}\right)
$$
\end{fact}
\begin{proof}
	We record their proof for completeness. For the first statement,
	 we have
	 $$
\begin{aligned}
\mathbb{P}(X>t) &=\frac{1}{\sqrt{2 \pi}} \int_{t}^{\infty} e^{-x^{2} / 2} d x \\
& \leqslant \frac{1}{\sqrt{2 \pi}} \int_{t}^{\infty} \frac{x}{t} e^{-x^{2} / 2} d x \\
&=\frac{e^{-t^{2} / 2}}{t \sqrt{2 \pi}}
\end{aligned}
$$
For the second statement, we have
$$
\begin{aligned}
\mathbb{P}(X>t) &=\frac{1}{\sqrt{2 \pi}} \int_{t}^{\infty} e^{-x^{2} / 2} d x \\
&=\frac{1}{\sqrt{2 \pi}} \int_{t}^{\infty} \frac{1}{x} \cdot x e^{-x^{2} / 2} d x \\
&=\frac{1}{\sqrt{2 \pi}}\left[-\frac{1}{x} e^{-x^{2} / 2} \cdot\right]_{t}^{\infty}-\frac{1}{\sqrt{2 \pi}} \int_{t}^{\infty} \frac{1}{x^{2}} \cdot e^{-x^{2} / 2} d x \\
& \geqslant \frac{1}{\sqrt{2 \pi}}\left[-\frac{1}{x} e^{-x^{2} / 2} \cdot\right]_{t}^{\infty}-\frac{1}{\sqrt{2 \pi}} \int_{t}^{\infty} \frac{x}{t^{3}} \cdot e^{-x^{2} / 2} d x \\
&=\frac{1}{\sqrt{2 \pi}}\left(\frac{1}{t}-\frac{1}{t^{3}}\right) e^{-t^{2} / 2}
\end{aligned}
$$
\end{proof}

\begin{lemma}[Lemma A.5 in \cite{HopkinsSSS16}]\label{lem:Gaussian-polynomial-tail}
	For each $\ell \geqslant 1$ there is a universal constant $c_{\ell}>0$ 
	such that for every $f$ a degree- $\ell$ polynomial of standard Gaussian random variables $X_{1}, \ldots, X_{m}$ and $t \geqslant 2$
	$$
	\mathbb{P}(|f(X)|>t \mathbb{E}|f(X)|) \leqslant e^{-c_{\ell} t^{2 / \ell}}
	$$
	The same holds (with a different constant $c_{\ell}$ ) if 
	$\mathbb{E}|f(x)|$ is replaced by 
	$\left(\mathbb{E} f(x)^{2}\right)^{1 / 2}$.	
\end{lemma}

\begin{lemma}[Fact C.1 in \cite{HopkinsSSS16}]\label{lem:Gaussian-vector-norm}
	Suppose $a_1,a_2,\ldots,a_n$ are independently sampled from $N(0,\frac{1}{d}\Id_d)$, 
	then with probability $1-n^{-\omega(1)}$, we have
	\begin{enumerate}[(a)]
		\item for each $i\in n$, $\norm{a_i}^2=1\pm \tilde{O}\Paren{\frac{1}{\sqrt{d}}}$
		\item for each $i,j\in n$, $i\neq j$, we have $\iprod{a_i,a_j}^2=\tilde{O}\Paren{\frac{1}{d}}$	 
	\end{enumerate}
\end{lemma}

\subsection{Concentration of random matrices}
\begin{lemma}\label{lem:concentration-input-vectors}
	For $n\leq d^{3/2}/\polylog d$, let $a_1, \cdots, a_n$ be $n$ i.i.d. 
	random unit vectors  
	\begin{enumerate}[(a)]
		\item For any $i\neq j$,
		\[\Abs{\Iprod{a_i, a_j}}\overset{w.ov.p}{=}\tilde{O}\Paren{\frac{1}{\sqrt{d}}}.\]
		\item 
		\[\Norm{\sum_{i=1}^na_ia_i^\top} \overset{w.ov.p}{=} \tilde{O}\Paren{\frac{n}{d}}.\]
		\item \[\Norm{\sum_{i=1}^n a_i^{\otimes 2} \Paren{a_i^{\otimes 2}}^\top} \overset{w.ov.p}{=} \tilde{O}\Paren{\frac{n}{d}}.\]
		\item \[\Norm{\sum_{i=1}^na_i^{\otimes 3} \Paren{a_i^{\otimes 3}}^\top} \overset{w.ov.p}{=} 1\pm \tilde{O}\Paren{\frac{n}{d^{3/2}}}.\]
	\end{enumerate}
\end{lemma}
\begin{proof}
	\begin{enumerate}[(a)]
		\item We rewrite $a_i=\frac{b_i}{\norm{b_i}}$,
		where $b_1,b_2,\ldots,b_n\sim N(0,\frac{1}{d}\Id_d)$ are 
		independent. Then $\iprod{a_i, a_j}=\frac{\iprod{b_i,b_j}}{\norm{b_i}\norm{b_j}}$.
		Now using lemma~\ref{lem:Gaussian-vector-norm}, we have the claim. 
			 
		\item We rewrite $a_i=\frac{b_i}{\norm{b_i}}$,
		where $b_1,b_2,\ldots,b_n\sim N(0,\frac{1}{d}\Id_d)$ are 
		independent.
		Then by fact C.2 in~\cite{HopkinsSSS16}, with overwhelming probability, 
		we have $\Norm{\sum_{i=1}^n b_ib_i^\top}\leq \tilde{O}\Paren{\frac{n}{d}}$.
		Now by lemma~\ref{lem:Gaussian-vector-norm}, we have 
		$$\Norm{\sum_{i=1}^n a_ia_i^\top}\leq \tilde{O}\Paren{\frac{n}{d}}$$

		\item Let $U\in \mathbb{R}^{d^2 \times n}$ be a matrix with $i$-th row
		given by $a_i^{\otimes 2}$, then we have
		\begin{equation*}
			\Norm{\sum_{i=1}^n a_i^{\otimes 2} \Paren{a_i^{\otimes 2}}^\top}
			= \Norm{UU^\top}= \Norm{U^\top U}
		\end{equation*}
		Now we have $(U^\top U)_{ii}=\iprod{a_i,a_i}^2=1$, and by (a)
		$(U^\top U)_{ij}= \iprod{a_i,a_j}^2=\tilde{O}\Paren{\frac{1}{d}}$.
		Thus by Gershgorin circle theorem, we have
		$$\norm{U^\top U}\leq \max_{i\in [d^2]} \sum_{j\in [d^2]}
		\Abs{(U^\top U)_{ij}}=\tilde{O}\Paren{\frac{n}{d}}$$
		
		\item Let $U\in \mathbb{R}^{d^3 \times n}$ be a matrix with $i$-th row
		given by $a_i^{\otimes 3}$, then we have
		\begin{equation*}
			\Norm{\sum_{i=1}^n a_i^{\otimes 3} \Paren{a_i^{\otimes 3}}^\top}
			= \Norm{UU^\top}= \Norm{U^\top U}
		\end{equation*}
		Now we have $(U^\top U)_{ii}=\iprod{a_i,a_i}^3=1$, and by (a) with overwhelming
		probability
		$(U^\top U)_{ij}= \iprod{a_i,a_j}^3=\tilde{O}\Paren{\frac{1}{d^{3/2}}}$.
		Thus by Gershgorin circle theorem, we have
		$$\norm{U^\top U}\leq \max_{i\in [d^3]} \sum_{j\in [d^3]}
		\Abs{(U^\top U)_{ij}}=\tilde{O}\Paren{\frac{n}{d^{3/2}}}$$
	\end{enumerate}
\end{proof}

\begin{corollary}\label{cor:random_sign-concentration-input-vectors}
		For $n\leq d^{3/2}/\polylog d$, let $a_1, \cdots, a_n$ be $n$ i.i.d. random unit vectors, and let $s_1, \ldots, s_n$ be independent random signs. 
	\begin{enumerate}[(a)]
		\item 
		\[\Norm{\sum_{i=1}^ns_i\cdot a_ia_i^\top} \overset{w.ov.p}{=} \tilde{O}\Paren{\sqrt{\frac{n}{d}}+1}.\]
		\item \[\Norm{\sum_{i=1}^ns_i\cdot a_i^{\otimes 2} \Paren{a_i^{\otimes 2}}^\top} \overset{w.ov.p}{=} \tilde{O}\Paren{\sqrt{\frac{n}{d}}}.\]
		\item \[\Norm{\sum_{i=1}^n s_i\cdot a_i^{\otimes 3} \Paren{a_i^{\otimes 3}}^\top} \overset{w.ov.p}{=} \tilde{O}(1).\]
	\end{enumerate}
\end{corollary}

\begin{lemma}[Lemma 5.9 in~\cite{HopkinsSSS16}]
	For $R=\sqrt{2}\Paren{\E_{a\sim N(0,\Id_d)}(aa^\top)^{\otimes 2}}^{+1/2}$, 
	 denote $\Phi=\sum_{i} e_{i}^{\otimes 2} \in \mathbb{R}^{d^{2}}$, 
	 (a) we have $\norm{R}=1$ and moreover
	 \begin{equation*}
		R=\sqrt{2}\left(\Sigma^{+}\right)^{1 / 2}=
		\Pi_{\mathrm{sym}}-\frac{1}{d}\left(1-\sqrt{\frac{2}{d+2}}\right) \Phi \Phi^{\top}
	 \end{equation*}
	 (b) for any $v\in \mathbb{R}^d$,
	 \begin{equation*}
		\|R(v \otimes v)-v\otimes v\|_{2}^{2}=\left(\frac{1}{d+2}\right) \cdot\|v\|^{4}
	 \end{equation*}
\end{lemma}
\begin{proof}
	(a) has been proved in Lemma 5.9 of~\cite{HopkinsSSS16}. For (b), 
	without loss of generality, we assume $\norm{v}=1$.
	Then we have
	\begin{equation*}
		R(v \otimes v)-v\otimes v=-\frac{1}{d}\left(1-\sqrt{\frac{2}{d+2}}\right) \iprod{\Phi^{\top},v\otimes v} \Phi
	\end{equation*}
	 Since $\norm{\Phi}=\sqrt{d}$ and $\iprod{\Phi^{\top},v\otimes v}=\sum_{i=1}^d \iprod{v,e_i}^2=1$,
	 we have
	 \begin{equation*}
		 \Norm{R(v \otimes v)-v\otimes v}^2 =\frac{1}{d+2}
	 \end{equation*}
	 which concludes the proof.
\end{proof}

\begin{lemma}\label{lem:isotropic-norm}[Similar to Lemma 5.11 in \cite{HopkinsSSS16}]
	Let $a_{1}, \ldots, a_{n}\in \mathbb{R}^d$ independently and uniformly sampled from the unit sphere. 
	Let $R=\sqrt{2}\cdot \Paren{\E \Paren{(aa^\top)^{\otimes 2}}}^{+1/2}$. 
	Let $u_{i}=a_{i} \otimes a_{i}$. With overwhelming probability, every $j \in [n]$ satisfies 
	(a) 
	$\sum_{i \neq j}\left\langle u_{j}, R^{2} u_{i}\right\rangle^{2}=
	\tilde{O}\left(n / d^{2}\right)$ 
	(b) 
	$\norm{R u_{j}-u_j}^{2} \le \tilde{O}\Paren{\frac{1}{d}}$.
\end{lemma}
\begin{proof}
(a) We follow the same proof as in the lemma 5.11 of \cite{HopkinsSSS16}
(which is for $a_{1}, \ldots, a_{n}\sim N(0,\Id_d)$):
$$
\begin{aligned}
\sum_{i \neq j}\left\langle u_{j}, R^{2} u_{i}\right\rangle^{2} &
=\sum_{i \neq j}\left\langle u_{j},\left(\Pi_{\text {sym }}-\frac{1}{d+2} \Phi \Phi^{\top}\right) u_{i}\right\rangle^{2} \\
&=\sum_{i \neq j}\left(\left\langle a_{j}, a_{i}\right\rangle^{2}-\frac{1}{d+2}\left\|u_{j}\right\|^{2}\left\|u_{i}\right\|^{2}\right)^{2}\\
&= \sum_{i \neq j} \tilde{O}(1 / d)^{2}\\
&= \tilde{O}\left(n / d^{2}\right)\,.
\end{aligned}
$$

(b) This follows directly from Lemma 5.9(b) by replacing 
 $v$ with $a_i$. 
\end{proof}

\begin{lemma}[Lemma 5.9 in \cite{HopkinsSSS16}]\label{lem:isotropic-4order-spectral-norm}
	For $n\leq \tilde{O}\Paren{d^{3/2}}$, let $R=\sqrt{2} \Paren{\E a^{\otimes 2}\Paren{a^{\otimes 2}}^\top}^{-1/2}$ where 
	$a\sim N(0,\Id)$, and $a_1,a_2,\ldots,a_n\in \mathbb{R}^{d}$ be i.i.d random vectors 
	sampled uniformly from the unit sphere. Then with probability at least $1-\tilde{O}\Paren{\frac{n}{d^{3/2}}}$,
	 we have
	 \begin{equation*}
		 \Norm{\sum_{i=1}^n Ra_i^{\otimes 2}\Paren{Ra_i^{\otimes 2}}^\top-\Pi}
		 \leq \tilde{O}\Paren{\frac{n}{d^2}}
	 \end{equation*}
	 where $\Pi$ is the projection matrix to the span of $\Set{Ra_i^{\otimes 2}}$.
\end{lemma}

\begin{lemma}\label{lem:rectangle-matrix-norm}
	For vectors $a_1,a_2,\ldots,a_n\in \mathbb{R}^d$ 
	sampled uniformly at random from unit sphere, 
	and $R=\sqrt{2} \Paren{\E a^{\otimes 2}\Paren{a^{\otimes 2}}^\top}^{+1/2}$,
	we have
	\begin{equation*}
		\Norm{\sum_{i=1}^n Ra_i^{\otimes 2}\Paren{(Ra_i^{\otimes 2})^{\otimes 2}}^\top}
		\leq 1+\tilde{O}\Paren{\frac{n}{d^{3/2}}}
	\end{equation*}
\end{lemma}
\begin{proof}
	Let $U\in \mathbb{R}^{n\times d^2}$ be a matrix with the $i$-th row vector given By
	 $Ra_i^{\otimes 2}$, and let $V\in \mathbb{R}^{n\times d^4}$
	be a matrix with the $i$-th row vector given by $\Paren{Ra_i^{\otimes 2}}^{\otimes 2}$.
	 Then we have $\sum_{i=1}^n Ra_i^{\otimes 2}\Paren{(Ra_i^{\otimes 2})^{\otimes 2}}^\top=UV^top$.
	 Our strategy is then to bound $\norm{U}$ and $\norm{V}$.
	 
	 First with high probability we have 
	 \begin{equation*}
		\norm{U}=\sqrt{\norm{UU^\top}}
	 =\sqrt{\Norm{\sum_{i=1}^n Ra_i^{\otimes 2}\Paren{Ra_i^{\otimes 2}}^\top}}
	 \leq 1+ \tilde{O}\Paren{\frac{n}{d^{3/2}}}
	 \end{equation*}
	 Second with high probability have 
	 \begin{equation*}
	\norm{V}=\sqrt{\norm{VV^\top}}=\sqrt{
		\Norm{\sum_{i=1}^n \Paren{Ra_i^{\otimes 2}\Paren{Ra_i^{\otimes 2}}^\top}^{\otimes 2}}}
	\leq 1+\tilde{O}\Paren{\frac{n}{d^{3/2}}}
	 \end{equation*}
	 It then follows that
	 $\norm{UV^\top}\leq \norm{U}\norm{V}\leq 1+\tilde{O}\Paren{\frac{n}{d^{3/2}}}\,.$
\end{proof}

\begin{lemma}[Concentration of random tensor contractions~\cite{MaSS16}]
	\label{lem:random-tensor-contraction}
	Let $g$ be a standard Gaussian vector in 
	$\mathbb{R}^{k}, g \sim \mathcal{N}\left(0, \mathrm{Id}_{k}\right) .$
	Let $A$ be a tensor in $\left(\mathbb{R}^{k}\right) \otimes\left(\mathbb{R}^{\ell}\right) 
	\otimes\left(\mathbb{R}^{m}\right)$,
	and call the three modes of $A \alpha, \beta, \gamma$ respectively. 
	Let $A_{i}$ be a $\ell \times m$ slice of $A$ along mode $\alpha .$ 
	Then,
	$$
	\mathbb{P}\left[\left\|\sum_{i=1}^{k} g_{i} A_{i}\right\| 
	\geqslant t \cdot \max \left\{\left\|A_{\{\alpha \beta\}\{\gamma\}}
	\right\|,\left\|A_{\{\alpha \gamma\}\{\beta\}}\right\|\right\}\right] 
	\leqslant(m+\ell) \exp \left(-\frac{t^{2}}{2}\right)
	$$
\end{lemma}

\subsection{Rademacher bounds on general matrices}

\begin{theorem}(Follows directly from \cite[Theorem~4.6.1]{Tropp_matrix_concentration})
\label{thm:matrix_rademacher}
Let $A_1, \ldots, A_n$ be a sequence of symmetric matrices with dimension $d^{\Theta(1)}$ and let $s_1, \ldots, s_n$ be a sequence of \iid Rademacher random variables.
Let $Y = \sum_{i=1}^n s_i \cdot A_i$ and $v(Y) = \Normt{\sum_{i=1}^n A_i^2}$.
Then with overwhelming probability
\begin{align*}
\Normt{Y} \leq \tilde{O} \Paren{\sqrt{v(Y)}}.
\end{align*}
\end{theorem}

\begin{lemma}\label{lem:decoupling-tensor_product-Hermitian_matrices}(\cite[Corollary~5.5]{HopkinsSSS16})
	Let $s_1,\ldots, s_n$ be independent random signs. Let $A_1,\ldots, A_n$ and $B_1,\ldots, B_n$ be  Hermitian matrices.
	Then, w.ov.p., 
	\[\Norm{\sum_{i\in[n]}s_i\cdot A_i\otimes B_i}\leq \tilde{O}\Paren{\Paren{\max_{i\in [n]}\norm{B_i}}\cdot \Norm{\sum_{i\in[n]}A_i^2}^{\frac{1}{2}}}.\]
\end{lemma}

The next lemma doesn't contain any randomness but it's very similar to the one above and used in the same context, so we will also list it here.

\begin{lemma}
\label{lemma:lin_alg_fact_1}
For $i = 1, \ldots, n$ let $A_i, B_i$ be symmetric matrices and suppose that for all $i$ we have that $A_i$ is psd.
Then $\Norm{\sum_{i=1}^n A_i \otimes B_i} \leq \Paren{\max_{i \in [n]} \Norm{B_i}} \cdot \Norm{\sum_{i=1}^n A_i}$.
\end{lemma}
\begin{proof}
Let $b = \max_{i \in [n]} \Norm{B_i}$.
For each $i$ we have that $A_i \otimes B_i \preceq A_i \otimes b\cdot \Id$ since $A_i$ and $b \cdot \Id - B_i$ are psd and the Kronecker product of two psd matrices is also psd.
By summing over all $i$ we get that $\sum_{i=1}^n A_i \otimes B_i \preceq b \cdot \Paren{\sum_{i=1}^n A_i} \otimes \Id$ which implies the claim.
\end{proof}

Finally we use a decoupling lemma from probability theory. A special version of this lemma
 has been used in~\cite{HopkinsSSS16}.
\begin{theorem}[Theorem 1 in \cite{delaPenaMS95}]\label{thm:decoupling-inequality}
	For any constant $k$,
	let $s,\left\{s^{(1)}\right\},\left\{s^{(2)}\right\}, \ldots, s^{(\ell)}\in \{\pm 1\}^d$ be 
	independent Rademacher vectors. Let $\left\{M_{i_1,i_2,\ldots,i_\ell}: i_1,i_2,\ldots,i_\ell\in [d]\right\}$ 
	be a family of matrices. 
	Then there is constant $C$ which depends only on $k$, so that for every $t>0$,
	$$
	\mathbb{P}\left(\left\|\sum_{0\leq i_1 \neq i_2\neq \ldots \neq i_\ell\leq d} 
	s_{i_1} s_{i_2} \ldots s_{i_\ell} M_{i_1, i_2,\ldots,i_\ell}\right\|_{o p}>t\right) 
	\leqslant C \cdot \mathbb{P}\left(C\left\|\sum_{0\leq i_1 \neq i_2\neq \ldots \neq i_\ell\leq d} 
	s^{(1)}_{i_1} s^{(2)}_{i_2}\ldots s^{(\ell)}_{i_\ell} M_{i_1, i_2,\ldots,i_\ell}\right\|_{o p}>t\right)
	$$
\end{theorem}

\subsection{Optimizer of tensor injective norm}

\begin{lemma}[Lemma 5.20 in \cite{HopkinsSSS16}]
	\label{lem:tensor-vector-correlation}
	Let $T=\sum_{i \in[n]} a_{i} \otimes a_{i} \otimes a_{i}$ 
	for normally distributed vectors $a_{i} \sim \mathcal{N}\left(0, \frac{1}{d} \mathrm{Id}_{d}\right) .$
	For all $0<\gamma, \gamma^{\prime}<1$, 
	\begin{itemize}
		\item With overwhelming probability, 
		for every $v \in \mathbb{R}^{d}$ 
		such that $\sum_{i \in[n]}\left\langle a_{i}, v\right\rangle^{3} \geqslant 1-\gamma$
		$$
		\max _{i \in[n]}\left|\left\langle a_{i}, v\right\rangle\right| \geqslant 1-O(\gamma)-\tilde{O}\left(n / d^{3 / 2}\right)
		$$
		\item With overwhelming probability over $a_{1}, \ldots, a_{n}$,
		if $v \in \mathbb{R}^{d}$ with $\|v\|=1$ satisfies
		$\left\langle v, a_{j}\right\rangle \geqslant 1-\gamma^{\prime}$
		for some $j$ then $\sum_{i}\left\langle a_{i}, v\right\rangle^{3} \geqslant 1-O\left(\gamma^{\prime}\right)-\tilde{O}\left(n / d^{3 / 2}\right)$
	\end{itemize}
\end{lemma}

\section{Linear algebra}\label{sec:linear-algebra}
In this section, we record some linear algebra 
 facts and results used in the paper.

\begin{lemma}\label{lem:bound-noise-signal-iprod}
    For $n'=\tilde{O}(d^{3/2})$ and $d\leq n\leq n'$,
    suppose vectors $b_2,\ldots,b_n$ satisfy $\Norm{M-\Pi}\leq 
    \tilde{O}\Paren{\frac{n'}{d^{3/2}}}$, where $M=\sum_{i=1}^n b_ib_i^\top$ and 
    $\Pi$ is the projection matrix to 
    the span of $\Set{b_i:i\in [2,n]}$. 
     Then we have $\Norm{M^\top M-M}\leq \tilde{O}\Paren{\frac{n'}{d^{3/2}}}$.
\end{lemma}
\begin{proof}
    Since $\Pi^2=\Pi$, we have
    $M^\top M-\Pi=M^\top M-M^\top \Pi+M^\top \Pi-\Pi
    =M^\top (M-\Pi)+(M-\Pi)^\top \Pi$. Since $\norm{\Pi}=1$ and 
    $\norm{M}\leq \norm{\Pi}+\norm{\Pi-M}\leq 1+\tilde{O}\Paren{\frac{n}{d^{3/2}}}$,
     it follows that $\norm{M^\top (M-\Pi)+(M-\Pi)^\top \Pi}
     \leq \norm{M-\Pi}\Paren{\norm{M^\top}+\norm{\Pi}}\leq 
     \tilde{O}\Paren{\frac{n}{d^{3/2}}}$ and we have the claim.
\end{proof}

\subsection{Fast SVD algorithm}
For implementation, we use the lazy SVD algorithm 
from~\cite{NIPS2016_c6e19e83}. 
\begin{lemma}[Implicit gapped eigendecomposition;
    Lemma 7 in~\cite{HopkinsSS19}, Corollary 4.4 in~\cite{NIPS2016_c6e19e83}]
    \label{lem:LazySVD}
    Suppose a symmetric matrix $M \in$ $\mathbb{R}^{d \times d}$ 
    has an eigendecomposition $M=\sum_{j} \lambda_{j} v_{j} v_{j}^{\top}$, 
    and that $M x$ may be computed within $t$ time steps for 
    $x \in \mathbb{R}^{d}$. Then $v_{1}, \ldots, v_{n}$ and 
    $\lambda_{1}, \ldots, \lambda_{n}$ may be computed in time 
    $\tilde{O}\left(\min \left(n(t+nd) \delta^{-1 / 2}, d^{3}\right)\right)$, 
    where $\delta=\left(\lambda_{n}-\lambda_{n+1}\right) / \lambda_{n} .$ 
    The dependence on the desired precision is polylogarithmic.   
\end{lemma}

\section{Fast matrix multiplications and tensor contractions}\label{section:matrix-multiplication-constants}

To easily compute the running time of \cref{theorem:main} under a specific set of parameters $n,d$, we include here a table (\cref{fig:matrix-multiplication-table}) from \cite{GallU18} with upper bounds on rectangular matrix multiplication constants.
We remind the reader that basic result in algebraic complexity theory states that the algebraic complexities of the following three problems are the same:  
\begin{itemize}
	\item computing  a $(d^{k}\times d)\times (d\times d)$  matrix multiplication,
	\item computing  a $(d\times d^k)\times (d^k\times d)$  matrix multiplication,
	\item computing  a $(d\times d)\times (d\times d^k)$  matrix multiplication.
\end{itemize}

\begin{figure}
	\centering	\includegraphics[width=15cm]{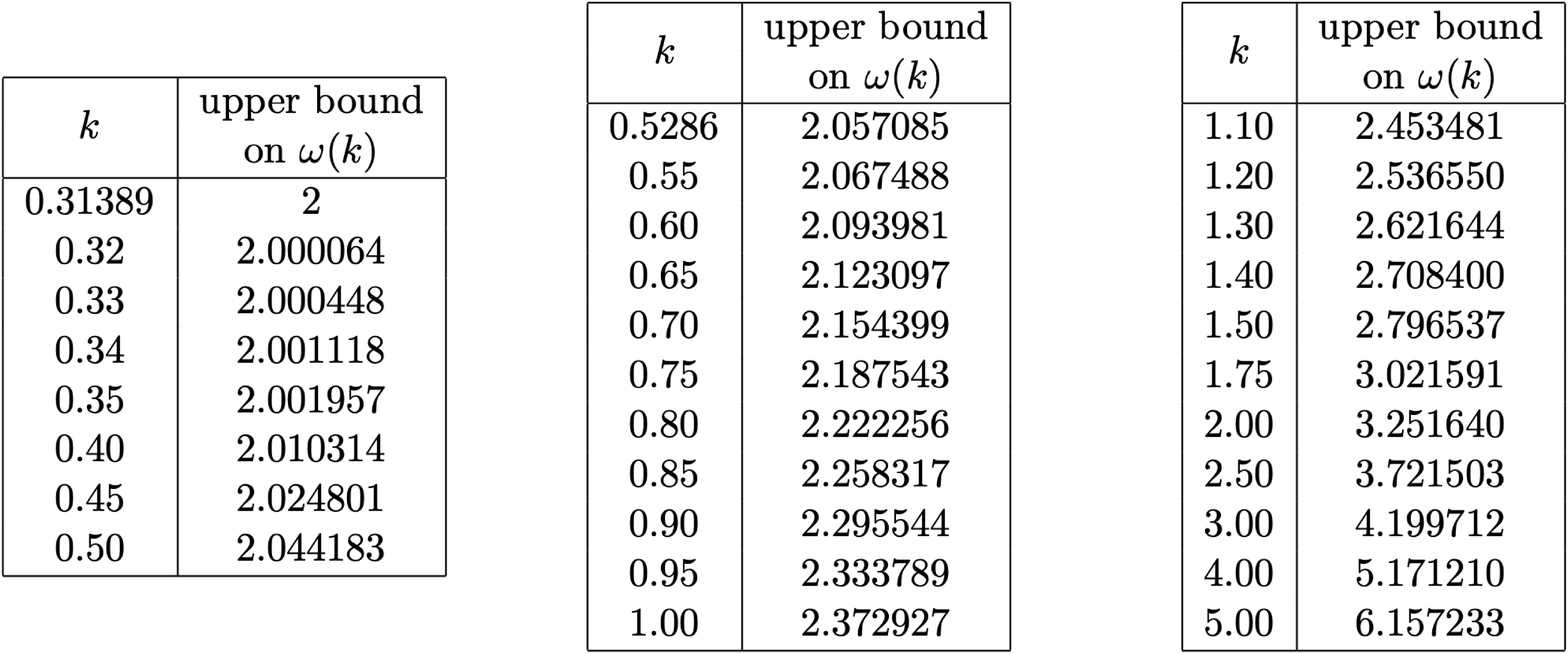}
	\caption{$\omega(k)$ denotes the exponent of the multiplication of an$(d\times d^k)$ by a $ (d^k\times d)$ matrix, so that the running time is $O(d^{\omega(k)})$. } 
	\label{fig:matrix-multiplication-table}
\end{figure}

\subsection{Fast algorithms for low rank tensors}
We state the running time for some common tensor operations
given implicit representation. The proofs are very similar to the lemma 8
in \cite{HopkinsSS19}.

The first lemma is about computing tensor contraction.
\begin{lemma}[Time for computing tensor contraction]\label{lem:time-tensor-contraction}
	Let $\ell\in \tilde{O}(d^2)$.
	Suppose we are given $U,V\in \mathbb{R}^{d^3\times n}$. Consider
	vectors $x_1,x_2,\ldots x_\ell\in \mathbb{R}^{d^2}$ 
	,$g_1,g_2,\ldots,g_\ell\in \mathbb{R}^{d^2}$, and tensor 
	$\mathbf{T}\in \Paren{\mathbb{R}^{d}}^{\otimes 6}$ satisfying
	$T_{\{1,2,3\}\{4,5,6\}}=UV^\top$. Then there is an algorithm computing 
	$\Paren{x_i^\top\otimes\Id_{d^2}\otimes g_i^\top}\mathbf{T}$ for all $i\leq \ell$,
	in $O\Paren{n\cdot d^4+d^{2\Paren{\omega\Paren{\frac{1}{2}(1+\log_{d}n)}}}}$
	time. When $n=O\Paren{d^{3/2}/\polylog(d)}$, this is bounded by 
	$\tilde{O}(n\cdot d^4+d^{5.25})$ time.
\end{lemma}
\begin{proof}
	Since $\Paren{x_i^\top\otimes\Id_{d^2}\otimes g_i^\top}\mathbf{T}
	= \Paren{x_i\otimes \Id_{d}}^\top UV^\top (g_i\otimes \Id_{d})$,
	we only need to obtain $Y_i=V^\top (g_i\otimes \Id_{d})$
	and $Z_i=U^\top \Paren{x_i\otimes \Id_{d}}$ for all $i\in [\ell]$,
	and then compute $Z_i^\top Y_i$ for all $i\in [\ell]$. 
	Since $Y_i\in \R^{n\times d}$ and $Z_i\in \R^{n\times d}$, the last step
	takes time $n\cdot d^2\cdot\ell=\tilde{O}(n\cdot d^4)$.
	
	To obtain $Y_i$ for all $i\in [\ell]$, we construct a $d^2\times \ell$ matrix $G$,
	whose $i$-th column is given by $g_i$. 
	Then $Y_i$ can all be obtained as 
	sub-matrix of $M_1=V^\top (G\otimes \Id_{d})$. 
	We write $V$ as a block matrix:
	$V^\top=(V_1^\top,V_2^\top,\ldots,V_d^\top)$ where 
	$V_1,V_2,\ldots,V_d\in \R^{n\times d^2}$.
	Then $M_1$ is equivalent to a
	reshaping of $(V')^\top G$ where $V'=(V_1,V_2,\ldots,V_d)$. 
	Since $V'\in \mathbb{R}^{d^2\times nd}$, 
	$G\in \mathbb{R}^{d^2 \times \ell}$, and 
	$\ell=\tilde{O}(d^2)$, 
	this matrix multiplication takes time at most
	$O\Paren{d^{2\Paren{\omega\Paren{\frac{1}{2}(1+\log_{d}n)}}}}$.
	By the same reasoning, it takes time at most
	$O\Paren{d^{2\Paren{\omega\Paren{\frac{1}{2}(1+\log_{d}n)}}}}$
	to obtain $Z_i$ for all $i\in [\ell]$. 
	
	In conclusion, the running time of is bounded by 
	$O\Paren{n\cdot d^4+d^{2\Paren{\omega\Paren{\frac{1}{2}(1+\log_{d}n)}}}}$.
	Since $\omega(5/4)\leq 2.622$, this is bounded by 
	$O(n\cdot d^4+d^{5.25})$.
\end{proof}

The second lemma is about computing singular value decomposition for
rectangular flattening of a low rank order-$6$ tensor. The proof
has already appeared in the proof of lemma 8 in~\cite{HopkinsSS19}.
\begin{lemma}[Time for computing singular value decomposition]\label{lem:time-svd-tensor}
	Suppose we are given matrices $U\in \mathbb{R}^{d^3\times n}$
	and $Z\in \mathbb{R}^{nd \times d^2}$. 
	Then for matrix $M\coloneqq Z^\top (UU^\top \otimes \Id_d)Z$ and $k=O(n)$,
	there is a $\tilde{O}(n^2d^3\delta^{-1})$ time algorithm obtaining 
	$P\in \R^{d^3\times k}$ and diagonal matrix 
	$\Lambda\in \R^{k\times k}$ such that
	\begin{equation*}
		\Norm{M^{1/2}-P\Lambda^{1/2}P^\top}\leq (1+\delta) \rho_k
	\end{equation*}
	where $\rho_k$ is the $k$-th largest eigenvalue of $M^{1/2}$.
\end{lemma}
\begin{proof}
	We first claim that matrix-vector multiplication by $M$ can be 
	implemented in $O\left(n d^{3}\right)$ time, 
	with $O\left(n^{2} d^{3}\right)$ preprocessing time for computing 
	the product $U^{\top} U$. The matrix-vector multiplications by $Z$ 
	and $Z^{\top}$ take time $O\left(n d^{3}\right)$, and then multiplying
	$Z y$ by $U^{\top} U \otimes \operatorname{Id}_{d}$ is 
	reshaping-equivalent to multiplying $U^{\top} U$ into the $n \times d$ 
	matrix reshaping of $Z y$, which takes
	$O\left(n^{2} d\right)$ time with the precomputed $n \times n$ matrix $U^{\top} U$. 
	Therefore, by \cref{lem:LazySVD}, it
	takes time $\tilde{O}\left(n^{2} d^{3} \delta^{-1 / 2}\right)$ 
	to yield a rank-$k$ eigendecomposition 
	$P \Lambda P^{\top}$ such that 
	$\left\|M^{1 / 2}-P \Lambda^{1 / 2} P^{\top}\right\| 
	\leqslant(1+\delta) \rho_{k}$
\end{proof}

\section{Missing proofs}
\label{sec:missing-proofs}

In this section we will give the proofs we omitted in the main body of the paper.

\subsection{Reducing to isotropic components}
\label{sec:isotropic-components}
In this section, we prove that the components $a_i^{\otimes 2}$ 
are nearly isotropic
 in the sense of Frobenius norm. Concretely we prove the following theorem. 
\begin{lemma}[Restatement of \cref{lem:equivalence-isotropic}]
    For $n=O\Paren{d^{3/2}/\polylog(d)}$ and $n'\leq n$, let 
    $a_1,a_2,\ldots,a_{n'}\in \mathbb{R}^d$ be $(n,d)$-nicely-separated.
    Let $R=\sqrt{2}\cdot \Paren{\E_{a\sim N(0,\Id_d)} a^{\otimes 2}\Paren{a^{\otimes 2}}^\top}^{+1/2}$, 
    for any tensor $\hat{\mathbf M}=\sum_{i=1}^{n} a_i^{\otimes 6}+\mathbf E$ with 
    $\norm{\mathbf E}_F\leq  \tilde{O}\Paren{\frac{n'}{d^{3/2}}}\cdot \sqrt{n}$, we have 
    $$\Norm{\hat{\mathbf M}-\sum_{i=1}^{n'}\Paren{Ra_i^{\otimes 2}}^{\otimes 3}}_F\leq \tilde{O}\Paren{\frac{n'}{d^{3/2}}}\cdot \sqrt{n}\,.$$ 
\end{lemma}

This will allow us to rewrite 
$\hat{\mathbf M}=\sum_{i=1}^{n'}\Paren{Ra_i^{\otimes 2}}^{\otimes 3}+{\mathbf E'}$
where $\norm{\mathbf E'}\leq \frac{1}{\polylog(d)}$. The advantage is that the 
component vectors
 $Ra_i^{\otimes 2}$ now become isotropic, and the spectral norm of 
 $\sum_{i=1}^{n'}Ra_i^{\otimes 2}\Paren{Ra_i^{\otimes 2}}^\top$ is tightly bounded.

The lemma follows as a corollary of the statement below:
\begin{lemma}\label{lem:equivalence-near-orthonormal}
    For $n=O\Paren{d^{3/2}/\polylog(d)}$ and $d\leq n'\leq n$, let 
    $a_1,a_2,\ldots,a_{n'}\in \mathbb{R}^d$ be $(n,d)$-nicely-separated.
    Let $R=\sqrt{2}\cdot 
    \Paren{\E_{a\sim N(0,\Id_d)} a^{\otimes 2}
    \Paren{a^{\otimes 2}}^\top}^{+1/2}$.
    Let vectors
    $b_i\coloneqq Ra_i^{\otimes 2}$ for $i\in [n']$.
    Then we have
    \begin{equation*}
        \Norm{\sum_{i=1}^{n'} b_i^{\otimes 3}-a_i^{\otimes 6}}_F\leq 
        10\delta\sqrt{n'}
    \end{equation*}
    where $\delta=\tilde{O}\Paren{\frac{n}{d^{3/2}}}$.
\end{lemma}
\begin{proof}
 We decompose the square of Frobenius norm into the sum of two parts:
 \begin{align*}
    \Norm{\sum_{i=1}^{n'} b_i^{\otimes 3}-\sum_{i=1}^{n'} a_i^{\otimes 6}}^2_F&=
     \sum_{i,j\in [n']} \Iprod{b_i^{\otimes 3}-a_i^{\otimes 6},b_j^{\otimes 3}-a_j^{\otimes 6}}\\
     &= \sum_{i\in [n']} \Norm{b_i^{\otimes 3}-a_i^{\otimes 6}}^2+
     \sum_{\substack{i,j\in [n']\\ i\neq j}} 
     \Iprod{b_i^{\otimes 3}-a_i^{\otimes 6},b_j^{\otimes 3}-a_j^{\otimes 6}}     
 \end{align*}
 For the first part, by $(n,d)$ nicely-separated assumption(\cref{def:nicely-separated}), we have $\norm{b_i-a_i^{\otimes 2}}^2\leq \tilde{O}(1/d)$ 
 and thus $\iprod{b_i,a_i^{\otimes 2}}\leq 1-\tilde{O}(1/d)$. 
 It follows that $\iprod{b_i,a_i^{\otimes 2}}^3\geq 1-\tilde{O}(1/d)$ and
  $\Norm{b_i^{\otimes 3}-a_i^{\otimes 6}}^2\leq \tilde{O}(1/d)$. 
  By summation, we have 
  $$\sum_{i=1}^{n'} \Norm{b_i^{\otimes 3}- a_i^{\otimes 6}}^2\leq \tilde{O}(1/d)\cdot n
  \leq \delta^2 n$$

  For the second part, we have
  \begin{align*}
    \sum_{\substack{i,j\in [n']\\ i\neq j}} 
    \Iprod{b_i^{\otimes 3}-a_i^{\otimes 6},b_j^{\otimes 3}-a_j^{\otimes 6}} 
    & = \sum_{\substack{i,j\in [n']\\ i\neq j}} \iprod{b_i,b_j}^3
    -2\sum_{\substack{i,j\in [n']\\ i\neq j}}\iprod{a_i^{\otimes 2},b_j}^3
    +\sum_{\substack{i,j\in [n']\\ i\neq j}} \iprod{a_i,a_j}^6\\
  \end{align*}
  For the first term, by assumption, for each $j\in [n]$
  \begin{equation*}
    \Abs{\sum_{i\in [n']\setminus \{j\}}\iprod{b_i,b_j}^3}\leq
    \sum_{i\in [n']\setminus \{j\}}\Abs{\iprod{b_i,b_j}}\iprod{b_i,b_j}^2\leq 
    \tilde{O}\Paren{\frac{n}{d^2}}
  \end{equation*}
  thus we have 
  $$\Abs{\sum_{\substack{i,j\in [n']\\ i\neq j}} \iprod{b_i,b_j}^3}\leq 
  (1+o(1)) \sum_{\substack{i,j\in [n']\\ i\neq j}} \iprod{b_i,b_j}^2
  \leq n'\cdot \tilde{O}\Paren{\frac{n}{d^2}}\leq \delta^2\cdot n'$$
  For the second term, denote $c_i=a_i^{\otimes 2}-b_i$, 
  using the $(n,d)$-nicely-separated assumption that 
  $\norm{c_i}^2\leq \tilde{O}(1/d)$ and 
  $\Norm{\sum_{j\in [n']}b_jb_j^\top}\leq 1+o(1)$,
  we have
  \begin{align*}
    \Abs{\sum_{\substack{i,j\in [n']\\ i\neq j}}\iprod{a_i^{\otimes 2},b_j}^3}
    &\leq (1+o(1))\sum_{\substack{i,j\in [n']\\ i\neq j}}\iprod{a_i^{\otimes 2},b_j}^2\\
    &= (1+o(1))\sum_{\substack{i,j\in [n']\\ i\neq j}} \iprod{c_i+b_i,b_j}^2\\
    & \leq 2(1+o(1))\sum_{\substack{i,j\in [n']\\ i\neq j}}\iprod{c_i,b_j}^2+
    2(1+o(1))\sum_{\substack{i,j\in [n']\\ i\neq j}}\iprod{b_i,b_j}^2\\
    & \leq 2(1+o(1))\sum_{i\in [n']} c_i^\top \Paren{\sum_{j\in [n]\setminus\{i\}} b_jb_j^\top}
     c_i + 2(1+o(1)) n'\cdot \tilde{O}(n/d^2)\\
     & \leq 2(1+o(1))\sum_{i\in [n']} \norm{c_i}^2 + \tilde{O}(n n'/d^2)\\
    & \leq \tilde{O}\Paren{nn'/d^2+n/\sqrt{d}}\\
    & \leq o(\delta^2\cdot n')
  \end{align*}

  For the third term, by the $(n,d)$-nicely-separated property, we have 
  $\iprod{a_i,a_j}^6\leq \frac{1}{d^{3/2}}$. Therefore we have
  \begin{equation*}
    \sum_{\substack{i,j\in [n']\\ i\neq j}} \iprod{a_i,a_j}^6 
    \leq \frac{1}{d^{3/2}}\cdot n^{\prime 2}\leq \delta^2 n'
  \end{equation*}

  Therefore in all, we have 
  $\Norm{\sum_{i=1}^{n'} b_i^{\otimes 3}-\sum_{i=1}^{n'} a_i^{\otimes 6}}^2_F\leq 2\delta^2 n'$
  and thus we have the claim. 
\end{proof}

\begin{proof}[Proof of \cref{lem:equivalence-isotropic}]
By \cref{lem:equivalence-near-orthonormal}, we have 
$\Norm{\sum_{i=1}^{n'} b_i^{\otimes 3}-\sum_{i=1}^{n'} a_i^{\otimes 6}}_F\leq \tilde{O}\Paren{\frac{n}{d^{3/2}}} \sqrt{n'}$.
Since $\Norm{M-\sum_{i=1}^n a_i^{\otimes 6}}_F\leq \tilde{O}\Paren{\frac{n}{d^{3/2}}} \sqrt{n'}$,
 by triangle inequality, we have the claim.
\end{proof}

 \subsection{Satisfaction of nice-separation property by independent random vectors}
 \label{sec:proof-nicely-separated}
 In this section, we prove \cref{lem:component-assumptions} using the concentration 
  results from \cref{section:concentration-bounds}.

  \begin{proof}[Proof of \cref{lem:component-assumptions}]
      Property (i),(ii),(iii) follows from lemma~\cref{lem:concentration-input-vectors}.
        Property (iv) follows from the lemma~\cref{lem:isotropic-4order-spectral-norm}.
    Property (5),(6) follows from the lemma~\cref{lem:isotropic-norm}.
    Property (7),(8) follows from the lemma~\cref{lem:Gaussian-vector-norm}.
  \end{proof}

\subsection{Gaussian rounding}

\subsubsection{Spectral gap from random contraction}
\label{sec:random-contraction-missing-proof}
In this section, we will prove the spectral gap of diagonal terms.
\begin{lemma}[Restatement of~\cref{lem:diagonal-spectral-gap}]
    Consider the setting of \cref{lem:gaussian-rounding-correctness}.
        Let $R=\sqrt{2}\cdot \Paren{\E_{a\sim \Id_d}\Paren{aa^\top}^{\otimes 2}}^{+1/2}$. 
        Let $S_0 \subseteq [n]$ be of size $n'$ where $d \leq n' \leq n$ and assume that the set $\Set{a_i \suchthat i \in S_0}$ is $(n,d)$-nicely separated.
        Further, let $\hat{\mathbf M}$ be such that
        \begin{align*}
        \normf{\mathbf M^{\le 1} - \sum_{i\in S_0} (R a_i^{\otimes 2})^{\otimes 3}} \leq \e \sqrt{n'} \quad \text{ and } \quad \Norm{ M^{\le 1}_{\Set{1,2,3,4}\Set{5,6}}}, \Norm{ M^{\le 1}_{\Set{1,2,5,6}\Set{3,4}}}\leq 1.
        \end{align*}
        Consider the matrix $M_g = \Paren{g\otimes \Id_{d^2}\otimes \Id_{d^2}} \flattent{\mathbf M^{\le 1}}{1,2}{3,4}{5,6}$ in~\cref{alg:gaussian-rounding}.
         Then for every $\alpha\geq 1+10\log \log n/\log n$,
         there exists a subset 
           $S\subseteq S_0$ of size $m \geq 0.99n'$, 
           such that for each $i\in S$, and
           $v=Ra_i^{\otimes 2}$, with probability at least
           $1/d^{2\alpha}$ over $g$,
           we have $M=cvv^\top+N$ where  
            \begin{itemize}
                \item 
                $\norm{cvv^\top}\geq (1+\frac{1}{\log d})\norm{N}$
                \item $\norm{Nv},\norm{vN}\leq \epsilon c\norm{v}^2$ 
            \end{itemize}
\end{lemma}

The proof of this lemma involves a simple fact from standard Gaussian tail bound:
\begin{lemma}\label{lem:Gaussian-correlation-concentration}
Given any unit norm vector $v\in \mathbb{R}^{d^2}$, for standard random Gaussian
 vector $g\sim N(0,\Id_{d^2})$, we have
$$\Pr\Set{\Abs{\Iprod{g,v}}\geq \sqrt{2\alpha\log n}}=\tilde{\Theta}(n^{-\alpha})$$
\end{lemma}
\begin{proof}
    Since the distribution of $\Iprod{g,v}$ is given by $N(0,1)$. 
    By taking $t=\sqrt{2\alpha\log n}$ in the fact~\ref{lem:Gaussian-tail}, 
    we have the claim.
    
\end{proof}

We will also use the following simple fact(a similar fact
appears in~\cite{SchrammS17}):
\begin{fact}\label{fact:Frobenius-piegon-hole}
If $P_{1}, \ldots, P_{n}\in \mathbb{R}^{d^2\times s}$ s.t 
$$\Norm{\sum_{i=1}^n P_{i}P_i^\top}
\leq 1+o(1)$$ and $E\in \mathbb{R}^{d^2\times d^2}$ s.t 
$\Norm{E}_F\leq \epsilon\sqrt{n}$
, then for a $1-\delta$ fraction of $i \in [n]$
$$
\left\|P_i^\top E\right\|_{F} \leqslant \epsilon / \delta
$$ 
\end{fact}
\begin{proof}
This follows from the fact that
\begin{align*}
    \sum_{i=1}^n\left\|P_i^\top E\right\|_{F}^{2} &=
    \sum_{i=1}^n \Iprod{E,P_iP_i^\top E}\\
    &=\Iprod{E,\Paren{\sum_{i=1}^n P_i P_i^\top}E}\\
    &\leq \Norm{E}_F \Norm{\Paren{\sum_{i=1}^n P_i P_i^\top}E}_F\\
    &\leq \Norm{E}_F^2 \Norm{\sum_{i=1}^n P_i P_i^\top}\\
    &\leq \Norm{E}_F^2
\end{align*}
\end{proof}

Now we prove the \cref{lem:diagonal-spectral-gap}:
\begin{proof}[Proof of \cref{lem:diagonal-spectral-gap}]
For notation simplicity, for $j\in [d^2]$ 
 we denote matrices $T_j$ as the $j$-th slice in the first mode
 of $\flattent{\mathbf M^{\le 1}}{1,2}{3,4}{5,6}\in \mathbb{R}^{d^2\times d^2\times d^2}$. 
Further we denote $\mathbf X=\sum_{i\in S_0} (R a_i^{\otimes 2})^{\otimes 3}$, $\mathbf{E}=\flattent{\mathbf M^{\le 1}}{1,2}{3,4}{5,6}-\mathbf{X}$
 and $E_j$ as the $j$-th slice in the first mode.

W.l.o.g. assume that $S_0 = [n']$.
For each $i\in [n']$, we denote $b_i=Ra_i^{\otimes 2}$. 
We first prove that for each $i\in [n']$ s.t 
$\Norm{\Paren{b_ib_i^\top\otimes \Id_{d^2}}E}_F\leq 100\epsilon$ and
 $\left\| E\left(b_{i} \otimes \mathrm{Id}_{d^{3}}\right)\right\|_{F}
  \leq 100\epsilon$, 
we have $i\in S$. By \cref{lem:isotropic-4order-spectral-norm}, we have
 $\Norm{\sum_{i=1}^{n'} b_ib_i^\top}\leq 1+\tilde{O}
\Paren{\frac{n}{d^{3/2}}}$, and 
$$\Norm{\sum_{i=1}^{n'} (b_ib_i^\top)(b_ib_i^{\top})^\top}=\Norm{
\sum_{i=1}^{n'} \norm{b_i}^2 b_ib_i^\top} \leq 1+\tilde{O}
\Paren{\frac{n}{d^{3/2}}}$$.
Thus by \cref{fact:Frobenius-piegon-hole}, 
the assumptions $\Norm{\Paren{b_ib_i^\top\otimes \Id_{d^2}}E}_F\leq 100\epsilon$ and
$\left\| E\left(b_{i} \otimes \mathrm{Id}_{d^{3}}\right)\right\|_{F}
 \leqslant 100\epsilon$ are satisfied for 
at least $0.99n'$ of the 
component vectors. The lemma thus follows. 

Without loss of generality, we suppose
$\Norm{\Paren{b_i\otimes \Id_{d^2}}E}_F\leq 100\epsilon$ and 
$\Norm{\Paren{b_ib_i^\top\otimes \Id_{d^2}}E}_F\leq 100\epsilon$.
We denote $g^{\parallel}=\frac{1}{\norm{b_i}^2}\Iprod{g,b_i}b_i$ and $g^{\perp}=g-g^{\parallel}$
.Then by the property of Gaussian distribution, 
 $g^{\parallel},g^{\perp}$ 
are independent. Then we have
\begin{equation*}
    M=\iprod{g,b_1}b_1b_1^\top+\sum_{j=1}^{d^2} 
    \Paren{g^{\parallel}+g^{\perp}}_j\cdot \Paren{\mathbf X-b_1^{\otimes 3}+E}_j
    = \iprod{g,b_1}b_1b_1^\top+N
\end{equation*}
where $N=\sum_{j=1}^{d^2} 
\Paren{g^{\parallel}}_j\cdot \Paren{\mathbf X-b_1^{\otimes 3}+\mathbf E}_j
+\sum_{j=1}^{d^2}g^{\perp}_j\cdot T_j$.

First by \cref{lem:Gaussian-correlation-concentration}, with probability
 at least $\Theta(d^{-2\alpha})$, 
 $\iprod{g,b_1}=\norm{b_1}\norm{g^{\parallel}}
 \geq \sqrt{4\alpha\log d}$.
 We denote this event as $\mathcal{G}_{1}(\alpha)$. 
 On the other hand, we denote
\begin{equation*}
    \mathcal{E}_{>1}(\rho) \stackrel{\text { def }}{=}
    \left\{\left\|\sum_{j=1}^{d^{2}} g_{j}^{\perp} \cdot T_{j}\right\| 
    \leqslant \sqrt{4(1+\rho) \log d}\right\}
\end{equation*}
By \cref{lem:random-tensor-contraction} and the independence between 
$g^{\perp}$ and $g^{\parallel}$
, we have
\begin{equation*}
    \Pr_g\Brac{\mathcal{E}_{>1}(\rho) 
    \Mid \mathcal{G}_{1}(\alpha)}\ge 1-d^{-\rho}
\end{equation*}

Next we bound $g\cdot\Paren{\mathbf X-b_1^{\otimes 3}}$
 and $\sum_{j} g_{j}^{\parallel} E_{j}$ separately. 
 For the first one, by the nicely-separated assumption, we have
 $\max_{i\geq 2} \abs{\iprod{b_j,b_1}}\leq \sqrt{n}/d$
  and $\norm{\sum_{i=2}^n  b_jb_j^\top}\leq 2$. It follows that
 \begin{align*}
    \Norm{\sum_{j} g_{j}^{\parallel} \Paren{\mathbf X-b_1^{\otimes 3}}}
     =\frac{\norm{g^{\parallel}}}{\norm{b_1}}\cdot \Norm{\sum_{i=2}^n 
     \iprod{b_j,b_1} b_jb_j^\top}
     \leq 2\norm{g^{\parallel}}\cdot \max_{i\geq 2} \abs{\iprod{b_j,b_1}}
     \leq \tilde{O}\Paren{\frac{\sqrt{n}}{d})}\norm{g^{\parallel}}
 \end{align*}
 For the second one we have
\begin{equation*}
    \sum_{j} g_{j}^{\parallel} E_{j}=\frac{\left\langle g, b_1
    \right\rangle}{\norm{b_1}} \cdot \sum_{j} b_1(j) \cdot 
    E_{j}=\frac{\left\langle g, b_1
    \right\rangle}{\norm{b_1}} \cdot
    \left(b_{1} \otimes \operatorname{Id}_{d^{2}}\right) E
\end{equation*}
Since by assumption, we have 
$\Norm{\left(b_1\otimes \operatorname{Id}_{d^{2}}\right) E}_F\leq 
100\epsilon$. We have
\begin{equation*}
    \Norm{\sum_{j} g_{j}^{\parallel} E_{j}}\leq 100\epsilon\norm{g^{\parallel}} 
\end{equation*}
Combining both parts, 
the event $\mathcal{G}_{1}(\alpha)$ and $\mathcal{E}_{>1}(\rho)$ 
implies 
\begin{equation*}
    \norm{N}\leq \Norm{\sum_{j} g_j^{\parallel} E_{j}}+
    \Norm{\sum_j g_j^{\perp}T_j}+\Norm{\sum_{j} g_{j}^{\parallel} \cdot\Paren{\mathbf X-b_1^{\otimes 3}}}
    \leq \Paren{100\epsilon+\sqrt{\frac{1+\rho}{\alpha}}}
     \norm{g^{\parallel}}
\end{equation*}

Finally, we consider the event
\begin{equation*}
\mathcal{E}_{b_{1}, E}(\theta) \stackrel{\text { def }}{=}
\left\{\left\|\left(\sum_{j=1}^{d^{2}} g_{j}^{\perp} 
\cdot T_{j}\right) b_{1}\right\|_{2},\left\|
\left(\sum_{j=1}^{d^{2}} g_{j}^{\perp} \cdot T_{j}\right)^{\top} b_{1}\right\|_{2}
 \leqslant 100\epsilon \cdot \sqrt{2(1+\theta)}\right\}  
\end{equation*}

First we consider the following decomposition
\begin{equation*}
    \left(\sum_{j=1}^{d^{2}} g_{j}^{\perp} \cdot T_{j}\right) b_{1}
    =\sum_{j=1}^{d^{2}} g_{j}^{\perp} \cdot\left(\mathbf X-b_{1}^{\otimes 3}\right)_{j}
     b_{1}+\sum_{j=1}^{d^{2}} g_{j}^{\perp} \cdot E_{j} b_{1}
\end{equation*}

For the first term, let $X^\perp=\sum_{i=1}^{n'} b_i b_i^\top$ and
$X_g^\perp=\sum_{i=1}^{n'} 
\iprod{b_i,g^{\perp}}b_i b_i^\top$. Then since
 $X_g^\perp\preceq \Paren{\max_{1\leq i\leq n} \Abs{\iprod{b_i,g^{\perp}}}}\cdot X^\perp$, we have
\begin{align*}
    \Norm{\sum_{j=1}^{d^{2}} g_{j}^{\perp} \cdot\left(\mathbf X-b_{1}^{\otimes 3}\right)_{j}
     b_{1}}^2 &= \Norm{\sum_{i=1}^{n'} \iprod{b_i,g^{\perp}}b_i b_i^\top b_1}^2  \\                 
    &=\iprod{b_1,(X_g^\perp)^2 b_1}\\
    &\leq \Paren{\max_{1\leq i\leq n} \Abs{\iprod{b_i,g^{\perp}}}^2}
    \iprod{b_1,(X^\perp)^2 b_1}
\end{align*}

By \cref{lem:bound-noise-signal-iprod} and 
the $(n,d)$-nicely-separated property, 
with probability $1-o(1)$ 
we have $\norm{(X^\perp)^2-X^\perp}\leq \tilde{O}\Paren{\frac{n}{d^{3/2}}}$.
It then follows that
\begin{align*}
    \iprod{b_1,(X^\perp)^2 b_1}&\leq 
    \iprod{b_1,(X^\perp) b_1}+\tilde{O}\Paren{\frac{n}{d^{3/2}}}\\
    &\leq \sum_{i=2}^{n'} \iprod{b_i,b_1}^2+\tilde{O}\Paren{\frac{n}{d^{3/2}}}\\
    &\leq \tilde{O}\Paren{\frac{n}{d^2}}+\tilde{O}\Paren{\frac{n}{d^{3/2}}}\\
    &\leq \tilde{O}\Paren{\frac{n}{d^{3/2}}}
\end{align*}
The last step follows from the fact that 
$\sum_{i=2}^{n'} \iprod{b_i,b_1}^2\geq \tilde{O}\Paren{\frac{n}{d^{3/2}}}$ 
.Since with probability $1-n^{-\Omega(\log n)}$ over $g$,
$$\max_{1\leq i\leq n'} \Abs{\iprod{b_i,g^{\perp}}}\le \sqrt{2\log^2 n}$$
we have 
\begin{equation*}
    \Norm{\sum_{j=1}^{d^{2}} g_{j}^{\perp} \cdot\left(\mathbf X-b_{1}^{\otimes 3}\right)_{j}
     b_{1}}\leq \tilde{O}\Paren{\frac{n}{d^{3/2}}}
\end{equation*}

For the second term, by assumption we have
\begin{align*}
    \mathbb{P}\left[\left\|\sum_{j=1}^{d^{2}} g_{j}^{\perp} \cdot 
    E_{j} a\right\|_{2} \leqslant 100\epsilon 
    \sqrt{2(1+\theta)}\right] \geqslant 1-d^{-\theta}
\end{align*}
It follows that 
\begin{equation*}
    \mathbb{P}\left[\mathcal{E}_{a_{1}, E}(\theta)\right] \geqslant 1-2 d^{-\theta}
\end{equation*}

Now since 
\begin{equation*}
    \Pr_g\Brac{\mathcal{E}_{>1}(\rho)\cap \mathcal{E}_{b_{1}, E}(\theta)
    \Mid \mathcal{G}_{1}(\alpha)}\ge 1-d^{-\rho}-2d^{-\theta}-n^{-\Omega(\log n)}
\end{equation*}
by the independence between $\mathcal{E}_{b_{1}, E}(\theta)$ 
and $\mathcal{G}_{1}(\alpha)$, we have
 \begin{equation*}
    \Pr_g\Brac{\mathcal{E}_{>1}(\rho)\cap \mathcal{E}_{b_{1}, E}(\theta)
    \cap\mathcal{G}_{1}(\alpha)}=\Pr_g\Brac{\mathcal{E}_{>1}(\rho)\cup \mathcal{E}_{a_{1}, E}(\theta)
    \Mid \mathcal{G}_{1}(\alpha)}\Pr_g\Brac{\mathcal{G}_{1}(\alpha)}\geq
    (1-d^{-\rho}-2d^{-\theta}) \Theta(n^{-\alpha})
 \end{equation*}

Now we write $M_g=cb_1b_1^\top+N$. By setting $\rho,\theta=\frac{\log\log n}{\log n}$,
 and $\alpha = (1+2\rho)\geqslant (1+\frac{1}{\log n})^2(1+\rho)$, we have
all three conditions are satisfied when 
$\mathcal{E}_{>1}(\rho)\cap \mathcal{E}_{b_{1}, E}(\theta)
\cap\mathcal{G}_{1}(\alpha)$ holds. 
Indeed, by event $\mathcal{G}_{1}(\alpha)$ and $\mathcal{E}_{>1}(\rho)$, 
we have
$c=\norm{b_1}\norm{g^{\parallel}}\geq\sqrt{4\alpha\log d}$
and $\norm{N}\leq \Paren{100\epsilon+\sqrt{\frac{1+\rho}{\alpha}}}
\norm{g^{\parallel}}\leq (1+\frac{1}{\log n})c$; 
By event $\mathcal{E}_{b_{1}, E}(\theta)$,
 $\norm{Nb_1},\norm{N^\top b_1}\leq 100\epsilon \cdot 
 \sqrt{2(1+\theta)}+\Norm{\sum_{j} g_j^{\parallel} E_{j}}+
\Norm{\sum_{j} g_{j}^{\parallel} \cdot\Paren{\mathbf X-b_1^{\otimes 3}}}
\leq \frac{c}{\polylog(d)}$. 
\end{proof}

\subsubsection{Recovering constant fraction of components}
\label{sec:partial-recovery-missing-proof}

\begin{lemma}[Restatement of~\cref{lem:top-singular-vector}]
Consider the setting of \cref{lem:gaussian-rounding-correctness}.
Let $S_0 \subseteq [n]$ be of size $n' \leq n$ and assume that the set $\Set{a_i \suchthat i \in S_0}$ is $(n,d)$-nicely separated.
Consider the matrix $M_g$ and its top right singular vector $u_r \in \R^{d^2}$ obtained in one iteration of~\cref{alg:gaussian-rounding}.
Then there exists a set $S \subseteq S_0$, such that for each $i \in S$, 
it holds with probability $\tilde{\Theta}(d^{-2})$ that 
\begin{itemize}
    \item $\iprod{u, Ra_i^{\otimes 2}} \geq 1 - \frac{1}{\polylog (d)}$.
    \item the ratio between largest and second largest singular values of 
    $M_g$ is larger than $1+\frac{1}{\polylog (d)}$
\end{itemize}
\end{lemma}

To prove the lemma above we will use
 a lemma on getting estimation vector from the spectral gap, 
 which already appears in the previous literature:
 \begin{lemma}[Lemma 4.7 in \cite{SchrammS17}]\label{lem:recovery-single-spike}
Let $M_g$ be a $\mathbb{R}^{n\times n}$ symmetric matrix s.t 
$M_g=cvv^\top+N$ where $v$ has unit norm, $c\geq (1+\delta)\norm{N}$, and 
$\norm{Nv},\norm{vN}\leq \gamma\norm{v}^2$. Suppose $\gamma\leq 
\frac{\delta}{\polylog(d)}$,
 then the top eigenvector of $M_g$ denoted by $u$ satisfies:
 \begin{equation*}
     \Iprod{u,v}^2\geq 1-\frac{1}{\polylog(d)}
 \end{equation*} 
 Further the ratio between largest and second largest singular values of 
 $M_g$ is larger than $1+O\Paren{\frac{1}{\polylog(d)}}$.
 \end{lemma} 
 
  \begin{proof}[Proof of~\cref{lem:top-singular-vector}]
  W.l.o.g assume that $S_0 = [n']$.
For $i\in [n']$, we denote $b_i\coloneqq Ra_i^{\otimes 2}$. 
Combining \cref{lem:recovery-single-spike}, \cref{lem:diagonal-spectral-gap}, 
for some $S\subseteq [n']$ with size at least $0.99n'$,
for each $i\in S$,
with probability $\tilde{\Theta}(d^{-2})$ over $g$,
 we have $M_g=cb_ib_i^\top+N$, where 
 $\norm{b_i}=1\pm \tilde{O}\Paren{\frac{1}{\sqrt{d}}}$
  and $|c|\geq (1+\frac{1}{\log(d)})\norm{N}$
  , and $\norm{Nv},\norm{vN}\leq \frac{1}{\polylog(d)}$.

Now by 
\cref{lem:recovery-single-spike}, there exists unit norm vector 
$u\in \{u_L,u_R\}$ s.t $\iprod{u,Ra_i^{\otimes 2}}^2 \geq 1-\frac{1}{\polylog(d)}$.
 Since $\Norm{Ra_i^{\otimes 2}-a_i^{\otimes 2}}\leq \tilde{O}\Paren{\frac{1}{\sqrt{d}}}$
 , it follows that $\Abs{\iprod{u,Ra_i^{\otimes 2}}}\geq 1-\frac{1}{\polylog(d)}$.
 \end{proof}

 \begin{lemma}[Restatement of~\cref{lem:extracting-component-from-square}]
    Consider the setting of \cref{lem:gaussian-rounding-correctness}. 
    Suppose for some unit norm vector $a\in \mathbb{R}^d$,
      and unit vector $u\in \mathbf{R}^{d^2}$, $\iprod{u,Ra^{\otimes 2}}\geq
       1-\frac{1}{\polylog(d)}$. Then flattening $u$ into a $d\times d$
    matrix $U$, the top left or right singular vector of $U$ denoted by $v$ will
     satisfy $\iprod{a,v}^2\geq 1-\frac{1}{\polylog(d)}$. 
\end{lemma}
\begin{proof}
    Since $\norm{Ra_i^{\otimes 2}-a_i^{\otimes 2}}^2
    \leq \tilde{O}(\frac{1}{\sqrt{d}})$, we have $\iprod{u,a^{\otimes 2}}\geq
    1-\frac{1}{\polylog(d)}$.
    Let the singular value decomposition of $U$ be 
    $U=\sum_{i=1}^d \sigma_i w_iv_i^\top$, where $\sigma_i$ are singular vectors. 
    Then by the best rank-$1$ approximation property of $\sigma_1w_1v_1^\top$, 
    we have
     $\norm{\sigma_1w_1v_1^\top-U}_F\leq \norm{aa^\top-U}_F$. By triangle inequality,
    we have $\normf{\sigma_1w_1v_1^\top-aa^\top}\leq 2\normf{aa^\top-U}$. 
    Since $\iprod{u,a^{\otimes 2}}\geq
    1-\frac{1}{\polylog(n)}$, we have $\normf{U-aa^\top}\leq \frac{1}{\polylog(d)}$.
    It follows that $\normf{\sigma_1 w_1v_1^\top-aa^\top}\leq \frac{1}{\polylog(d)}$.
    Since $\sigma_1\leq 1$, we have
    $2\sigma_1\iprod{w_1,a}\iprod{v_1,a}\geq 1+\sigma_1^2-\frac{1}{\polylog(d)}$,
     which implies that $\iprod{w_1,a}\iprod{v_1,a}\geq 1-\frac{1}{\polylog(d)}$. 
     Now since $w_1,v_1,a$ has unit norm, we have
     $\iprod{w_1,a}^2,\iprod{v_1,a}^2\geq 1-\frac{1}{\polylog(d)}$.
\end{proof}
  
\begin{lemma}[Restatement of~\cref{lem:recover-one-component}]
Let $S_0 \subseteq [n]$ be of size $d\leq n'\leq n$ and assume that the set $\Set{a_i \suchthat i \in S_0}$ is $(n,d)$-nicely separated.
Consider $v_l$ and $v_r$ in~\cref{alg:gaussian-rounding}, then there exists a set $S \subseteq S_0$ of size $m \geq 0.99n'$ such that for each $i \in S$ it holds with probability $\tilde{\Theta}(d^{-2})$ that $\max_{v \in \{\pm v_l, \pm v_r\}} \iprod{v, a_i}\geq 1-\frac{1}{\polylog(d)}$.
\end{lemma}
\begin{proof}
    Combining \cref{lem:top-singular-vector} and 
    \cref{lem:extracting-component-from-square}, we have the claim.
\end{proof}

\subsubsection{Pruning list of components}
\label{sec:pruning}

\begin{lemma}[Restatement of~\cref{lem:pruning-components}]
Let $S$ be the set of vector computed in~\hyperref[alg:gaussian-rounding]{Step 3} of~\cref{alg:gaussian-rounding} and $S'$ be the subset of components of~\cref{lem:recover-one-component} then for each $b \in S$ there exists a unique $i \in S'$ such that $\iprod{b, a_i} \geq 1- \frac{1}{\polylog d}$.
\end{lemma}

\noindent In order to prove this we will use the following two facts.

 \begin{fact}\label{lem:vector-recovering-same-component}
     Let $a,b_1,b_2\in\mathbb{R}^d$ be unit norm vectors.
     If $\iprod{a,b_1}\geq 1-\delta$ and 
      $\iprod{a,b_2}\geq 1-\delta$, then 
      $\iprod{b_1,b_2}\geq 1-2\delta$
 \end{fact}
\begin{proof}
    Since we have $\norm{a-b_1}^2=2-2\iprod{a,b_1}\leq 2\delta$ and same for
    $\norm{a-b_1}$, it follows that
    \begin{align*}
        \iprod{b_1,b_2}&=\iprod{a,a}+\iprod{b_1,b_2-a}+\iprod{b_1-a,b_2}\\
        & \geq 1-2\sqrt{2\delta}
    \end{align*}
\end{proof}

 \begin{fact}\label{lem:vector-recovering-different-components}
     Let $a_1,a_2,b_1,b_2\in\mathbb{R}^d$ be unit norm vector such that $\iprod{a_1,b_1}\geq 1-\delta_1$, $\iprod{a_2,b_2}\geq 1-\delta_1$, and $\Abs{\iprod{a_1,a_2}}\leq \delta_2$.
     Then $\iprod{b_1,b_2}\leq \frac{\delta_2+8\delta_1}{2}$.
 \end{fact}
 \begin{proof}
    Since $\iprod{a_1,b_1}=\frac{2-\norm{a_1-b_1}^2}{2}$ and 
    $\iprod{a_2,b_2}=\frac{2-\norm{a_2-b_2}^2}{2}$, we have $\norm{a_1-b_1}\leq \sqrt{2\delta_1}$
     and $\norm{a_2-b_2}\leq \sqrt{2\delta_1}$. 
     For the same reason, $\norm{a_1-a_2}^2=2-2\iprod{a_1,a_2}\geq 2-2\delta_2$
     By triangle inequality, we then have
     $\norm{b_1-b_2}\geq \sqrt{2-\delta_2}-2\sqrt{2\delta_1}$. It then follows that
     \begin{equation*}
        \iprod{b_1,b_2}=\frac{2-\norm{b_1-b_2}^2}{2}\geq \frac{\delta_2+8\delta_1}{2}
     \end{equation*}
\end{proof}

 \noindent Now we are ready to prove~\cref{lem:pruning-components}.
 \begin{proof}[Proof of~\cref{lem:pruning-components}]
 By the discussion above~\cref{lem:pruning-components} we know that for $C$ computed in~\hyperref[alg:gaussian-rounding]{Step 1} of~\cref{alg:gaussian-rounding} it holds that
  \begin{equation*}
        \forall i \in S' \colon \max_{b\in \mathcal{C}} \Abs{\iprod{b,a_i}}\geq 1-\frac{1}{\polylog(n)}
    \end{equation*}
and
\begin{equation*}\label{eq:list-recovery-ref}
        \forall b \in C \colon \max_{i\in S} \Abs{\iprod{b,a_i}}\geq 1-\frac{1}{\polylog(n)}
\end{equation*}

\noindent To prove the lemma it is sufficient to show that 
 \begin{itemize}
\item for each $b_j\in S'$ there exists a unique $i\in S$ such that
    \begin{equation*}
          \iprod{b_j,a_i}\geq 1-\delta
    \end{equation*}
 \item for each $i\in S$ there exists a unique $b_j\in S'$ such that
        \begin{equation*}
           \iprod{b_j,a_i}\geq 1-\delta
       \end{equation*}
 \end{itemize}

Regarding the first point: By the first condition in \cref{eq:list-recovery-ref}, 
 for each $j\in S'$, there exists $i\in S$ 
 such that $\iprod{b_j,a_i}\geq 1-\delta$. 
 For the sake of contradiction assume that there exists $k \in S, k \neq i$ such that $\iprod{b_j,a_k}\geq 1-\delta$.
By our assumptions on the components (cf.~\cref{def:nicely-separated}) we have $\Abs{\iprod{a_i,a_k}}\leq \delta$.
Thus, invoking~\cref{lem:vector-recovering-different-components} with $b_1,b_2 = b_i$, $a_1 = a_i$, and $a_2 = a_k$, we get that $1 = \iprod{b_j, b_j} \leq \frac{9}{2} \cdot \delta < 1$. 
Hence, for each $b_j\in S'$, there is exactly one $i\in [n]$ such that $\iprod{b_j,a_i}\geq 1-\delta$.
 
Regarding the second point: By~\cref{lem:vector-recovering-same-component}, for any
  two vectors $b_{j_1},b_{j_2}$ s.t $\iprod{b_{j_1},a}\geq 1-\delta$ and 
  $\iprod{b_{j_2},a}\geq 1-\delta$, 
  we must have 
  $\iprod{b_{j_1},b_{j_2}}\geq 1-2\delta\geq 0.99$.
 Thus by the construction of $S'$, 
 for each $a_i$ there is at most one $b_j\in S'$,
 such that $\iprod{a_i,b_j}\geq 1-\delta$. On the other hand suppose 
 there exists $i\in S$
 such that $\max_{j\in S'} \iprod{b_j,a_i}\leq 1-\delta$. 
 Then for each $b_j\in S'$, we have $\iprod{b_j,a_{\ell}}\geq 1-\delta$
 for some $\ell\neq i$.
 Further by the list recovery guarantee,
 there exist $k\in [L]$ s.t $\iprod{b_k,a_i}\geq 1-\delta$. This means that 
 by~\cref{lem:vector-recovering-different-components},
  for any vector $b$ in $S'$, $\iprod{b_k,b}\leq O(\delta)$. 
  By construction, such vector $b_k$ should 
   be contained in the set $S'$, which leads to contradiction. 

 \end{proof}

 \subsection{Full recovery}
 In this section, we prove a technical lemma used for the proof 
  of \cref{thm:Full-recovery}.
  \begin{lemma}\label{lem:full-recovery-correlation-bound}
    For $d\leq n\leq O\Paren{d^{3/2}/\polylog(d)}$ and $m\geq d$, suppose vectors $a_1,a_2,\ldots,a_{m}$ are $(n,d)$ 
    nicely-separated, and $c_1,c_2,\ldots,c_{m}\in \mathbb{R}^d$
     has norm bounded by $\tilde{O}\Paren{\frac{\sqrt{n}}{d}}$. 
     Suppose for each $j\in [6]$, either for each $i\in [m]$,
     $g_i^{(j)}=a_i$, or for each $i\in [m]$, $g_i^{(j)}=c_i$.    
    Further suppose that for at least one of $j\in \{1,2,3\}$
     and at least one of $j\in \{4,5,6\}$
    , $g_i^{(j)}=c_i$.
    Suppose $M\in \mathbb{R}^{m\times m}$ has entries. 
       \begin{equation*}
           M_{i,j}=
           \iprod{g_i^{(1)}, g_j^{(4)}}\iprod{g_i^{(2)},g_j^{(5)}}
           \iprod{g_i^{(3)},g_j^{(6)}}
       \end{equation*}
       Then the frobenius norm of $M_{i,j}$ is bounded by $\tilde{O}\Paren{\sqrt{\frac{n}{d\sqrt{d}}}}$
  \end{lemma}

\begin{proof}
    We divide the choices of 
    $g^{(1)},g^{(2)},\ldots,g^{(6)}$ into 4 different cases, 
    according to the inner product 
     in $\iprod{g_i^{(1)}, g_j^{(4)}},
     \iprod{g_i^{(2)},g_j^{(5)}},\iprod{g_i^{(3)},g_j^{(6)}}$. 
     Particularly if $g_i^{(t)}=a_i$ and $g_j^{(t+3)}=c_j$,
      or $g_i^{(t)}=c_i$ and $g_j^{(t+3)}=a_j$, then we call 
       $\iprod{g_i^{(t)},g_j^{(t+3)}}$ a 
       cross inner product pair.
\begin{enumerate}[(1).]
    \item There are no cross inner product pairs, i.e
    $$\Abs{\Set{k\in [3]:\{g_i^{(k)},g_j^{(k+3)}\}\in \Set{a_i,a_j},\{c_i,c_j\}}}
    = 3\,.$$ 
    Since $a_i$ satisfies the $(n,d)$ nicely-separated assumption, 
    $\iprod{a_i,a_j}^2 \leq \tilde{O}\Paren{\frac{1}{d}}$.
    Since $\norm{c_i}\leq \frac{\sqrt{n}}{d}$, 
    $\iprod{c_i,c_j}^2\leq \tilde{O}\Paren{\frac{1}{d}}$. 
    In this case we have
    \begin{align*}
        \normf{M}^2=\sum_{\substack{i,j\in [n]\\i\neq j}}
        \iprod{g_i^{(1)}, g_j^{(4)}}^2\iprod{g_i^{(2)},g_j^{(5)}}^2
        \iprod{g_i^{(3)},g_j^{(6)}}^2
        \leq  n^2 \cdot (\frac{1}{d})^3 = \frac{n^2}{d^3}
    \end{align*}
    
    \item There is one cross inner product pair, i.e
    $$\Abs{\Set{k\in [3]:\{g_i^{(k)},g_j^{(k+3)}\}\in \Set{\{a_i,a_j\},\{c_i,c_j\}}}}
    =2\,.$$ 
    Since $a_i$ satisfies $(n,d)$ nicely-separated assumption
    , we have $\iprod{a_i,a_j}^2 \leq \tilde{O}\Paren{\frac{1}{d}}$, and 
    \begin{equation*}
        \Norm{\sum_{j\in [n]} a_ja_j^\top}\leq \frac{n}{d}
    \end{equation*}
    Further $\norm{c_i}\leq \frac{\sqrt{n}}{d}$ and 
    $\iprod{c_i,c_j}^2\leq \Paren{\frac{n}{d^2}}^2
    \leq \tilde{O}\Paren{\frac{1}{d}}$. 
    Thus we have
    \begin{align*}
        \normf{M}^2& =\sum_{\substack{i,j\in [n]\\i\neq j}}\iprod{g_i^{(1)}, g_j^{(4)}}^2\iprod{g_i^{(2)},g_j^{(5)}}^2\iprod{g_i^{(3)},g_j^{(6)}}^2
        \\
        &\leq  \sum_{\substack{i,j\in [n]\\i\neq j}} \frac{1}{d^2}\cdot \iprod{c_i,a_j}^2\\
        &= \frac{1}{d^2} \sum_{i\in [n]} c_i^\top \Paren{\sum_{\substack{j\in [n]\\j\neq i}} a_ja_j^\top} c_i \\
        & \leq \frac{1}{d^2} \sum_{i\in [n]} \norm{c_i}^2 \Norm{\sum_{\substack{j\in [n]\\j\neq i}}a_ja_j^\top}\\
        & \leq \frac{1}{d^2} \cdot n\cdot \frac{n}{d^2}\cdot \frac{n}{d}\\
        & \leq o\Paren{n^2/d^3} 
    \end{align*} 

    \item There are $2$ cross inner product pairs, i.e,
    $$\Abs{\Set{k\in [3]:\{g_i^{(k)},g_j^{(k+3)}\}\in \Set{\{a_i,a_j\},\{c_i,c_j\}}}}
    =1\,.$$ 
    Since $a_i$ satisfies $(n,d)$ nicely-separated assumption
    , we have $\iprod{a_i,a_j}^2 \leq \tilde{O}\Paren{\frac{1}{d}}$.
    Further $\norm{c_i}\leq \frac{\sqrt{n}}{d}$ and 
    $\iprod{c_i,c_j}^2\leq \Paren{\frac{n}{d^2}}^2
    \leq \tilde{O}\Paren{\frac{1}{d}}$. 
    We consider two different sub-cases:
    \begin{itemize}
        \item $M_{i,j}= \iprod{a_i,c_j}^2
        \iprod{a_j,c_i}^2 \iprod{c_i,c_j}^2$ or $M_{i,j}= \iprod{a_i,c_j}^2
        \iprod{a_j,c_i}^2 \iprod{a_i,a_j}^2$. 
        By the $(n,d)$ nicely-separated assumption on $a_j$, we have
        \begin{equation*}
            \Norm{\sum_{j\in [n]} a_ja_j^\top}\leq \frac{n}{d}
        \end{equation*}
        Thus in this case we have
        \begin{align*}
            \Norm{M}_F^2 &\leq \frac{1}{d} \sum_{\substack{i,j\in [n]\\i\neq j}}
            \iprod{a_i,c_j}^2
            \iprod{a_j,c_i}^2 \\
            & = \frac{1}{d}\cdot\frac{1}{\sqrt{d}} \sum_{\substack{i,j\in [n]\\i\neq j}}
             c_j^\top a_ia_i^\top c_j\\
            & \leq \frac{1}{d}\cdot\frac{1}{\sqrt{d}}\cdot \Paren{\sum_{j}\norm{c_j}^2}
            \cdot \Norm{\sum_{\substack{i\in [n]\\i\neq j}} a_ia_i^\top}\\
            &\leq \frac{1}{d}\cdot\frac{1}{\sqrt{d}}\cdot\frac{n}{\sqrt{d}}\cdot \frac{n}{d}\\
            & \leq \tilde{O}\Paren{\frac{n^2}{d^3}}
        \end{align*}
        \item $M_{i,j}= \iprod{a_i,c_j}^4
        \iprod{c_i,c_j}^2$ or $M_{i,j}= \iprod{a_i,c_j}^4
         \iprod{a_i,a_j}^2$. 
         By the $(n,d)$ nicely-separated assumption on $a_j$, we have
         \begin{equation*}
             \Norm{\sum_{j\in [n]} (a_ja_j^\top)^{\otimes 2}}\leq \frac{n}{d}
         \end{equation*}
         In this case we have 
        \begin{align*}
            \Norm{M}_F^2 &\leq \frac{1}{d} \sum_{i\neq j} \iprod{a_i,c_j}^4
            \\
            &\leq \frac{1}{d} \cdot \Paren{\sum_{j=1}^n \norm{c_j}^4} \cdot \Norm{\sum_{i}a_i^{\otimes 2}
            \Paren{a_i^{\otimes 2}}^\top}\\
            &\leq  \tilde{O}(\frac{1}{d} \cdot \frac{n}{d}\cdot \frac{n}{d})\\
            &\leq  \tilde{O}\Paren{n^2/d^3}
        \end{align*}
    \end{itemize}
    \item For the final case, we have three cross inner product pairs, i.e 
    $$\Abs{\Set{k\in [3]:\{g_i^{(k)},g_j^{(k+3)}\}\in \Set{\{a_i,a_j\},\{c_i,c_j\}}}}
    =0\,.$$ 
    Then w.l.o.g let $M_{i,j}=\iprod{a_i,c_j}^2 \iprod{a_j,c_i}$. 
    
    In this case, we use the fact that 
    \begin{align*}
        \Norm{M}_F^2 &= \sum_{\substack{i,j\in [n]\\i\neq j}} \iprod{a_i,c_j}^4 \iprod{a_j,c_i}^2
        \\
        & = \sum_{\substack{i,j\in [n]\\i\neq j}} \iprod{a_i^{\otimes 2}-Ra_i^{\otimes 2}+
        Ra_i^{\otimes 2},c_j^{\otimes 2}}^2
        \iprod{a_j,c_i}^2\\
        & \leq 2 \sum_{\substack{i,j\in [n]\\i\neq j}}\iprod{a_i^{\otimes 2}-Ra_i^{\otimes 2}
        ,c_j^{\otimes 2}}^2\iprod{a_j,c_i}^2+2 \sum_{\substack{i,j\in [n]\\i\neq j}}\iprod{Ra_i^{\otimes 2},c_j^{\otimes 2}}^2
        \iprod{a_j,c_i}^2 \\     
    \end{align*}
    For the first term, by the $(n,d)$ nicely-separated property, we have
    $\Norm{Ra_i^{\otimes 2}-a_i^{\otimes 2}}^2\leq \tilde{O}\Paren{\frac{1}{d}}$, 
    Thus 
    \begin{align*}
        \sum_{\substack{i,j\in [n]\\i\neq j}}
        \iprod{a_i^{\otimes 2}-Ra_i^{\otimes 2}
        ,c_j^{\otimes 2}}^2\iprod{a_j,c_i}^2
        &\leq \tilde{O}\Paren{\frac{1}{d^2}} \sum_{\substack{i,j\in [n]\\i\neq j}} 
        \iprod{a_j,c_i}^2 \\
        & =\tilde{O}\Paren{\frac{1}{d^2}}\sum_{\substack{i,j\in [n]\\i\neq j}} 
        c_i^\top a_ja_j^\top c_i\\
        & \leq \tilde{O}\Paren{\frac{1}{d^2}}\cdot 
        n\cdot\max_i \norm{c_i}^2\cdot \Norm{\sum_{j\in [n]\setminus \{i\}}
        a_ja_j^\top}\\
        & \leq \tilde{O}\Paren{\frac{1}{d^2}}\cdot n\cdot \tilde{O}\Paren{\frac{n}{d^2}}
        \cdot \tilde{O}\Paren{\frac{n}{d}}\\
        & \leq \tilde{O}\Paren{\frac{n^3}{d^5}}=o(n^2/d^3)
    \end{align*}
    For the second term, by the $(n,d)$ nicely-separated property
    of $a_i$, we have $\Norm{\sum_{i}Ra_i^{\otimes 2}
    \Paren{Ra_i^{\otimes 2}}^\top}\leq 2$. We then have
    \begin{align*}
        \sum_{\substack{i,j\in [n]\\i\neq j}}\iprod{Ra_i^{\otimes 2}
        ,c_j^{\otimes 2}}^2\iprod{a_j,c_i}^2
        & \leq \frac{1}{\sqrt{d}} \sum_{\substack{i,j\in [n]\\i\neq j}} 
        \iprod{Ra_i^{\otimes 2},c_j^{\otimes 2}}^2\\
        & \leq \frac{1}{\sqrt{d}} \cdot \Paren{\sum_{j=1}^n \norm{c_j}^4} \cdot 
        \Norm{\sum_{i}Ra_i^{\otimes 2}
        \Paren{Ra_i^{\otimes 2}}^\top}\\
        &\leq  \tilde{O}\Paren{\frac{1}{\sqrt{d}} \cdot \frac{n}{d}}\cdot 2\\
        &\leq \tilde{O}\Paren{\frac{n}{d\sqrt{d}}}
    \end{align*} 
\end{enumerate} 
Thus overall we can conclude that for each choice of $g$, 
$\normf{M}\leq \tilde{O}\Paren{\sqrt{\frac{n}{d\sqrt{d}}}}$. 
  \end{proof}

\end{document}